\documentclass{article}
\def\arxivcover{}

\usepackage[final]{neurips_2026}
\AtBeginDocument{\newgeometry{left=1.2in,right=1.2in,top=0.85in,bottom=0.85in,
  headheight=14pt,headsep=14pt,footskip=24pt}}
\pdfoutput=1
\usepackage{etoolbox,verbatim}
\newtoggle{iclr}
\toggletrue{iclr}
\newtoggle{arxiv}
\togglefalse{arxiv}
\newtoggle{neurips}
\togglefalse{neurips}
\newtoggle{mehul}
\togglefalse{mehul}
\newtoggle{shashwat}
\togglefalse{shashwat}
\newtoggle{customthms}
\togglefalse{customthms}
\newtoggle{colt}
\togglefalse{colt}

\usepackage{xcolor}
\usepackage{colortbl} 
\usepackage{tcolorbox}
\tcbuselibrary{skins,breakable}

\tcbset{
  aibox/.style={
      width=\linewidth,
      top=6pt,
      bottom=0pt,
      left=1.5pt,
      right=1.5pt,
      colback=blue!60!gray!5,
      colframe=black,
      colbacktitle=black,
      enhanced,
      center,
      attach boxed title to top left={yshift=-0.1in,xshift=0.15in},
      boxed title style={boxrule=0pt,colframe=white,},
    }
}
\newtcolorbox{AIbox}[2][]{aibox,title=#2,#1}

\definecolor{primalcolor}{HTML}{A60000}
\definecolor{contrarycolor}{HTML}{00A6A6}
\definecolor{darkcontrarycolor}{HTML}{004C4C}
\definecolor{lightblue}{HTML}{2970CC}
\definecolor{lightpurple}{HTML}{673147}
\definecolor{ForestGreen}{HTML}{FF5733}
\definecolor{myred}{HTML}{AA4A44}
\definecolor{hyppurple}{HTML}{800080}

\newcommand{\linkcolor}{darkcontrarycolor}
\newcommand{\urlcolor}{darkcontrarycolor}
\newcommand{\citecolor}{darkcontrarycolor}

\newcommand{\thmcolordark}{red!30!black}
\newcommand{\goalblue}{blue!60!black}
\newcommand{\gccolor}[1]{{\color{\goalblue} #1}}

\newcommand{\termcolordark}{red!50!black}
\newcommand{\termbold}[1]{{\color{\termcolordark} {\textbf{#1}}}}
\newcommand{\moddcolor}{red!50!black}
\newcommand{\modcolor}[1]{{\color{\moddcolor} #1}}

\definecolor{sparsecolorr}{HTML}{2A6E4B}

\newcommand{\sparsecolor}[1]{{\color{sparsecolorr} #1}}

\definecolor{steelmancolorr}{HTML}{7E57C2}

\newcommand{\steelmancolor}[1]{{\color{steelmancolorr} #1}}
\newcommand{\steelmanx}[1]{\steelmancolor{\textmd{Steelman #1}}}

\usepackage[T1]{fontenc}
\usepackage{courier}
\usepackage{capt-of}
\usepackage[font=small,labelfont=bf]{caption}

\usepackage{natbib}
\usepackage{scalefnt}
\usepackage{amssymb,upgreek,bm}
\usepackage{mathtools}
\usepackage[english]{babel}  
\usepackage{multirow,tcolorbox,tikz-cd,turnstile,xspace,xparse}
\usetikzlibrary{arrows.meta,positioning}
\usepackage{relsize,breakcites,appendix}
\usepackage{longtable,tablefootnote,booktabs,float,makecell}
\usepackage[normalem]{ulem}
\usepackage{tabu}
\usepackage[shortlabels]{enumitem} 
\usepackage{footnote}
\usepackage{boxedminipage}
\usepackage{multirow,nicefrac}
\usepackage{verbatim} 
\usepackage{wrapfig}
\usepackage{graphicx}

\usepackage[colorlinks=true,linkcolor=\linkcolor,urlcolor=\urlcolor,citecolor=\citecolor,breaklinks]{hyperref}

\usepackage{prettyref}

\iftoggle{colt}{
    \DeclareRobustCommand{\qed}{
        \usepackage{thmtools}
          \ifmmode \mathqed
          \else
            \leavevmode\unskip\penalty9999 \hbox{}\nobreak\hfill
            \quad\hbox{\qedsymbol}%
          \fi
    }
}
{
    \usepackage{amsthm}
     
}
\usepackage{thm-restate}

\iftoggle{arxiv}
{
    \usepackage{mathrsfs}

    \usepackage[
  letterpaper,
  left=1in,
  right=1in,
  top=0.9in,
  bottom=0.9in,
  headheight=14pt,
  headsep=14pt,
  footskip=20pt
]{geometry}

\usepackage{fancyhdr}
\usepackage{lipsum} 

\pagestyle{fancy}
\fancyhf{} 
\fancyfoot[C]{\small \thepage}

\renewcommand{\headrulewidth}{0.4pt}
\renewcommand{\footrulewidth}{0pt}
    \usepackage[bitstream-charter]{mathdesign}
     
    \usepackage[scaled=0.92]{PTSans}

    \usepackage{subfig}

    \usepackage[noend]{algpseudocode}
    \usepackage{algorithm}
   
}
{
    \usepackage{mathrsfs}
    \usepackage{subfig}

    \usepackage[noend]{algpseudocode}
    \usepackage{algorithm}
}

\DeclareMathAlphabet{\mathbfsf}{\encodingdefault}{\sfdefault}{bx}{n}

\numberwithin{equation}{section}

\iftoggle{arxiv}
{
    \newcommand{\colorbold}[1]{
    \textbf{\textcolor{\thmcolor}{#1}}}}
{
    \newcommand{\colorbold}[1]{\textbf{#1}
}
}
\iftoggle{arxiv}
{
    \newcommand{\colorpar}[1]{
    \paragraph{\textcolor{\thmcolor}{#1}}}
}
{
    \newcommand{\colorpar}[1]{\paragraph{#1}}
}

\usepackage[capitalise,nameinlink]{cleveref}
\Crefname{Appendix}{Appendix}{Appendices}
  \newtheorem{theorem}{Theorem}
  \newtheorem{lemma}{Lemma}[section]
  \newtheorem{corollary}{Corollary}[section]
  \newtheorem{proposition}[lemma]{Proposition}
  \theoremstyle{definition}
  \newtheorem{definition}{Definition}[section]
  \newtheorem{example}{Example}[section]
  \newtheorem{remark}{Remark}[section]
  \theoremstyle{plain}

\usepackage{thmtools}

\Crefname{equation}{Eq.}{Eqs.}
\Crefname{assumption}{Assumption}{Assumptions}
\Crefname{condition}{Condition}{Conditions}
\Crefname{claim}{Claim}{Claims}
\Crefname{property}{Property}{Properties}
\Crefname{construction}{Construction}{Constructions}

\declaretheoremstyle[
    headformat=\normalfont\textcolor{\thmcolordark}{\bfseries\NAME\,\NUMBER}\NOTE,%
    notefont={\normalfont\textcolor{\thmcolordark}{\bfseries}}, 
    notebraces={}{},
    bodyfont=\normalfont\itshape,
    spaceabove = 6pt,
    spacebelow = 6pt,
    ]{coloredthmversion}

\declaretheoremstyle[
    headformat=\normalfont\textcolor{\thmcolordark}{\bfseries\NAME\,\NUMBER}\NOTE,%
    bodyfont=\normalfont\itshape,
    spaceabove = 6pt,
    spacebelow = 6pt,
    ]{coloredthm}

\declaretheoremstyle[
    headformat=\normalfont\textcolor{\thmcolordark}{\bfseries\NAME\,\NUMBER}\NOTE,%
    bodyfont=\normalfont,
    spaceabove = 6pt,
    spacebelow = 6pt,
    ]{coloreddef}

\iftoggle{customthms}
{
    \theoremstyle{coloredthmversion}
}
{}

\iftoggle{customthms}{
  \theoremstyle{coloredthm}
  \newtheorem{theorem}{Theorem}
  \newtheorem{lemma}{Lemma}[section]
  \newtheorem{corollary}{Corollary}[section]
  \newtheorem{proposition}[lemma]{Proposition}
}
{}

\newtheorem*{thminformal*}{Informal Theorem}

\iftoggle{customthms}{
    \theoremstyle{coloreddef}
    \newtheorem{definition}{Definition}[section]

    \newtheorem{property}{Property}[section]
}
{
  
}

\newtheorem{assumption}{Assumption}[section]
\newtheorem{condition}{Condition}[section]

\makeatletter
\newcommand{\neutralize}[1]{\expandafter\let\csname c@#1\endcsname\count@}
\makeatother

\iftoggle{customthms}{
    \newtheoremstyle{named}{}{}{\itshape}{}{\bfseries}{}{.5em}{\Cref{#3} {\normalfont (informal)} }{}
    \theoremstyle{named}
    
    \theoremstyle{plain}
}
{}

\newtheorem*{theorem*}{Theorem}
\newtheorem*{lemma*}{Lemma}
\newtheorem*{corollary*}{Corollary}
\newtheorem*{proposition*}{Proposition}
\newtheorem*{claim*}{Claim}
\newtheorem*{fact*}{Fact}
\newtheorem*{observation*}{Observation}
\newtheorem*{definition*}{Definition}
\newtheorem*{remark*}{Remark}
\newtheorem*{example*}{Example}

\def\by{\mathbf{y}}
\def\bz{\mathbf{z}}

\def\bq{\mathbf{q}}
\def\br{\mathbf{r}}

\def\bw{\mathbf{w}}

\def\bQ{\mathbf{Q}}

\def\bphi{\bm{\phi}}
\def\bpsi{\bm{\psi}}

\def\ddefloop#1{\ifx\ddefloop#1\else\ddef{#1}\expandafter\ddefloop\fi}
\def\ddef#1{\expandafter\def\csname bb#1\endcsname{\ensuremath{\mathbb{#1}}}}
\ddefloop ABCDEFGHIJKLMNOPQRSTUVWXYZ\ddefloop

\def\ddefloop#1{\ifx\ddefloop#1\else\ddef{#1}\expandafter\ddefloop\fi}
\def\ddef#1{\expandafter\def\csname frak#1\endcsname{\ensuremath{\mathfrak{#1}}}}
\ddefloop ABCDEFGHIJKLMNOPQRSTUVWXYZ\ddefloop

\def\ddefloop#1{\ifx\ddefloop#1\else\ddef{#1}\expandafter\ddefloop\fi}
\def\ddef#1{\expandafter\def\csname fr#1\endcsname{\ensuremath{\mathfrak{#1}}}}
\ddefloop ABCDEFGHIJKLMNOPQRSTUVWXYZ\ddefloop

\def\ddefloop#1{\ifx\ddefloop#1\else\ddef{#1}\expandafter\ddefloop\fi}
\def\ddef#1{\expandafter\def\csname eul#1\endcsname{\ensuremath{\EuScript{#1}}}}
\ddefloop ABCDEFGHIJKLMNOPQRSTUVWXYZ\ddefloop

\def\ddefloop#1{\ifx\ddefloop#1\else\ddef{#1}\expandafter\ddefloop\fi}
\def\ddef#1{\expandafter\def\csname scr#1\endcsname{\ensuremath{\mathscr{#1}}}}
\ddefloop ABCDEFGHIJKLMNOPQRSTUVWXYZ\ddefloop

\def\ddefloop#1{\ifx\ddefloop#1\else\ddef{#1}\expandafter\ddefloop\fi}
\def\ddef#1{\expandafter\def\csname b#1\endcsname{\ensuremath{\mathbf{#1}}}}
\ddefloop ABCDEFGHIJKLMNOPQRSTUVWXYZ\ddefloop

\def\ddefloop#1{\ifx\ddefloop#1\else\ddef{#1}\expandafter\ddefloop\fi}
\def\ddef#1{\expandafter\def\csname bhat#1\endcsname{\ensuremath{\hat{\mathbf{#1}}}}}
\ddefloop ABCDEFGHIJKLMNOPQRSTUVWXYZ\ddefloop

\def\ddefloop#1{\ifx\ddefloop#1\else\ddef{#1}\expandafter\ddefloop\fi}
\def\ddef#1{\expandafter\def\csname btil#1\endcsname{\ensuremath{\tilde{\mathbf{#1}}}}}
\ddefloop ABCDEFGHIJKLMNOPQRSTUVWXYZ\ddefloop

\def\ddefloop#1{\ifx\ddefloop#1\else\ddef{#1}\expandafter\ddefloop\fi}
\def\ddef#1{\expandafter\def\csname bst#1\endcsname{\ensuremath{\mathbf{#1}^\star}}}
\ddefloop ABCDEFGHIJKLMNOPQRSTUVWXYZ\ddefloop

\def\ddefloop#1{\ifx\ddefloop#1\else\ddef{#1}\expandafter\ddefloop\fi}
\def\ddef#1{\expandafter\def\csname bst#1\endcsname{\ensuremath{\mathbf{#1}^\star}}}
\ddefloop abcdeghijklmnopqrstuvwxyz\ddefloop

\def\ddefloop#1{\ifx\ddefloop#1\else\ddef{#1}\expandafter\ddefloop\fi}
\def\ddef#1{\expandafter\def\csname bhat#1\endcsname{\ensuremath{\hat{\mathbf{#1}}}}}
\ddefloop abcdefghijklmnopqrstuvwxyz\ddefloop

\def\ddefloop#1{\ifx\ddefloop#1\else\ddef{#1}\expandafter\ddefloop\fi}
\def\ddef#1{\expandafter\def\csname b#1\endcsname{\ensuremath{\mathbf{#1}}}}
\ddefloop abcdeghijklnopqrstuvwxyz\ddefloop

\def\ddefloop#1{\ifx\ddefloop#1\else\ddef{#1}\expandafter\ddefloop\fi}
\def\ddef#1{\expandafter\def\csname barb#1\endcsname{\ensuremath{\bar{\mathbf{#1}}}}}
\ddefloop abcdefghijklmnopqrstuvwxyz\ddefloop

\def\ddef#1{\expandafter\def\csname c#1\endcsname{\ensuremath{\mathcal{#1}}}}
\ddefloop ABCDEFGHIJKLMNOPQRSTUVWXYZ\ddefloop
\def\ddef#1{\expandafter\def\csname h#1\endcsname{\ensuremath{\widehat{#1}}}}
\ddefloop ABCDEFGHIJKLMNOPQRSTUVWXYZ\ddefloop
\def\ddef#1{\expandafter\def\csname hc#1\endcsname{\ensuremath{\widehat{\mathcal{#1}}}}}
\ddefloop ABCDEFGHIJKLMNOPQRSTUVWXYZ\ddefloop
\def\ddef#1{\expandafter\def\csname t#1\endcsname{\ensuremath{\widetilde{#1}}}}
\ddefloop ABCDEFGHIJKLMNOPQRSTUVWXYZ\ddefloop
\def\ddef#1{\expandafter\def\csname tc#1\endcsname{\ensuremath{\widetilde{\mathcal{#1}}}}}
\ddefloop ABCDEFGHIJKLMNOPQRSTUVWXYZ\ddefloop

\usepackage{tcolorbox}
\tcbuselibrary{skins,breakable}

\tcbset{
  aibox/.style={
      width=\linewidth,
      top=6pt,
      bottom=0pt,
      left=1.5pt,
      right=1.5pt,
colback=blue!60!gray!5,
colframe=black,
      colbacktitle=black,
      enhanced,
      center,
      attach boxed title to top left={yshift=-0.1in,xshift=0.15in},
      boxed title style={boxrule=0pt,colframe=white,},
    }
}

\newcommand{\I}{\mathbf{I}}

\let\Pr\relax
\DeclareMathOperator{\Pr}{\mathbb{P}}

\newcommand{\Exp}{\mathbb{E}}

\newcommand{\Normal}{\mathrm{N}}

\newcommand{\ballkr}[1][r]{\cB_{k}(r)}

\iftoggle{arxiv}
{
\newcommand{\confinlinestyle}{
}
}
{
\newcommand{\confinlinestyle}{\textstyle
}
}

\iftoggle{arxiv}
{
\newcommand{\confpar}[1]{\paragraph{#1}}
}
{
\newcommand{\confpar}[1]{\textbf{#1}}
}

\DeclareMathSymbol{\shortminus}{\mathbin}{AMSa}{"39}

\newcommand{\R}{\mathbb{R}}

\usepackage{pifont}

\newcommand{\statespace}{S}
\newcommand{\actionspace}{A}

\newcommand{\piema}{\pi_{\bar\theta}}

\newcommand{\algofont}[1]{{\scalefont{1.0}{\texttt{\textbf{#1}}}}}

\newcommand{\Qtarg}{Q_{\mathrm{targ}}}
    \newcommand{\Buffer}{\cB}

\newcommand{\Lcritic}{L_{\mathrm{critic}}}

\newcommand\mdpenv{M_\textsc{Env}}
\newcommand{\Jcost}{J}

\newcommand{\igpfirst}{K}

\newcommand{\igplast}{0}
\newcommand{\igpiter}{k}

\newcommand{\igpinc}{k-1}

\newcommand{\aitert}[1][t]{a_{#1,\igpiter}}
\newcommand{\ainct}[1][t]{a_{#1,\igpinc}}

\newcommand{\alastt}[1][t]{a_{#1,\igplast}}

\newcommand{\done}{\texttt{done}}

\newcommand{\rboff}{\mathcal{D}_{\text{off}}}

\newcommand{\ratio}{\omega_\theta}

\newcommand{\gnext}{g^{\mathrm{next}}}
\newcommand{\Gnext}{G^{\mathrm{next}}}
\newcommand{\gcurr}{g^{\mathrm{curr}}}
\newcommand{\Gcurr}{G^{\mathrm{curr}}}
\newcommand{\Gterm}{G^{\mathrm{term}}}

\newcommand{\pred}{\mathrm{pred}}
 
\newcommand{\Lpol}{L_{\mathrm{policy}}}

\newcommand{\subroutinefont}[1]{{\textsc{#1}}}
\newcommand{\thetatarget}{\bar{\theta}}

\newcommand{\cclip}{c_{\mathrm{clip}}}

\newcommand{\getWeight}{\subroutinefont{GetWeighting}}
\newcommand{\ytarg}{y^{\mathrm{target}}}
\newcommand{\bytarg}{\by^{\mathrm{target}}}

\newcommand{\Qagg}{Q_{\mathrm{agg}}}

\newcommand{\regu}{L_{\mathrm{reg}}}

\newcommand{\rdone}{r_{\mathrm{done}}}
\newcommand{\rsparse}{r_{\mathrm{sparse}}}
\newcommand{\rdense}{r_{\mathrm{dense}}}

\newcommand{\Vadarnc}{\textbf{\textsc{VaDar}}}
\newcommand{\Vadar}{\modcolor{\textbf{\textsc{VaDar}}}}

\newcommand{\VadarVector}{\modcolor{\textbf{\textsc{Vector-VaDar}}}}
\newcommand{\VadarShaped}{\modcolor{\textbf{\textsc{VaDar-Shaped}}}}
\newcommand{\PrefSubgoal}{\textsc{Pref-Subgoal}}
\newcommand{\LastQ}{\modcolor{\textbf{\textsc{VaDarLast-Q}}}}
\newcommand{\VadarGM}{\modcolor{\textbf{\textsc{VaDar}-\textsc{GM}}}}
\newcommand{\VadarNF}{\modcolor{\textbf{\textsc{VaDar}-\textsc{NF}}}}
\newcommand{\VadarCF}{\modcolor{\textbf{\textsc{VaDar}-\textsc{CF}}}}

\newcommand{\desc}{\mathrm{desc}}

  \newcommand{\hgcnc}{\textbf{\textsc{HGC}}}
  
  \newcommand{\hgcrlnc}{\textsc{HGC}\text{-}\textsc{RL}}

  \newcommand{\hgc}{\textbf{\gccolor{\hgcnc}}}
  \newcommand{\ihgc}{\textrm{i}\hgc}
  \newcommand{\hgcrl}{\gccolor{\hgcrlnc}}

\newcommand{\rbm}{\textsc{Robometer}}
\newcommand{\pbrl}{\textsc{Pbrl}}

\newcommand{\vectorqfirst}{Q_{\phi_1}}
\newcommand{\vectorqlast}{Q_{\phi_M}}

\newcommand{\aggSAC}{\textsc{Aggregated}-\SAC{}}
\let\algofont\undefined

\let\Vadarnc\undefined

\newcommand{\algofont}[1]{{\scalefont{1.0}{\texttt{\textbf{#1}}}}}

\definecolor{dppocol}{HTML}{777777}
\definecolor{dsrlcol}{HTML}{925BFF}
\definecolor{qccol}{HTML}{2A52BE}
\definecolor{expocol}{HTML}{50C878}
\definecolor{srcol}{HTML}{5B1166}
\definecolor{bpttcol}{HTML}{FFB366}

\newcommand{\FQLnc}{\algofont{FQL}}
\newcommand{\FQL}{{\FQLnc}}

\newcommand{\OGPOcolor}{red!70!black}
\newcommand{\OGPOtext}{OGPO}
\newcommand{\OGPOnc}{\algofont{\OGPOtext}}

\newcommand{\OGPO}{{\OGPOnc}}
\newcommand{\OGPOplus}{{\color{\OGPOcolor}\OGPOnc+}}

\newcommand{\SACnc}{\algofont{SAC}}
\newcommand{\SAC}{{\SACnc}}

\newcommand{\envfont}[1]{\textsc{#1}}

\newcommand{\Robomimic}{\envfont{Robomimic}}

\newcommand{\presoakpan}{\envfont{PreSoakPan}}
\newcommand{\loaddishwasher}{\envfont{LoadDishwasher}}
\newcommand{\arrangetea}{\envfont{ArrangeTea}}

\newcommand{\Libero}{\envfont{Libero}}
\newcommand{\OGBench}{\envfont{OGBench}}
\newcommand{\Robocasa}{\envfont{Robocasa365}}
\newcommand{\ToolHang}{\envfont{Tool-Hang}}
\newcommand{\Transport}{\envfont{Transport}}
\newcommand{\PushTObs}{\envfont{PushT-Obstacle}}

\newcommand{\MIQ}{\textsc{MIQ}}

\newcommand{\Vadarnc}{\textsc{VaDar}}
\let\Vadarnc\iclrSavedVadarnc
\renewcommand{\Vadarnc}{\textmd{\textsc{VaDar}}}
\renewcommand{\Vadar}{\modcolor{\textmd{\textsc{VaDar}}}}

\renewcommand{\VadarVector}{\modcolor{\textmd{\textsc{Vector-VaDar}}}}
\renewcommand{\VadarShaped}{\modcolor{\textmd{\textsc{VaDar-Shaped}}}}
\renewcommand{\LastQ}{\modcolor{\textmd{\textsc{VaDarLast-Q}}}}
\renewcommand{\VadarGM}{\modcolor{\textmd{\textsc{VaDar}-\textsc{GM}}}}
\renewcommand{\VadarNF}{\modcolor{\textmd{\textsc{VaDar}-\textsc{NF}}}}
\renewcommand{\VadarCF}{\modcolor{\textmd{\textsc{VaDar}-\textsc{CF}}}}
\renewcommand{\hgcnc}{\textmd{\textsc{HGC}}}

\renewcommand{\hgcrlnc}{\textmd{\textsc{HGC}\text{-}\textsc{RL}}}
\renewcommand{\hgc}{\gccolor{\hgcnc}}
\renewcommand{\ihgc}{\textmd{\textrm{i}\hgc}}
\renewcommand{\hgcrl}{\gccolor{\hgcrlnc}}

\usepackage{preamble/color-edits}

\addauthor{ms}{magenta}
\addauthor{sv}{teal}

\addauthor{mg}{orange}
\addauthor{ss}{blue}
\addauthor{sbp}{purple}

\newcommand{\ignore}[1]{}

\iftoggle{arxiv}
{
\input{preamble/arxiv_title}
}
{}

\title{\Vadar: Implicit Goal Conditioning via Value Disaggregation}

\author{
  \small
  Shashwat Saxena\footnote{\texttt{\{ssaxena2, mehulgoel, sreyasv, sarveshp, msimchow\}@andrew.cmu.edu}\label{foot:cmu}}\textsuperscript{$*$},
  Mehul Goel\textsuperscript{\ref{foot:cmu}$*$},
  Sreyas Venkataraman\textsuperscript{\ref{foot:cmu}$*$},
  Sarvesh Patil\textsuperscript{\ref{foot:cmu}},
  Max Simchowitz\textsuperscript{\ref{foot:cmu},$b$}
  \\
  \vspace{-.3em}
  \rule{.38\textwidth}{.7pt}
  \\
  \footnotesize
  $^*$Equal contribution. \\
  \footnotesize
  $^{a}$Carnegie Mellon University \\
  \footnotesize
  $^{b}$Amazon Scholar, Amazon FAR
}

\usepackage{placeins}
\usepackage{needspace}
\newcommand{\iclrendwrap}{%
  \par
  \ifnum\value{WF@wrappedlines}>1
    \vspace{\dimexpr\value{WF@wrappedlines}\baselineskip-\baselineskip\relax}%
  \fi
  \WFclear
}

\usepackage{flafter} 
\providecommand{\newtext}[1]{#1}

\begin{document}
\ifdefined\arxivcover
\thispagestyle{empty}
\begin{tcolorbox}[
  colback=blue!60!gray!5, colframe=gray!50,
  boxrule=0pt, arc=2mm,
]
{\raggedright\LARGE\scshape
  \includegraphics[height=1em]{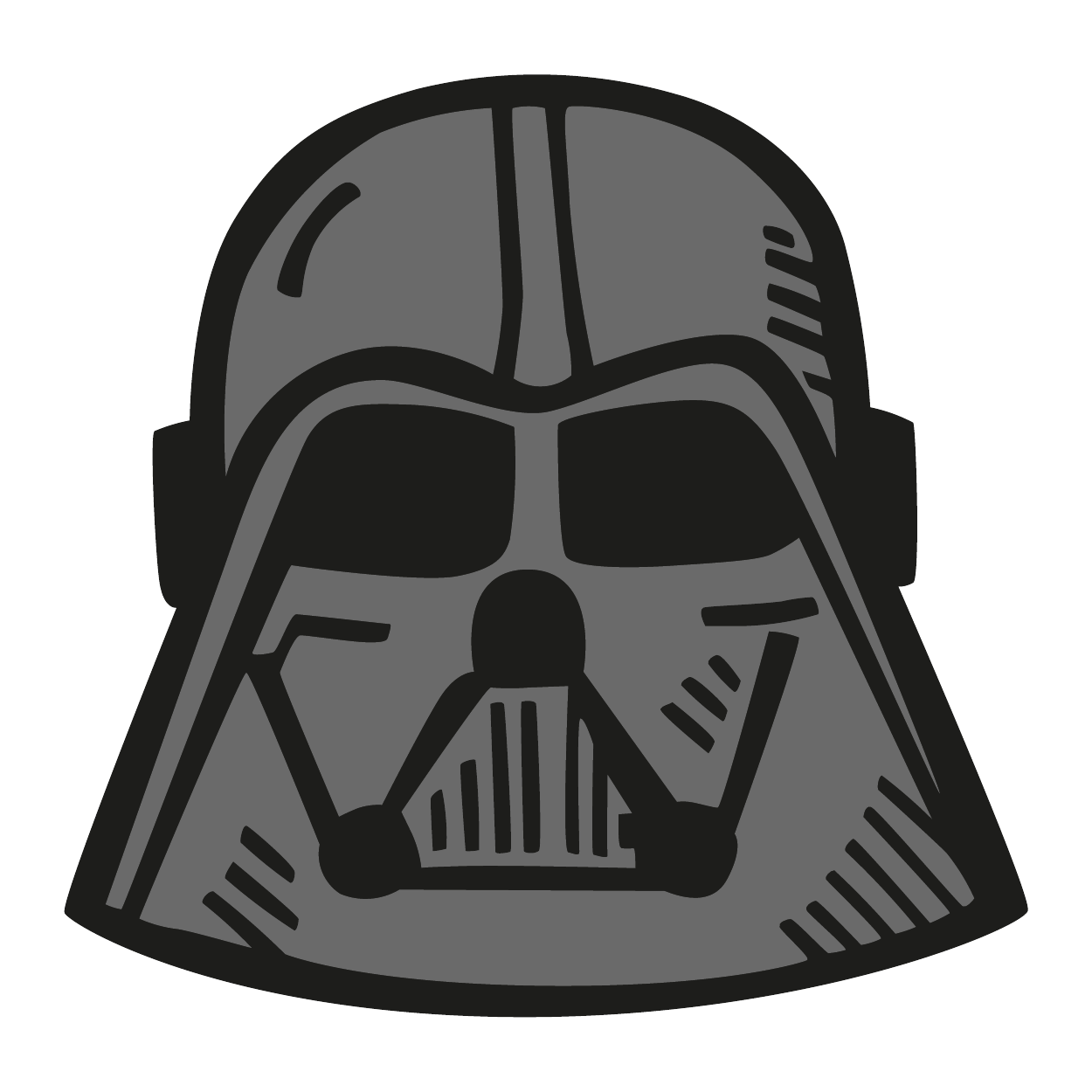}\ \Vadar:
  Implicit Goal Conditioning via Value Disaggregation\par}
\medskip
{\raggedright\small
  {\bfseries
  Shashwat Saxena\textsuperscript{1,*}\quad
  Mehul Goel\textsuperscript{1,*}\quad
  Sreyas Venkataraman\textsuperscript{1,*}\quad
  Sarvesh Patil\textsuperscript{1}\quad
  Max Simchowitz\textsuperscript{1}\par}
  \medskip
  \footnotesize
  \textsuperscript{1}Carnegie Mellon University \quad 
  \textsuperscript{*}Equal contribution.\par
}
\tcbline
\vspace{8pt}
\begin{minipage}[t]{\linewidth}
    \centering
    \includegraphics[width=\textwidth]{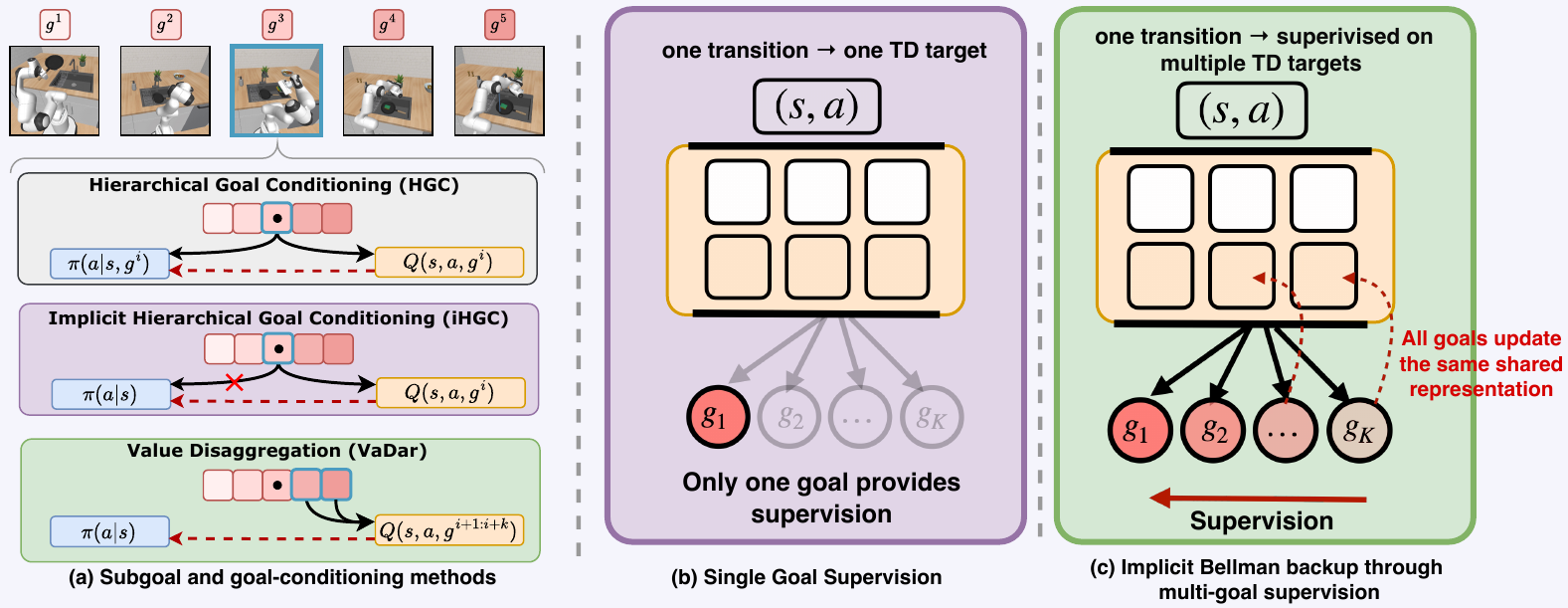}

     \captionof{figure}{\small
      \Vadar{} uses subgoals during training but requires none at deployment.
      \textbf{(a)} \hgc{} conditions both policy and critic on the current subgoal; \ihgc{} removes policy conditioning. \Vadar{} adds all-subgoal critic supervision and future-aware policy extraction.
      \textbf{(b)} With single-goal supervision, each transition supplies a TD target for only the current subgoal.
      \textbf{(c)} \Vadar{} supervises all subgoals on each transition, updating a shared critic representation. We hypothesize that this enables an implicit Bellman backup that improves final-goal value estimates. }
    \label{fig:teaser}
\end{minipage}

\par\vspace{8pt}
\tcbline
{\footnotesize
  {\small\texttt{\{ssaxena2,mehulgoel,sreyasv,sarveshp,msimchow\}@andrew.cmu.edu}\par}
  \smallskip
  \vspace{8pt}
  \textbf{Code:} \url{https://github.com/simchowitzlabpublic/vadar-public}\par
  \smallskip
  \textbf{Date:} \today\par
}
\end{tcolorbox}

\else
\maketitle
\fi

\begin{abstract}


Long-horizon tasks are important for modern AI and robotics, yet pose a challenge for reinforcement learning. Historically, a popular approach has been to decompose a task into subgoals, and learn a subgoal-conditioned policy to execute each subtask in sequence.  
We introduce an alternative called \emph{Value Disaggregation} (\Vadar{}), which  trains a single subgoal-independent policy to maximize a linear combination of per-subgoal critic values. 
Because subgoal information is needed only at \emph{training-time}, \Vadarnc{} removes the need for possibly costly or cumbersome subgoal-selection during policy deployment. Moreover, through careful experiments, we  show that \Vadarnc{} (1) is always competitive with, and frequently \emph{outperforms}, natural baselines like goal-conditioning and (2) succeeds on tasks with non-sequential goal structure and multiple optimization objectives. In particular, \Vadarnc{} is the \textbf{first} algorithm to solve challenging, long-horizon tasks in the \Robocasa{} suite, even when starting from a policy pretrained from behavior cloning  with \emph{near-zero} initial task success. Taken together, \Vadarnc{}'s successes suggest that subgoal information benefits actor-critic reinforcement learning by enhancing the efficacy of \emph{critic learning}, whereas policy goal-conditioning is often unnecessary.


\end{abstract}

\ifdefined\arxivcover\else
\begin{figure}[!t]
\input{arxiv_v2/body/intro_fig}
\end{figure}
\fi


\section{Introduction}

Learning to complete long-horizon, multi-step tasks is a fundamental challenge across a number of core AI domains, from robotics \citep{gupta2019relay}, to language modeling  \citep{wei2022chain}, to chemical reaction design \citep{liu2023fusionretro}. A popular approach, often referred to as  \emph{hierarchical goal conditioning} (\hgcnc) in reinforcement learning literature, decomposes these complex, long-horizon tasks into a set of subgoals, each of which can be  executed in sequence by a dedicated learned, goal-conditioned policy \citep{gupta2019relay, nasiriany2019planning, levy2019learningmultilevelhierarchieshindsight, nachum2018data}. However, attempting to learn per-sub-task policies introduces a new set of challenges. Myopic completion of a given sub-task might compromise performance on future sub-tasks. Moreover, explicit decompositions often require sub-tasks to be sequential in nature, and necessitate a dedicated sub-task planner to decide the current subtask to be attempted during policy deployment.

\iftoggle{arxiv}{\subsection{Contributions.}}{} In this paper, we propose an alternative to \hgcnc{} that we call \termbold{Value-Disaggregation} (\Vadar), a framework for sample-efficient reinforcement learning on long-horizon multi-step tasks. We focus on continuous control robotic manipulation tasks that can be decomposed into sub-tasks $g^1,\dots,g^{|G|}$. These subgoal evaluations are rapidly becoming highly reliable with today's VLMs \citep{du2023vision}.
While \hgcnc{} learns policies $\pi(s,g)$ that take sub-goals $g$ as explicit input, \Vadarnc{} instead learns a sub-goal agnostic policy $\pi(s)$, with a critic that estimates the values of all sub-tasks simultaneously. The policy is then trained to maximize a linear combination of critic values, optimizing for an aggregate performance on current-sub-goals as well as future tasks. We stress that \Vadarnc{} is a generic algorithmic template, being compatible with any preferred choice of critic learning algorithm or policy extraction method; we depict \Vadarnc{} in \Cref{fig:teaser}.

\iftoggle{neurips}{
\begin{figure}[H]
    \vspace{-0.5em}
    \centering
    \includegraphics[width=\textwidth]{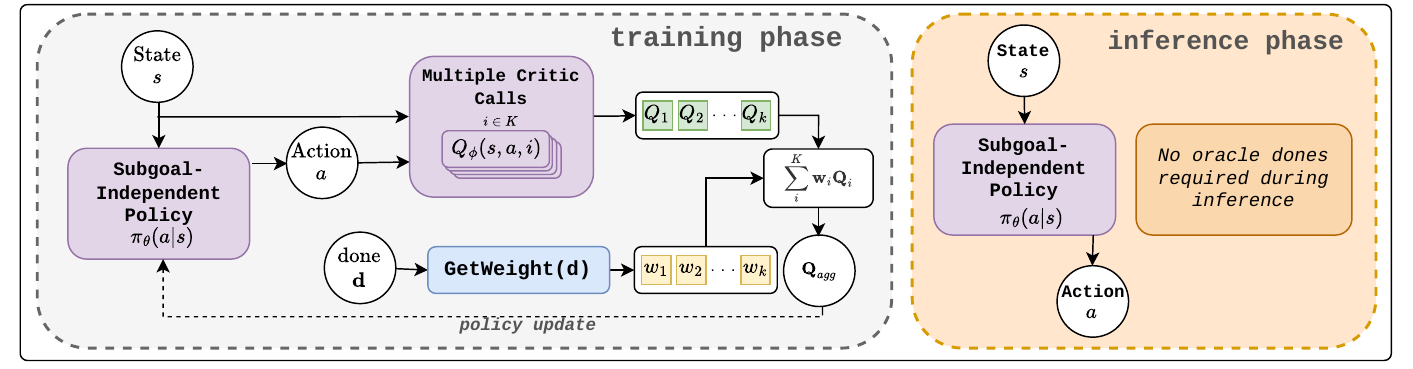}
    \caption{\Vadarnc{} algorithm. We learn per-goal critics with a shared backbone, and optimize a weighted sum of these to update a subgoal independent policy, which does not require subgoals to be supplied at inference.}
    \label{fig:vadar-draft}
    \vspace{-1.0em}
\end{figure}
}{}

\Vadarnc{} removes the need for subgoal-evaluation at inference \emph{by design}. Thus, compute-intensive VLM calls used to delineate subgoals, need only be made at \emph{training time}, and not during policy rollouts. Moreover, \Vadarnc{} is flexible across both task structure and reward constraints, enabling activations both to safety-constrained reinforcement learning (\Cref{sec:exp-multiobj}), and to non-sequential task orders which could admit a directed acyclic graph (DAG) structure. (\Cref{app:dag_extension}).

Across a variety of benchmarks for continuous control and policy extraction methods, we find that \Vadarnc{}, \textbf{(1)} Consistently outperforms \hgcnc{} (\Cref{sec:hgc-prelim}), hindsight experience replay (HER) \citep{andrychowicz2018hindsightexperiencereplay}, and dense preference-based rewards \citep{pmlr-v97-brown19a,liang2026robometer}; \textbf{(2)} Enables the first non-VLA BC-to-online RL method to converge on the evaluated long-horizon \Robocasa{} tasks \citep{nasiriany2026robocasa365} under fixed initial-state conditions; \textbf{(3)} Enables multi-objective reinforcement learning by incorporating safety constraints (\Cref{sec:exp-multiobj}), and enables test-time steering (\Cref{sec:exp-steering}) by combining critic learning with best-of-$N$ sampling \citep{li2026reinforcement}; and \textbf{(4)} Works seamlessly with VLM supervision (\Cref{sec:vlm}) and reaches convergence when goals are specified by reasoning Vision-Language Models (VLMs). This allows sample efficient robot learning from high-dimensional visual observations. 

To understand these gains, our ablations examine dense subgoal supervision and future-aware policy extraction (\Cref{sec:understanding_vadar}). We find that future-aware policy extraction reduces myopic behavior, while intermediate subgoal supervision improves learning even when policy extraction uses only final-subgoal values. This suggests that subgoals can support long-horizon RL through critic learning without explicit policy goal conditioning. 
Taken together, our findings \textbf{invert} the traditional perspectives on overcoming the challenges of long horizon, valued-based reinforcement learning. In short, we find: \emph{Rather than myopically conditioning policies via decomposed long-horizon sub-goals, predicting temporally rich subgoal successes can provide critics with \textbf{dense supervision} to enhance the accuracy of long-reaching measurements of full-task success.}

\iftoggle{neurips}
{
\confpar{Abridged Related Work.} 
Long-horizon tasks are commonly addressed through hierarchical goal-conditioning \citep{nachum2018data,levy2019learningmultilevelhierarchieshindsight}, which decomposes tasks into intermediate subgoals with goal-conditioned policies, or through designed/learned dense rewards \citep{chan2024dense, liang2026robometer}.
In contrast, \Vadar{} addresses long-horizon credit assignment through value disaggregation, decomposing value estimation rather than the task itself. This is closely related in spirit to successor-feature methods, which factor value functions into predictive state-occupancy statistics and reward-dependent weights. In \Cref{app:related}, we discuss related work in more detail. \mgcomment{the biggest concern at neurips reviews was the lack of related work, does that matter here?}
}
{}
\section{Related Work}
\label{sec:related}

\confpar{Long-horizon reinforcement learning via temporal abstraction.}
Long-horizon reinforcement learning is commonly addressed by introducing temporal abstraction: instead of reasoning over primitive actions, the agent reasons over extended behaviors, skills, or subgoals. Classical formulations include the options framework, where temporally extended actions are specified by initiation sets, intra-option policies, and termination conditions~\citep{sutton1999between,bacon2017option}; hierarchical abstract machines~\citep{parr1997reinforcement}; feudal reinforcement learning~\citep{dayan1992feudal,vezhnevets2017feudalnetworkshierarchicalreinforcement}; and MAXQ value decomposition~\citep{dietterich1999hierarchicalreinforcementlearningmaxq}. Subsequent work studies how such skills and hierarchies can be discovered from data~\citep{fox2017multileveldiscoverydeepoptions,evans2024creatingmultilevelskillhierarchies,cannon2025acceleratingtaskgeneralisationmultilevel,peng2022causalitydrivenhierarchicalstructurediscovery}. While these methods differ in how subtasks are specified or learned, they use subgoal structure to organize policy execution: the learned policy must either act through a prescribed hierarchy of subtasks\citep{dietterich1999hierarchicalreinforcementlearningmaxq, parr1997reinforcement} or utilize temporally extended controllers \citep{bacon2017option, fox2017multileveldiscoverydeepoptions}. In contrast, \Vadar{} assumes that subgoals are part of the MDP design, but does not learn or execute a hierarchy over them; subgoals are used only to construct training-time value targets for a single flat policy.

\confpar{Goal-conditioned reinforcement learning.} Goal conditioned RL provide a natural design for conditioning on desired outcomes/sub-goals.
Instead of learning a separate controller for each task, goal-conditioned methods learn policies or value functions indexed by a goal, enabling transfer across related goal-reaching problems.
Classical dynamic-goal learning first studied transfer across changing goal-reaching tasks~\citep{Kaelbling1993LearningTA}. 
This idea was later extended through universal value function approximators, which learn goal-indexed value functions with function approximation~\citep{schaul2015universal}, and hindsight relabeling, which reuses failed rollouts as successful examples for alternative achieved goals~\citep{andrychowicz2018hindsightexperiencereplay}. 
However, conditioning a flat policy on the final task goal does not by itself resolve long-horizon credit assignment. 
For this reason, hierarchical goal-conditioned RL (HGC) uses goal conditioning as a local control primitive: a high-level policy, planner, or task specification provides intermediate subgoals, and a low-level goal-conditioned policy learns to reach them~\citep{nachum2018data,levy2019learningmultilevelhierarchieshindsight, pmlr-v100-gupta20a,nair2018visualreinforcementlearningimagined}. 
This reduces the effective horizon, but introduce additional dependencies: they often assume a sequential subgoal structure, treat only one subgoal as active at a time, and require the deployed policy to be supplied with the correct current goal.  
These issues are costly and difficult when subgoal progress and completion must be evaluated by expert rules or large vision-language models during inference. 
\Vadar{} avoids this test-time dependency by learning a subgoal-agnostic policy; goal information enters only through the critic during training. 

\confpar{Value decomposition and reward-factorized critics.}
MAXQ decomposes the value of a hierarchical policy into child-task values and completion functions, reflecting the procedural structure of a subroutine hierarchy~\citep{dietterich1999hierarchicalreinforcementlearningmaxq}. Successor representations and successor features instead factor value prediction into predictive components and reward-dependent weights~\citep{dayan1993improving,barreto2017successor}. \Vadar{} follows this latter spirit of reward-factorized value estimation, but uses subgoal-completion signals to learn a vector-valued critic in an offline-to-online actor-critic pipeline. Closer to robotics, task-and-motion planning (TAMP) couples high-level symbolic task planners with low-level continuous motion planners over a fixed subtask vocabulary \citep{kaelbling2011hierarchical,garrett2021integrated}. \Vadar{} shares the assumption that the task admits a subgoal-level decomposition, but does not plan over symbolic actions or commit to a fixed sequencing at execution time: the decomposition enters only through the critic, and the deployed policy is a single flat, subgoal-agnostic controller.

\confpar{Representation learning for value-based RL.}
Recent work has shown that auxiliary prediction objectives can improve value learning by shaping the
representations used by the critic. MR.Q~\citep{fujimoto2025towards} learns state and state-action
representations that predict rewards, termination, and future latent states, using model-based objectives
to improve a model-free actor--critic without planning through the learned dynamics.
DR.Q~\citep{lyu2026debiasedmodelbasedrepresentationssampleefficient} builds on this approach by encouraging the learned representations to
retain information about future states through a mutual-information objective, together with a
recency-aware replay scheme. \Vadar{} is complementary to these approaches: rather than introducing
an auxiliary dynamics or representation-learning objective, it obtains additional supervision directly
from the RL problem by fitting subgoal value functions at \emph{multiple temporal resolutions}. Thus, the
shared critic representation is trained simultaneously by intermediate and long-horizon TD targets,
which we find can improve final-task value estimation even when the intermediate values are not used
for policy extraction.


\section{Preliminaries}
\label{sec:preliminaries}

\confpar{MDPs.} We formulate our problem as a \emph{Markov Decision Process} (MDP) $M := (\statespace, \actionspace, P_0, P, R,\gamma)$, with states $s\in\statespace$, actions $a\in\actionspace$, initial state distribution $P_0$, transition probabilities $P$,  reward $R$, and discount factor $\gamma \in (0,1)$. At each timestep $t$, the agent (e.g., robot) observes the state $s_t \in \statespace$, takes an action $a_t \sim \pi(a_t \mid s_t) \in\actionspace$, transitions to the next state $s_{t+1}$ according to $s_{t+1} \sim P(s_{t+1} \mid s_t, a_t)$ while receiving a reward $R(s_t, a_t)$.
For the MDP $\mdpenv$, we let $\Exp^{\pi,M}$ (resp. $\Pr^{\pi,M}$) denote the expectation (resp. probability distribution) over trajectories $(s_0, a_0, \dots, s_T, a_T)$ with length $T+1$, with initial state distribution $s_0 \sim P_0$ and transition operator $P$. The object of reinforcement learning (RL) is to train a policy to optimize the cumulative discounted return:
\iftoggle{arxiv}
{
\begin{equation}
\Jcost^M(\pi_\theta) = \mathbb{E}^{\pi_{\theta}}\!\left[\sum_{t\ge 0} \gamma^t R(s_t,a_t)\right].
\end{equation}
}
{$\Jcost^M(\pi_\theta) = \mathbb{E}^{\pi_{\theta}}\![\sum_{t\ge 0} \gamma^t R(s_t,a_t)]$.}
We also recall the Q-function $Q^{\pi,M}(s,a) := \Exp^{\pi}[\sum_{t \ge 0} \gamma^t R(s_t,a_t) \mid (s_t,a_t) = (s,a)]$ and value function $V^{\pi,M}(s) := \Exp_{a \sim \pi(s)}[Q^\pi(s,a)]$. When $M$ is clear from context, we can omit it from the subscript.

\confpar{Value-based Reinforcement Learning.}  
Value-based RL methods \citep{sutton2018reinforcement} simultaneously learn a policy $\pi_\theta$ and associated $Q_{\phi}$-function,   maintaining \footnote{ We defer critic ensembling and parameter exponential moving averages to Appendix  \ref{app:critic_ema}.} (1) a replay buffer $\Buffer=\{(s,a,s',r,\done)\}$ of past states $s$, actions $a$, successor states $s' \sim P(s,a)$, rewards $r$, and done indicators $\done$; (2) a policy $\pi_\theta$,   critic network $Q_{\phi}:\statespace \times \actionspace \to \R$ trained from this buffer. The parameters of  $Q_{\phi}$ are updated to approximate $Q^{\pi_{\theta}}$, and  $\pi_\theta(s)$ is updated to  maximize critic values. Specifically, updates are sampled $(s,a,s',r)\sim \Buffer$, and new actions $a' \sim \pi_{\theta}(s')$. Then,
\iftoggle{arxiv}
{performs the following updates
\begin{AIbox}{Value-Based RL Algorithm Template}
\begin{itemize}
    \item[1.] Updates each critic parameter $\phi$ via gradient descent on a given~\colorbold{critic loss} 
    \begin{align}
\Lcritic(Q_{\phi} \mid s,a,\ytarg)  = (Q_{\phi}(s,a) - \ytarg)^2 + \regu(s',a',\pi_\theta) \end{align} 
     which measure the squared Bellman residual to the future target 
     $\ytarg = r+ \gamma Q_{\phi}(s',a') $, and which may also incorporate some additional regularization $\regu(\cdot)$ depending on $s',a',\pi_\theta$ (omitted as arguments in the critic loss for simplicity).

    \item[2.]  Updates $\theta$ via gradient descent on a given \colorbold{policy loss} $\Lpol(\pi_\theta \mid s,Q)$  mapping from policies $\pi_\theta:S \to \Delta(A)$, states $s \in S$, and $Q$ to  $\R$. 
\end{itemize}
    
\end{AIbox}}
{
\begin{itemize}[leftmargin=10pt]
    \item[1.] \emph{Critic parameters}  $\phi$ are updated via gradient descent on a given~\termbold{critic loss} 
    \begin{align}
\Lcritic(Q_{\phi} \mid s,a,\ytarg)  = (Q_{\phi}(s,a) - \ytarg)^2 + \regu(s',a',\pi_\theta) \end{align} 
     which measures the squared Bellman residual to the future target 
     $\ytarg = r+ \gamma Q_{\phi}(s',a') $, and which may also incorporate some additional regularization $\regu(\cdot)$ depending on $s',a',\pi_\theta$ (omitted as arguments in the critic loss for simplicity).

    \item[2.]  \emph{Policy parameters} $\theta$ are updated via gradient descent on a given \termbold{policy loss} $\Lpol(\pi_\theta \mid s,Q)$  mapping from policies $\pi_\theta:S \to \Delta(A)$, states $s \in S$, and $Q$ to  $\R$. 
\end{itemize}
}

\confpar{Representative Algorithms: \SAC{}, \FQL{} and \OGPO{}.}
\relax{}
We instantiate \Vadar{} with Soft Actor-Critic (\SAC{})~\citep{haarnoja2018soft2}, Flow Q-Learning (\FQL{})~\citep{park2025flowqlearning}, and off-policy generative policy optimization (\OGPO{})~\citep{patil2026ogpo}. \SAC{} uses a maximum-entropy policy objective; \FQL{} and \OGPO{} support expressive generative control policies. These backbones allow us to test value disaggregation across different policy-extraction methods. Their policy objectives and algorithm-specific critic updates are detailed in \Cref{app:base_algorithms}.

\confpar{Sparse-Reward RL.} In many tasks, such as those encountered in robot manipulation, defining reward that densely measure progress towards goal completion is challenging \citep{amodei2016concreteproblemsaisafety,  krakovna2020specification}. Thus, we consider a \emph{sparse-reward} defined via the \emph{done} indicator $\done \in \{0,1\}$ indicating task completion. Standard practice is to set $r \gets \rsparse(\mathrm{done} )$, where 
\begin{align} \rsparse(\done) =  - 1 \cdot \I\{\done = 0 \} + \rdone \I\{\done = 1 \}, \label{eq:rdef}
\end{align}
and where $\rdone$ is a small positive constant (e.g. $\rdone = 5$). In words, $\rsparse$ penalizes each time-step until the task is not completed, and provides a small positive bonus for task success.

\confpar{Offline-to-Online RL.} In many robotics settings, environment interaction is expensive and a fixed dataset of demonstrations $\rboff$ is available. We follow the now-standard \emph{offline-to-online} recipe \citep{nair2021awacacceleratingonlinereinforcement}: pretrain a behavior-cloning policy $\pi_\textsc{bc}$ on $\rboff$, then initialize the online policy from  $\pi_\textsc{bc}$ and fine-tune online, optionally mixing $\rboff$ into the off-policy updates at a fixed ratio \citep{ball2023efficient}. All algorithms in our experiments share the same BC-pretraining and warm-start protocol, so comparisons isolate the online fine-tuning method.

\confpar{Dense Preference Reward.} An alternative approach to long-horizon RL is to learn dense preference-based rewards. We describe the formulation and implementation in \Cref{app:dense_preference_reward} and compare against this approach in \Cref{sec:experiments,sec:pref_vs_td}.

\section{Hierarchical Goal Conditioning and Its Limitations}
\label{sec:limits_gc}

\label{sec:hgc-prelim}

\begin{algorithm}[h]

\caption{Hiearchical Goal Conditioning RL (\hgcrl). }
\label{alg:gc-concise}
\begin{algorithmic}[1]
\State \textbf{Parameters:}  discount factor $\gamma$; learning rate $\eta$
\State \textbf{Initialize} critic networks $Q_{\phi}: S \times A \gccolor{\times G} \to \R$, policy $\pi_\theta$: $S \gccolor{\times G} \to \Delta(A)$,  buffer $\Buffer \gets \emptyset$ \label{line:hgc_init}

\For{episode = 1, 2, \dots}
    
    \While{episode not done}

        \State $a_t \gets \pi_{\theta}(s_t, \gccolor{g_t})$,  transition $s_{t+1} \sim P(s_t,a_t)$ \Comment{\gccolor{$g_t = \gcurr := \min\{g:\done_g = 0\}$)}} \label{line:hgc_act}
        \State $\Buffer \gets \Buffer \cup \{(s_t, a_t, s_{t+1}, r_{t}, \gccolor{g}_t)\}$, $r_{t} \gets \rsparse(\done_{\gccolor{g_t}})$ \label{line:hgc_goal_tracking_in_buffer}
        \For{number of updates}
        \State Sample $(s, a, s', r_{\gccolor{g}} \gccolor{,g})\sim \Buffer$, $a' \sim \pi_{\theta}(s')$
        \label{line:hgc_update}\State $\ytarg_{\gccolor{g}} = r_g + \gamma    Q_{\phi}(s',a',\gccolor{g}) \label{line:hgc_target}$\Comment{Target for \gccolor{current goal}}
        \State $\phi \gets \phi - \eta \nabla_{\phi}\Lcritic(Q_{\phi}(\cdot,\cdot,\gccolor{g}) \mid s,a,\ytarg_{\gccolor{g}} )$. \label{line:hgc_critic_update}
        
        \State $\theta \gets \theta - \eta \nabla_\theta \Lpol(\pi_\theta(\cdot, \gccolor{g}) \mid s,a,Q_{\phi}(\cdot,\cdot;\gccolor{g}))$ \label{line:hgc_actor_update}
        \EndFor
        \EndWhile
        \EndFor
\end{algorithmic}
\end{algorithm}

\emph{Long-horizon} 
reinforcement learning tasks   require many time steps to complete. Sparse-rewards make \emph{credit assignment} -  identifying which actions contribute to task success/failure - challenging. Thus, prior work introduces \termbold{hierarchical goal conditioning} (\hgc)\footnote{This same practices goes by other names in other communities; see \Cref{sec:related} for further discussion.}, which 
decomposes the problem into a sequence of \emph{subgoals} $G = \{g \}_{g \in G}$. Traditional goal-conditioned RL assumes that the completion of \emph{individual subgoals}, what we call \colorbold{subgoal-dones} \iftoggle{arxiv}
{\begin{align}
\done_g(s) := \I\{ g \text{ is completed at state } s \},
\end{align}}
{$\done_g(s) := \I\{ g \text{ is completed at state } s \},$}
can be evaluated at both {training time} time and during evaluation. 
Traditionally, $\done_g$ have been hand-specified, but modern VLMs can now evaluate $\done_g$'s  specified by natural language queries \citep{chen2026robovlm}.  
\iftoggle{arxiv}
{

}
{}
 For simplicity, in the body of this paper we assume these goals satisfy a linear order with $g^1 \prec g^2 \dots \prec g^{|G|} \in G$, where $g^{|G|}$ denotes final task completion and $g^1$ is the first goal (Appendix \ref{app:dag_extension}) easily removes this sequentiality). Under this sequentiality condition, we define the current goal $\gcurr = \min\{g:\done_g = 0\}$,  where $\min$ is taken with respect to the order $\prec$.
\Cref{alg:gc-concise} depicts how value-based RL methods can be modified to incorporate \hgc{}. These modifications of standard value-based RL, depicted in \gccolor{blue}, are: policies and critics take the current goal \gccolor{$g_t = \gcurr$} as an argument (Line~\ref{line:hgc_init}), so that in particular, the policy needs $g_t$ to select actions (Line~\ref{line:hgc_act}). This necessitates that \gccolor{$g_t$} is tracked in the replay buffer (Line \ref{line:hgc_goal_tracking_in_buffer}), is used to determine the critic target (Line \ref{line:hgc_target}), and features in both the policy and critic loss (Lines \ref{line:hgc_critic_update},~\ref{line:hgc_actor_update}). However, we identify 3 limitations of \hgcnc, outlined below.



\relax{}
\hypertarget{lim-hgc-test-time-goals}{\textbf{Limitation 1}}: Accurate goals must be supplied to the policy at test-time, which is prohibitive if sub-goal completions requires expert engineering or larger models to evaluate. \hypertarget{lim-hgc-myopia}{\textbf{Limitation 2}}: Optimizing for the current goal can be overly myopic if future actions depend subtly on \emph{how} the current goal is completed, such as in robotic pick and place tasks. \hypertarget{lim-hgc-sequential}{\textbf{Limitation 3}}: \hgc{} often requires goals to have a sequential structure, with only one active goal at a time.

We examine these limitations in \Cref{app:hgc-limitations}. As a point of comparison, we introduce implicit hierarchical goal conditioning (\ihgc), which retains goal-conditioned critics but uses a subgoal-independent policy $\pi(a\mid s)$, removing the need for test-time goal information (\hyperlink{lim-hgc-test-time-goals}{Limitation 1}). Although \ihgc{} can outperform traditional \hgc{} as shown in \cref{fig:ihgc-hgc-toolhang-fql}, both optimize the policy against the current-subgoal critic alone and remain susceptible to myopic behavior (\hyperlink{lim-hgc-myopia}{Limitation 2}). This motivates \Vadar{}, which incorporates future-subgoal values during policy optimization.


\section{The Value Disaggregation (\Vadar) Framework
\label{sec:reward_disagg}}


\begin{figure}[h]
    \centering
    \vspace{-1em}
\includegraphics[width=0.9\linewidth]{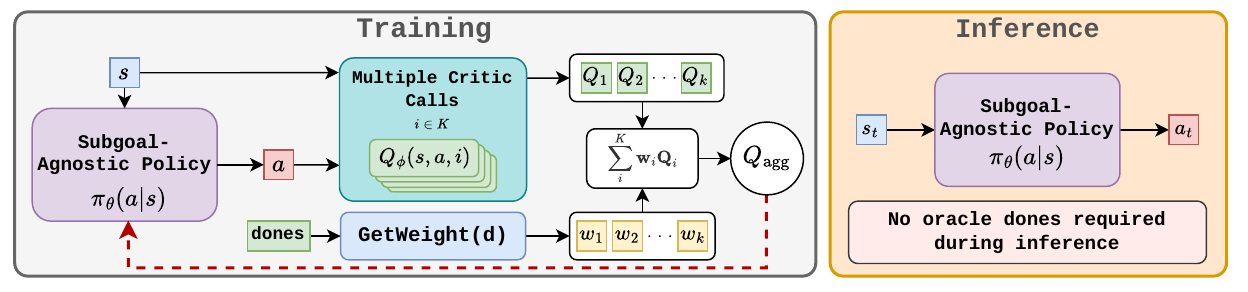}
    \caption{\textbf{\Vadar{} training and inference.}
  \textbf{Left:} Subgoal completion indicators weight the per-subgoal critic values, whose aggregate guides updates to a subgoal-independent policy.
  \textbf{Right:} The policy selects actions directly from the state, without subgoal information.}
    \label{fig:vadar-training-inference}
\end{figure}
\vspace{.5em}
We now introduce Value Disaggregation (\Vadarnc, \Cref{alg:vd-concise}), a framework for subgoal-driven long-horizon RL that addresses the limitations of goal conditioning described in \Cref{sec:limits_gc}. \Vadar{} combines three \textbf{key insights}: (1) Like implicit hierarchical goal conditioning (\ihgc), goal-conditioned critics can train a single, subgoal-independent policy. (2) Leveraging \emph{weighted combinations of critics} can balance current and future goal completion, thereby reducing the myopic overexploitation described in \hyperlink{lim-hgc-myopia}{Limitation 2}. Finally, (3) The extraction from multiple critics is naturally compatible with general goal topologies, addressing \hyperlink{lim-hgc-sequential}{Limitation 3} by allowing more than one subgoal to be active at any given time.




\begin{algorithm}[H]
\caption{Value Disaggregation (\Vadar)}
\label{alg:vd-concise}
\begin{algorithmic}[1]
\State \textbf{Parameters:}  discount factor $\gamma$; learning rate $\eta$; \modcolor{TD clipping parameter $\cclip$.}
\State \textbf{Initialize} goal-conditioned critic $Q_{\phi} : S \times A \times G \to \R$; policy $\pi_\theta$: $S \to \Delta(A)$; replay buffer $\Buffer \gets \emptyset$. \label{line:vadar_multi_subgoal_critic}
\For{episode = 1, 2, \dots}
    \While{episode not done}
        \State $a_t \gets \pi_{\thetatarget}(s_t)$,  transition $s_{t+1} \sim P(s_t,a_t)$
        \State Evaluate subtask completion \modcolor{vector: $ \bd_t = \{\done_g(s_{t+1})\}_{g \in G} \in \{0,1\}^G$} \label{line:vadar_get_dones}
        \State $\Buffer \gets \Buffer \cup \{(s_t, a_t, s_{t+1}, \modcolor{\br_t, \bd_t}\}$, where \modcolor{$\br_t[g] \gets  \rsparse(\bd_t[g])$}  \For{number of updates}
        \State Sample $(s, a, s', \modcolor{\br, \bd})\sim \Buffer$, $a' \sim \pi_{\theta}(s')$
        \State $\bytarg[g] = \modcolor{\mathrm{min}\{}\modcolor{\br}[g] +   \gamma {Q_{\phi}}(s',a',g),\,\modcolor{\cclip \cdot \rdone\}}$, for all $g \in G$ \Comment{\modcolor{all-goal} targets} \label{line:vadar_reward_clipping}
        \State $\phi \gets \phi - \eta \nabla_{\phi}\modcolor{\sum_{g \in G}}\Lcritic(Q_{\phi}(\cdot,\cdot,g) \mid s,a,\bytarg[\modcolor{g}])$. \Comment{\modcolor{multi-goal} critic loss}
        \State $\Qagg(\cdot,\cdot) \gets \modcolor{\sum_{g}\bw[g] }Q_{\phi}(\cdot,\cdot,g)$, where \modcolor{$\bw \in \R^{G} \gets \getWeight(\bd)$} \label{line:vadar_get_Q_weights}
        \State $\theta \gets \theta - \eta \nabla_\theta \Lpol(\pi_\theta \mid s,a,\Qagg)$
\EndFor
\EndWhile
\EndFor
\end{algorithmic}
\end{algorithm}

Conveniently, 
\Vadarnc{} requires \textbf{minimal modification} to standard off-policy RL pipelines, wherein only the critic architecture (conditioned on subgoals) and the corresponding aggregation strategies change.
We now walk through \Cref{alg:vd-concise} in detail, highlighting the components unique to \Vadarnc{} (shown in \modcolor{color} in the pseudocode and below). We give a detailed \textbf{practitioners guide} to using \Vadarnc{} in  \Cref{app:practitioners_guide}. In what follows, we use brackets $[g]$ to denote goal-wise indexing. 

\textbf{Policy update via value weighting.} The key innovation in
\Vadarnc{} is optimizing for weighted  contributions of different critics during policy extraction, which addresses the myopic behavior described in \Cref{sec:limits_gc} and illustrated in \Cref{app:hgc-limitations}.  At each environment step, \Vadarnc{} evaluates a \modcolor{vector} of subgoal completion indicators $\bd_t = \{\done_g(s_{t+1})\}_{g \in G} \in \{0,1\}^G$. It uses these to  computing a weighting $\getWeight$ (Line~\ref{line:vadar_get_Q_weights}) for policy extraction:
\begin{align}
   \confinlinestyle
   \Qagg(s,a)
   = \sum_{g \in G} \bw[g] \cdot Q_\phi(s,a,g),
   \qquad
   {\bw \gets \getWeight(\bd)} .
   \label{eq:q_agg}
   \confinlinestyle
\end{align}
The policy is then updated to maximize $\Qagg$ via the policy loss
$\Lpol(\pi_\theta \mid s,\Qagg)$. Because the policy
$\pi_\theta : S \to \Delta(A)$ is never conditioned on $g$
(\emph{subgoal agnostic}), subgoals are not needed at test time.
A good weighting function provides strong training signal for completing current
goals while improving behavior for future goals.

\begin{table}[!ht]
\iftoggle{neurips}{\vspace{-1.25em}}{}
\centering
\small
\setlength{\tabcolsep}{3.5pt}
\renewcommand{\arraystretch}{1.08}
\iftoggle{neurips}{%
  \caption{Value-weighting variants used for policy extraction in \Vadar{}.}
  \label{tab:value-weighting}
}{}
\begin{tabular}{@{}lll@{}}
\toprule
\textbf{Config.} & \textbf{Weight} $\bw[g]$ & \textbf{Interpretation} \\
\midrule
Current-myopic
& $\mathbb{I}[g = \gcurr]$
& Only the current subtask contributes. \\
Current-future
& $\mathbb{I}[g \succeq \gcurr]$
& All incomplete subtasks contribute. \\
Next-myopic
& $\mathbb{I}[g = \gnext]$
& Focus on the next subtask, gated by the current subtask. \\
Next-future
& $\mathbb{I}[g \succeq \gnext]$
& Incomplete subtasks with incomplete parents contribute. \\
\bottomrule
\end{tabular}
\iftoggle{neurips}{}{%
  \caption{Value-weighting variants used for policy extraction in \Vadar{}.}
  \label{tab:value-weighting}
}
\end{table}

\confpar{Weighting choices for sequential goals.}
As in \Cref{sec:preliminaries}, first assume that goals are ordered as $1,\dots,|G|$, with goal
$g^i$ serving as a prerequisite for goal $g^{i+1}$, denoted $g^i \prec g^{i+1}$. (see \Cref{app:dag_extension} for the general case. Given a vector of completion
indicators $\bd$, recall that
$\gcurr =  \inf\{g : \bd_{g} = 0\}$ denotes the current unfinished goal. 
\par
\Needspace{14\baselineskip}
\begin{wrapfigure}{r}{0.35\textwidth}
  \centering
  \vspace{-1em}
  \includegraphics[width=\linewidth]{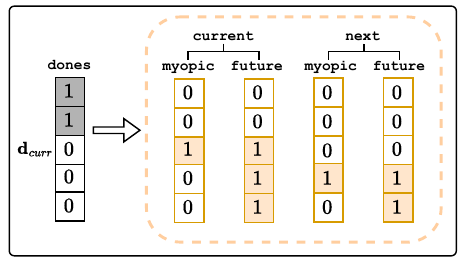}
  
  \caption{Subgoal weighting used by $\getWeight$.}
  \label{fig:subgoal-weighting}
  \vspace{-1em}
\end{wrapfigure}
We also define the next goal as $\gnext := \min\{g^+, g^{|G|}\}$, where $g^+  := \inf\{g':g' \succ \gcurr\}$ ; again, $\min$/$\inf$ is taken with respect to the goal order $\prec$, and  $\succeq$ in \Cref{tab:value-weighting} denotes the negation of $\prec$. \Cref{tab:value-weighting} summarizes four choices for the weighting function along two axes. The first axis distinguishes \emph{myopic} weightings, which optimize only a single goal, from \emph{future-aware} weightings, which include all relevant future goals. The second axis distinguishes \emph{current} weightings, which include the current goal, from \emph{next} weightings, which only optimize for goals after the current goal. These choices are illustrated in \Cref{fig:subgoal-weighting}. Detailed experiments can be found in \cref{sec:understanding_vadar}.



\iclrendwrap

\textbf{All-Subgoal Critic Learning}
Importantly, the weighted extraction \Cref{eq:q_agg} requires that the critic values $Q(s,a,g)$ for \emph{each goal} are well-estimated. Therefore, we design an all-subgoal critic update. Recall that, at each environment step, \Vadarnc{} evaluates a \modcolor{vector} of subgoal completion indicators $\bd_t$,  and computes a corresponding reward vector $\br_t = \rsparse(\bd_t)$ entrywise via the sparse reward function (\Cref{eq:rdef}, Line \ref{line:vadar_get_dones}). This converts a  sparse task reward into $|G|$  per-subgoal rewards (Line \ref{line:vadar_reward_clipping}), each of which provides signal as soon as its corresponding subgoal is completed. The tuple $(s_t, a_t, s_{t+1}, \modcolor{\br_t, \bd_t})$ is stored in the replay buffer. Given a sampled transition $(s, a, s', \modcolor{\br, \bd})$ from the buffer and a new action $a' \sim \pi_\theta(s')$, the \modcolor{vector} of TD targets is computed entrywise:
\begin{align}
    \bytarg[g] = \modcolor{\min\{} \modcolor{\br[g]} + \gamma\, \modcolor{Q_\phi}(s', a', g),\; \modcolor{\cclip \cdot \rdone[g]} \modcolor{\}} \hspace{0.25in} \forall g \in G, \label{eq:vector_td} \confinlinestyle
\end{align}
where the $\min$ is applied entrywise. The critic is updated by a gradient update on the critic loss across all subgoals: $\sum_{g \in G} \Lcritic(Q_\phi(s, a, g) \mid s, a, \ytarg_g)$. 

\confpar{Why clipping?} The critic update involves a ``clipping'' of the maximum critic value. This is because, 
after a subgoal is completed, the agent keeps collecting $\rdone$ at every step, so without the clip the Q-value would accumulate to $\rdone / (1-\gamma)$. The clip caps the target at $\cclip \cdot \rdone$ (Line \ref{line:vadar_reward_clipping}), capping critic values for completed subgoals. 

\begin{AIbox}{\Vadarnc{}: Key Design Decisions}
\begin{itemize}[leftmargin=2.5mm]
    \item \textbf{Subgoal-agnostic policy.} Subgoal information enters only through the critic at training time;  deployed policy $\pi_\theta(a \mid s)$ needs no goal selection/conditioning at inference.
    \item \textbf{Weighted aggregation for policy extraction.} Per-subgoal Q-values are aggregated into a single scalar signal $\Qagg$; the next-future (\VadarNF{}) weighting is our default method which balances current and future subgoal completion.
    \item \textbf{Reward clipping at $\cclip \cdot \rdone$.} Removes the need for critic loss gating on completed subgoals.

\end{itemize}
\end{AIbox}
\vspace{.3em}
 \Vadar{} benefits from both dense critic supervision and weighted policy extraction (\Cref{sec:exp-lastq,sec:exp-reward-shaping}). We provide further details on
  non-sequential tasks in \Cref{app:dag_extension}, multi-objective optimization in \Cref{app:multiobj-pushT}, and alternative critic parameterizations in
  \Cref{app:critic-parameterizations}. Finally, \Cref{sec:reward-shaping-theory} shows that fixed-weight value aggregation is \emph{theoretically (but only
  theoretically)} equivalent to reward shaping but we later show (\Cref{sec:exp-reward-shaping}) that this equivalence does \emph{not} explain \Vadar{}'s empirical success.




\newcommand{\densecolorr}{orange!80!black}
\newcommand{\densecolor}[1]{{\color{\densecolorr} #1}}

\section{Experiments}
\label{sec:experiments}
We evaluate \Vadar{} on a suite of long-horizon manipulation benchmarks. Each method is pretrained via behavior cloning on offline data, stopped early at a low success rate so all methods begin fine-tuning from a comparable starting point, and then fine-tuned online for a fixed number of environment steps. 
\iftoggle{arxiv}
{The remainder of this section is structured as follows: \Cref{sec:exp-long-horizon} compares \Vadar{} to \hgc{} and to the sparse-reward backbone across all benchmarks. \Cref{sec:exp-understanding-vadar-gains} steelmans the comparison with a spectrum of intermediate variants bridging \hgc{} and \Vadar{}. \Cref{sec:exp-iterative} studies the synergy between \Vadarnc{} and iterative critic architectures. \Cref{sec:pref_vs_td} compares \Vadarnc{} to dense preference rewards, and \Cref{sec:exp-multiobj} extends \Vadarnc{} to multi-objective rewards.
}{}
Per-task descriptions and hyperparameters are detailed in \Cref{app:task_descriptions}. 

\confpar{Baselines.}
We  study \Vadar{} with two different weighting schemes  (\Cref{tab:value-weighting}) - \VadarNF{}, using \emph{next-future} (plotted in a solid line) and \VadarCF{}, which use \emph{current-future} (plotted in a dotted line). We compare against 3 alternatives, each using the same base-algorithm per plot: (1) \textbf{Hierarchical Goal Conditioning} (\hgc{}), where both the policy and critic are conditioned on the current subgoal $g$, where $g$ is set to the next incomplete subgoal at every step 
(\Cref{alg:gc-concise}).
(2) \textbf{Implicit Hierarchical Goal Conditioning }(\ihgc{}), which retains the goal conditioned critic of \hgc{} but does not condition the policy on the subgoal.
(3)
\textbf{Dense Preference Reward}, denoted \densecolor{\textmd{Dense}}, where we adopt best-practices for a learning dense reward model via preferences (described in \Cref{app:dense_preference_reward}
), where the reward models are pretrained with additional rollouts on a richer dataset of successful and failed runs in order to steelman their performance.
(4) The \textbf{original} base-algorithm that using sparse reward with no goal information, denoted as ``\sparsecolor{\textmd{Sparse}}''.
{}
\iftoggle{neurips}
{\vspace{-2.2em}}{}  

\textbf{Environments.}
\relax{}
We evaluate on four benchmark suites spanning different task horizons and difficulties. \Robocasa{}~\citep{nasiriany2026robocasa365} provides our three longest-horizon tasks, while \Robomimic{}~\citep{robomimic2021} tests contact-rich, medium-horizon manipulation on \Transport{} and \ToolHang{}. We also evaluate on three multi-stage kitchen tasks from \Libero{}~\citep{liu2023liberobenchmarkingknowledgetransfer} and the long-horizon \texttt{scene-play-singletask-task4} task from \OGBench{}~\citep{park2025ogbenchbenchmarkingofflinegoalconditioned}. Task descriptions, subgoals, and their dependency structures are given in \Cref{app:task_descriptions}.
{}

\confpar{Base-algorithm instantiations.} We instantiate all above baselines with three value-based RL algorithms (details  given in \Cref{app:base_algorithms}).
On the most challenging long-horizon environments (\Robocasa{}), we use \OGPO{}~\citep{patil2026ogpo}, a state-of-the-art offline-to-online RL algorithm for flow-based policies. On the medium-horizon \Robomimic{} tasks, we evaluate on both \OGPO{} and the popular, though less performant \FQL{}~\citep{park2025flowqlearning}, demonstrating that the benefits of \Vadar{} apply to both.  Because \OGPO{} is considerably more sample efficient that than \FQL{}  from \sparsecolor{\textmd{Sparse}}  reward when all BC data are used, we also test \OGPO{} in a regime with limited BC data, denoted \textbf{low data}, on \Robomimic{}; details are described in \Cref{tab:hp-low-data}. 
Lastly, we demonstrate \Vadar{}'s applicability to non-generative policies with  the seminal Soft Actor Critic (\SAC{})  algorithm \citep{haarnoja2018soft2} the simpler  on \OGBench{} and \Libero{} environments.

\begin{figure}[h]
    \centering
    \vspace{-1em}
    \includegraphics[width=0.98\linewidth]{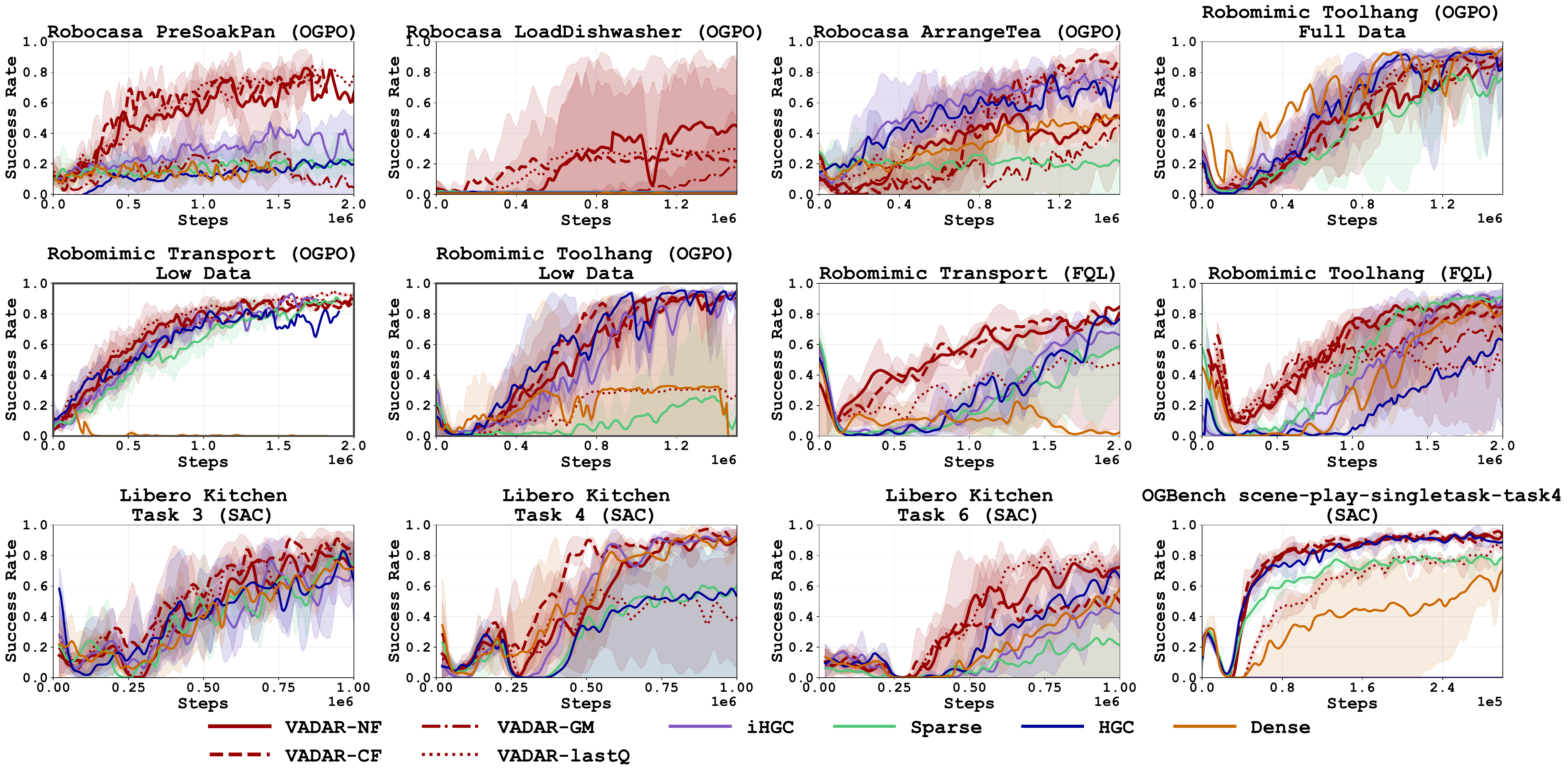}
    \caption{Success rate of \VadarNF{} (solid) and \VadarCF{} (dotted) against \hgc{}, its implicit variant \ihgc{}, the \densecolor{\textmd{Dense}} preference-reward baseline, and the \sparsecolor{\textmd{Sparse}} backbone across \Robocasa{}, \Robomimic{}, \Libero{} and \OGBench{}, for our three instantiations \OGPO{}, \FQL{} and \SAC{}.}
    \label{fig:main-results}
    
\end{figure}

\subsection{\Vadar{} outperforms natural baselines}
\label{sec:exp-long-horizon}

\begin{figure}[t!]
    \centering
    \includegraphics[width=\linewidth]{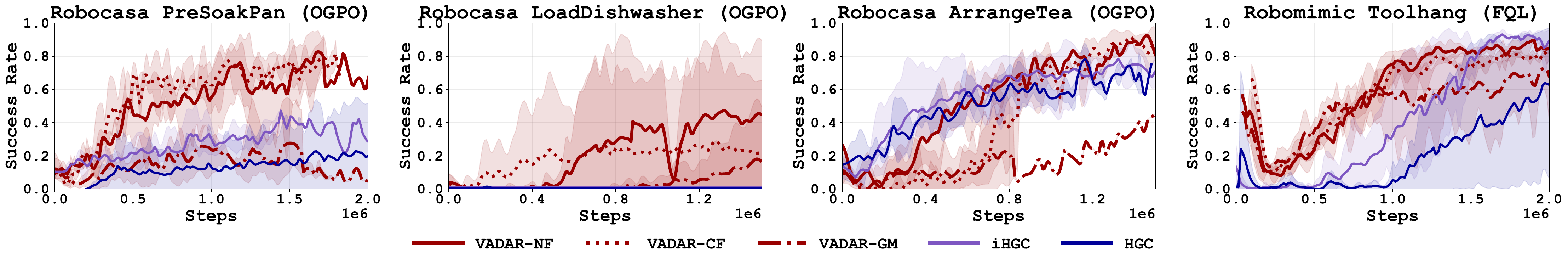}
    \caption{Weighting ablation of \Vadar{} (\VadarGM{}, \VadarCF{}, \VadarNF{}) against \hgc{} and \ihgc{} on \Robocasa{} with the \OGPO{} backbone and \Robomimic{} \ToolHang{} with the \FQL{} backbone. }
    \label{fig:gating}
    \vspace{-1em}
\end{figure}

\Cref{fig:main-results} shows that \Vadar{} improves various policy-extraction methods, with the largest gains on the most challenging long-horizon tasks. Not only does \Vadar{} benefit a SOTA method like \OGPO{}, but also enchances weaker methods like \FQL{} and \SAC{} with improved critic learning.  

\confpar{\Vadar{} supplies useful values for aggressive policy extraction.}
\OGPO{} aggressively optimizes critics using PPO-style policy extraction on the generative policy. We find that \Vadar{} with \OGPO{} matches or outperforms all baselines, with the largest gains on the most challenging long-horizon tasks (\Cref{fig:main-results}). On \Robocasa{}, \Vadar{} is the only method to achieve convergence on \loaddishwasher{}, while on \presoakpan{} it reaches nearly double the success rate of the strongest baseline, \ihgc{}. Similarly, on \arrangetea{}, \Vadar{} outperforms baselines with superior training stability. On \Robomimic{} with \OGPO{}, we find that \Vadar{}, \hgc{} and \ihgc{} obtain similar performance in both the low- and full-data regimes. 

\confpar{The benefits extend to first-order policy extraction.}
\Vadar{} also improves learning with \FQL{} and \SAC{}, which update the policy through action gradients of the critic. With \FQL{}, \Vadar{} learns faster on \Robomimic{} \Transport{} and \ToolHang{}, demonstrating that its benefits do not require \OGPO{}'s policy update. With \SAC{}, the advantage is most apparent on the longer-horizon \Libero{} tasks, whilst performance remains comparable to goal-conditioned baselines on \OGBench{}. Together, these comparisons support the use of value disaggregation across different policy-extraction methods. Even where performance is comparable, \Vadar{} retains the practical benefit of requiring no goal selection at inference.

\confpar{Additional baseline comparisons.}
In \Cref{app:additional-baselines}, we compare to two additional algorithm families, both of which significantly underperform \Vadar{}. First, we test hindsight experience replay (HER) \citep{andrychowicz2018hindsightexperiencereplay} on contact-rich \Libero{} manipulation, including variants supplied with the same subgoal information as \hgc{}. Second, we compare against reward learning with access to subgoal information, training separate preference-based reward models for each subgoal and aggregating with next-future weighting. 
{}
\confpar{VLM-provided subgoal labels.}
On \Robocasa{} \presoakpan{}, \Vadar{} also learns with VLM-provided subgoal labels, reaching similar final success to oracle supervision with only modest effects on sample efficiency (\Cref{sec:vlm-judge}).


\vspace{-.3em}
\subsection{Understanding What Enables \Vadar{}'s Performance}

\begin{figure}[t!]
\centering

\begin{minipage}[t]{0.60\linewidth}
    \vspace{0pt}
    \centering
    \scriptsize
    \renewcommand{\arraystretch}{1.12}

    \resizebox{\linewidth}{!}{%
    \begin{tabular}{lccc}
        \toprule
        \textbf{Variant} & \textbf{Critic supervision}
        & \textbf{Policy weighting} & \textbf{Policy} \\
        \midrule
        \hgc{} (reference)
            & current subgoal & current-myopic & $\pi(a\mid s,g)$ \\
        \ihgc{} (\Cref{sec:limits_gc})
            & current subgoal & current-myopic & $\pi(a\mid s)$ \\
        \steelmancolor{\textmd{Steelman 1}}
            & all subgoals & current-myopic & $\pi(a\mid s,g)$ \\
        \steelmancolor{\textmd{Steelman 2}}
            & all subgoals & current-future & $\pi(a\mid s,g)$ \\
        \steelmancolor{\textmd{Steelman 3}}
            & all subgoals & next-future & $\pi(a\mid s,g)$ \\
        \VadarGM{}
            & all subgoals & current-myopic & $\pi(a\mid s)$ \\
        \VadarCF{} (ours)
            & all subgoals & current-future & $\pi(a\mid s)$ \\
        \VadarNF{} (ours)
            & all subgoals & next-future & $\pi(a\mid s)$ \\
        \bottomrule
    \end{tabular}%
    }
\end{minipage}\hfill
\begin{minipage}[t]{0.38\linewidth}
    \vspace{0pt}
    \centering
    \includegraphics[width=\linewidth]{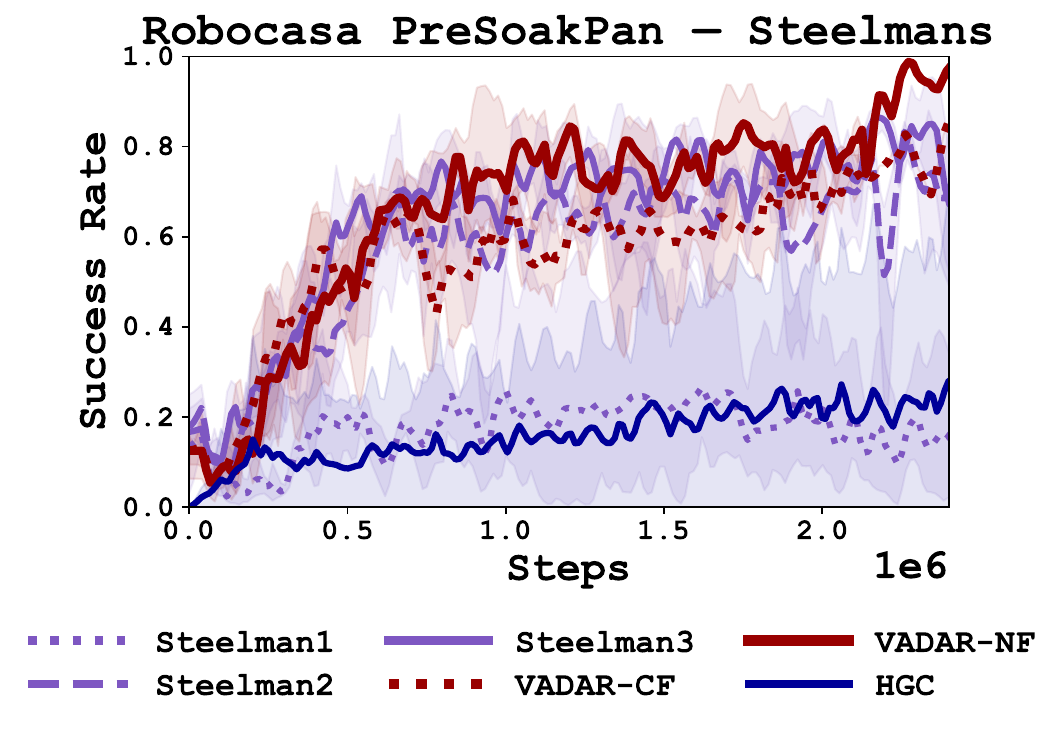}
\end{minipage}\par
\nointerlineskip

\begin{minipage}[t]{0.60\linewidth}
    \vspace{0pt}
    \captionof{table}{Variants bridging \hgc{} and \VadarNF{},
    varying critic supervision, policy weighting
    (\Cref{tab:value-weighting}), and policy goal conditioning.}
    \label{tab:steelman-variants}
\end{minipage}\hfill
\begin{minipage}[t]{0.38\linewidth}
    \vspace{0pt}
    \captionof{figure}{\steelmancolor{Steelman} sweep on
    \Robocasa{} \presoakpan{} with the \OGPO{} backbone.}
    \label{fig:steelmans-final}
\end{minipage}
\vspace{-1em}
\end{figure}
\textbf{Future-Aware Policy Extraction Improves Performance}
\label{sec:understanding_vadar}
\label{sec:exp-understanding-vadar-gains}
To isolate the effect of policy extraction, we compare the current-myopic (\VadarGM{}), current-future (\VadarCF{}), and next-future (\VadarNF{}) weightings described in \Cref{tab:value-weighting}. These ablations share all-subgoal critic supervision and goal-agnostic policies. \Cref{fig:gating} shows that \VadarGM{} does not outperform \hgc{} and \ihgc{}, while the future-aware variants improve sample efficiency and are the only methods to converge on \loaddishwasher{}. Hence, dense critic supervision alone is insufficient when used with myopic policy extraction. 
{}
Next, we test whether future-aware policy extraction also improves \hgc{} while retaining policy goal conditioning. On \Robocasa{} \presoakpan{} with \OGPO{}, we introduce three \emph{Steelman} variants, namely 
\steelmanx{1} using current-myopic, \steelmanx{2} using current-future, and \steelmanx{3} using next-future, extraction respectively. \cref{fig:steelmans-final} shows that \steelmanx{1} remains close to \hgc{}, while \steelmanx{2} and \steelmanx{3} approach \VadarCF{} and \VadarNF{}. These comparisons show that future-aware extraction benefits both policy formulations on this task, allowing \Vadar{} to retain these benefits without requiring subgoal selection at inference. 

\vspace{-.3em}

  \begin{figure}[h!]
      \centering
      \includegraphics[width=\linewidth]{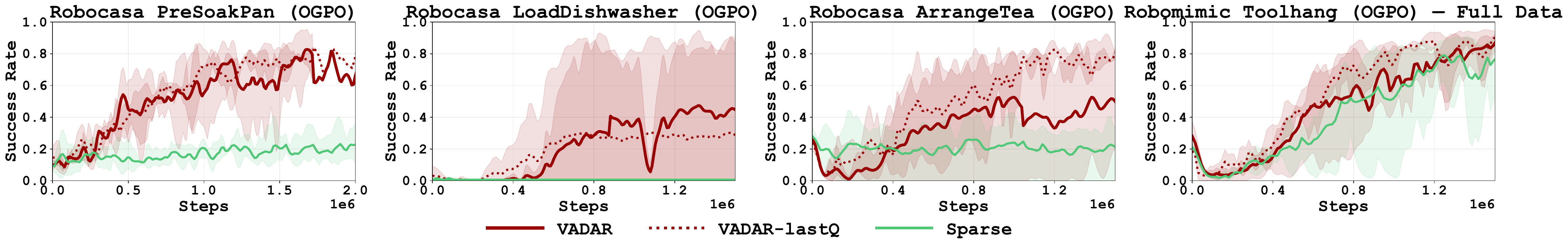}
      \caption{Success rate of \LastQ{} compared with \VadarNF{} and \sparsecolor{\textmd{Sparse}} on selected tasks. \Cref{fig:lastq-full} gives the results across all benchmarks and instantiations.}
      \label{fig:lastq}
    \end{figure}
    
\textbf{Dense Subgoal Supervision Improves Critic Learning}
\label{sec:exp-lastq}. To further test whether dense supervision across multiple temporal resolutions improves critic learning, we ask a stricter question: \emph{can intermediate subgoal supervision improve the final-task critic even when those intermediate values are never used for policy extraction?} We introduce \LastQ{}, which trains the same multi-goal critic as \Vadar{} (\Cref{sec:reward_disagg}) but extracts the policy using only the final-subgoal value $Q_\phi(s,a,g^{|G|})$. Thus, \LastQ{} uses the same policy objective as the \sparsecolor{\textmd{Sparse}} baseline, while additionally supervising the critic on intermediate subgoals.
{}
\Cref{fig:lastq} shows that \LastQ{} approaches \VadarNF{} on several tasks and substantially outperforms \sparsecolor{\textmd{Sparse}} on the \Robocasa{} benchmarks. Because the policy objective is unchanged, these gains isolate the effect of denser critic supervision. Full results across all benchmarks and policy-extraction methods are provided in \Cref{app:lastq-full}. This suggests that intermediate subgoal targets improve the critic even when their values are not used directly for policy extraction. 

\confpar{Effect of the policy extraction method.} While \LastQ{} sometimes matches \VadarNF{}, the size of this gap depends on the policy-extraction method. \LastQ{} is generally closer to \VadarNF{} under \OGPO{}, where the critic-derived signal supervises every denoising step of the generative policy. In \FQL{}, where the critic provides less dense supervision than \OGPO{}, \LastQ{} alone is not sufficient, and Next-future weighting provided additional gains.

\begin{AIbox}{Takeaway: Intermediate supervision improves final-task learning} \LastQ{}, which uses \Vadar{} for critic supervision, but only supervises on the \textbf{final subgoal},  improves over \sparsecolor{\textmd{Sparse}} on the long-horizon \Robocasa{} tasks even though intermediate subgoal values never enter its policy objective directly. This supports our hypothesis that dense subgoal supervision improves critic learning. Next-future weighting provides additional gains where final-subgoal extraction alone remains insufficient.
\end{AIbox}

\relax{}
\relax{}

\textbf{Further findings.} We further compare against \VadarShaped{}, which aggregates subgoal rewards before the Bellman backup and trains a scalar critic. It fails to learn on the harder \Robocasa{} tasks and is less sample efficient than \VadarNF{} elsewhere (\Cref{sec:exp-reward-shaping}). These and the \LastQ{} results  suggests that reward shaping alone does not explain the gains from all-subgoal critic supervision.
\confpar{Multi-objective optimization.}
\Vadar{} also applies when reward components represent distinct objectives. As a proof of concept, \Cref{sec:exp-multiobj} shows how to train separate critics for task completion and obstacle avoidance
improving performance and training stability over \SAC{} trained on the aggregated rewards with optimal reward weighting. 

\relax{}
\relax{}

\subsection{Test-Time Steering with a Frozen Policy}
\label{sec:exp-steering}

We hypothesize that \Vadar{}'s success comes from learning a better critic through denser supervision across multiple temporal resolutions. We test this directly by asking: \emph{does the learned critic better identify actions that lead to task success?} To isolate critic quality from policy learning, we fix a behavior-cloned policy and update only the critic. Following Q-chunking~\citep{li2025reinforcement}, we use best-of-$N$ steering: a frozen behavior-cloned policy proposes $N$ candidate action chunks, and the critic selects the highest-valued candidate. We compare \VadarNF{} and \sparsecolor{\textmd{Sparse}} critics trained on the same demonstrations and frozen-policy rollouts. For this steering experiment, next-future aggregation uses subgoal-completion indicators at inference; the unconditional policy learned in our main experiments requires no such information.

\begin{figure}[!htbp]
    \centering
    \includegraphics[width=\linewidth]{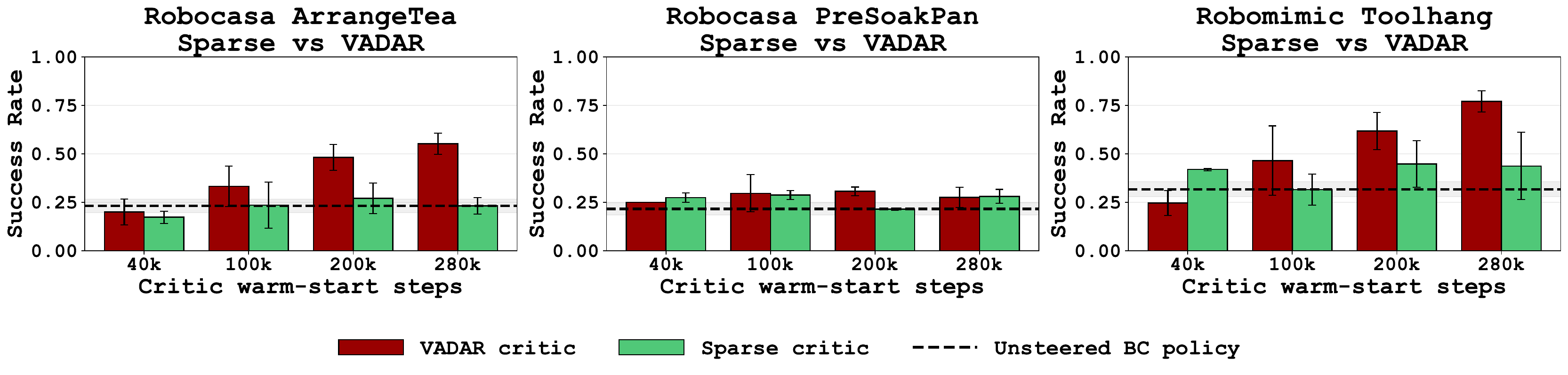}
    \caption{Best-of-$N$ steering ($N=4$) of a frozen BC policy with \VadarNF{} and \sparsecolor{\textmd{Sparse}} critics as critic training progresses. The dashed line marks the unsteered policy's success rate. We use the BC-pretrained flow policy from \OGPO{}, without online policy updates.}
    \label{fig:steering}
\end{figure}

\Cref{fig:steering} shows that the \Vadar{} critic improves action selection most clearly on \arrangetea{} and \ToolHang{}. At the largest critic-training budget with $N=4$, success rises from 23\% to 55\% on \arrangetea{} and from 32\% to 77\% on \ToolHang{}, compared with 23\% and 44\% for the \sparsecolor{\textmd{Sparse}} critic. Gains on \presoakpan{} are smaller. These results support the benefit of subgoal supervision for critic learning independently of policy updates. \Cref{app:steering} provides the full protocol, seed variability, candidate-count sweep, and critic visualizations.

\subsection{Hypothetical Mechanism: Implicit Bellman-Backup via  Dense Critic Supervision}
\label{sec:exp-bellman}

\noindent The frozen-policy experiments show that subgoal supervision improves action selection independently of policy updates. Together with the \LastQ{} results, they motivate a further question: \textbf{\emph{how can intermediate subgoal targets improve value predictions for the final task?}} We hypothesize that shared representations allow supervision at shorter horizons to inform value prediction at longer horizons. We refer to this proposed mechanism as an \colorbold{implicit Bellman backup}.
\begin{wrapfigure}{r}{0.36\textwidth}
    \centering
    \includegraphics[width=\linewidth]{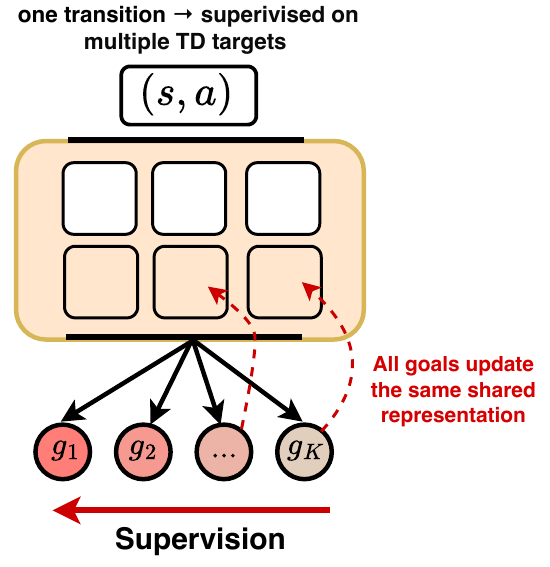}
    \caption{All-subgoal TD targets update shared critic representations.}
    \label{fig:implicit-bellman}
\end{wrapfigure}
 For a fixed policy, the classical Bellman recursion is
\begin{equation}
    \begin{aligned}
    Q^\pi(s,a) &= r(s,a)\\
    &\quad + \gamma\,
    \mathbb{E}_{s'\sim P(\cdot\mid s,a),\,a'\sim\pi(\cdot\mid s')}
    \big[Q^\pi(s',a')\big].
    \end{aligned}
    \label{eq:bellman}
\end{equation}
With sparse final-task rewards, one-step TD learning propagates the reward signal backward through successive bootstrapped targets. Function approximation can generalize across states, but the final-task reward alone provides little direct supervision for learning representations of intermediate progress.

\noindent
On the other hand, \Vadar{} supervises $Q_\phi(s,a,g)$ for every subgoal $g\in G$ at each sampled transition (\Cref{eq:vector_td}). These targets span multiple temporal resolutions: a state far from task completion may already be close to an intermediate subgoal. We speculate that predicting these nearer outcomes trains useful features before the final-task value is accurate.
Furthermore, because the subgoal predictions share network parameters (\Cref{fig:implicit-bellman}), features learned from intermediate targets are also available to the final-subgoal prediction $Q_\phi(s,a,g^{|G|})$. We hypothesize that this transfer helps the critic learn dependencies across stages of the task that would otherwise have to be learned from final-task rewards alone.

The \LastQ{} comparison provides evidence of this hypothesized mechanism: intermediate targets improve learning even though only the final-subgoal value enters policy extraction (\Cref{sec:exp-lastq}). These comparisons do not establish additional Bellman backups within a forward pass or a guaranteed reduction in the number of TD updates. \Cref{app:vadar_iterative_critics} develops a related interpretation for iterative critics, which explicitly refine a vector of value estimates over multiple integration steps.

\iclrendwrap

\vspace{-.2em}
\section{Limitations and Conclusion}
\label{sec:conclusion}
\vspace{-.2em}
This work introduces \Vadar{}, an approach to critic learning that utilizes subtask decomposition, which provides significant improvements across a wide range of policy extraction algorithms as well as challenging environment specifications such as long horizon or safety filtering. Still, a limitation of \Vadar{} is that it defining a fixed DAG structure of subtasks, whereas other methods enable conditioning on a continuum of goal information \citep{black2024susie, pi07_2026}. We look forward to generalizing the design principles of  \Vadar{} to this continuous or combinatorially large goal spaces, such as those described by images or text, in future work.
Moreover,  while recent literature has focused on improved methods for policy extraction, particularly with modern expressive policy parameterizations \citep{ren2024diffusion,patil2026ogpo,park2025flowqlearning,dong2025expo}, \Vadar{} demonstrates that there remains significant untapped potential for algorithmic improvements to critic learning, particularly in long-horizon tasks. Moreover, in light of the recent popularity of scaling test-time computation \citep{snell2024scaling}, our study suggests a complimentary scaling direction: scaling the \emph{density of training signals} used to supervise learning, which not only improves learning efficacy, but also reduces the need for complex test-time computation (e.g. VLM queries for goal-selection) at deployment. 


\Needspace{9\baselineskip}
\subsubsection*{Acknowledgments}

The authors would like to thank Nidhi Hegde for helping with the initial implementation. MS would like to thank Abhishek Gupta and Andrea Bajcsy for insightful discussions. SS would like to thank Manan Agarwal and Giri Anantharaman for their help with the baseline implementations  and Shreyas Kowshik, Michael Chen, Pranav Vignesh, Elizabeth Terveen and other members of CMU MaxLab  for the informative discussions and feedback on an earlier draft. This research used the Babel computational cluster resources provided by the School of Computer Science at Carnegie Mellon University.
MS and SBP would like to acknowledge support from the Toyota Research Institute (TRI) University 3.0 program.


\newpage

\bibliographystyle{iclr2027/iclr2027_conference}
\newpage
\bibliography{refs}

\appendix

\crefalias{section}{appendix}
\crefalias{subsection}{appendix}
\crefalias{subsubsection}{appendix}


\section{Practitioner's Guide to \Vadar{}}
\label{app:practitioners_guide}

Here we give a short hands-on reference for applying \Vadar{} to a new
task or base algorithm. 
\confpar{Subgoal decomposition.}
Most of the long horizon tasks can be decomposed naturally into subtasks which should be taked as the subgoals for \Vadar{}. For robotics tasks, we believe that the subgoals defined by us for our environments (\Cref{app:task_descriptions}) can serve as good examples for subgoal selections for a new task.  

\confpar{Weighting variant.}
From our experiments, we found \VadarNF{} to be the most performative variant
for robotic manipulation tasks, and recommend it as the default. \VadarCF{} can
train slightly faster early on when the first subgoal is hard to reach from
BC initialization, at the cost of greater susceptibility to over-exploiting
the current goal. The myopic variants are useful as ablations but not as
defaults.

\textbf{Adapting to a new base algorithm.}
\Vadar{} is a wrapper over any Value Based RL Algorithm. Given any value-based RL algorithm $\mathcal{A}$, the
required modifications are:
\begin{enumerate}
    \item Change critic output dim from $1$ to $|\mathcal{G}|$ (backbone
    unchanged); ensemble aggregation becomes entrywise.
    \item Store per-step completion vectors $\mathbf{d}_t$ and per-head
    rewards $\mathbf{r}_t$ in the replay buffer.
    \item Replace the critic loss with its summed-over-heads version; broadcast any $\mathcal{L}_{\mathrm{reg}}$
    across heads (see \Cref{app:sac} for more details).
    \item Apply the entrywise reward clip in the TD target.
    \item In the policy loss, replace $Q_\phi(s,a)$ with $Q_{\mathrm{agg}}(s,a)
    = \langle w(\mathbf{d}(s)), \langle Q_\phi(s,a, g), g \in G\rangle \rangle$.
\end{enumerate}
The policy is \emph{not} conditioned on subgoals. We have applied this recipe
unchanged to SAC, FQL, and OGPO (\Cref{app:base_algorithms}), and these serve as good references for applying \Vadar{} to any new Value Based RL Algorithm.


\section{Critic Parameterizations, Ensembling, and Target Networks}
\label{app:critic_ensembling}

This section expands the brief mention in \Cref{sec:preliminaries} of critic ensembling and target exponential moving averages (EMAs), which are standard practices in modern off-policy value-based RL and which we adopt without modification in \Vadar{}. The recipes below apply uniformly across the base instantiations used in our experiments; specific details that interact with the policy update are covered in \Cref{app:base_algorithms}.

\subsection{Critic Parameterizations}
\label{app:critic-parameterizations}

\confpar{Vector-valued  multi-goal critic.}  As an alternative to the goal-conditioned critic $Q_\phi(s,a,g)$, \Vadarnc{} can use a \modcolor{multi-goal} critic $\bQ_\phi: S \times A \to \modcolor{\R^G}$ that maps each state-action pair to a $|G|$-dimensional vector of Q-values. The $g$-th entry $\bQ_\phi[g](s,a)$ estimates the value of the state-action pair under the MDP induced by the per-subgoal reward $r_g$. Crucially, all $|G|$ entries share a single network backbone, enabling representation sharing across subgoals: features learned for predicting one subgoal's completion are available to all others, Both parameterizations provide dense supervision across subgoals; \Cref{sec:exp-critic-param} compares them empirically.
 \Cref{app:successor-features} shows the vector-valued critic can be interpreted simply as a structured instantiation of successor features \citep{barreto2017successor}, where each feature corresponds to subgoal performance.

\subsection{Critic Ensembles}

Off-policy methods are known to suffer from \emph{overestimation bias}: maximizing over a noisy approximation of $Q^{\pi}$ systematically inflates the resulting target~\citep{fujimoto2018addressing, chen2021randomized}. The standard remedy is to train an ensemble of $M$ critic networks $\vectorqfirst,\dots,\vectorqlast: S \times A \to \R$ and aggregate their predictions when forming the TD target. We modify the classical scalar-valued critic in each ensemble member to be a goal-conditioned critic $Q_{\phi_i}: S \times A \times G \to \R$. Concretely, given any aggregator $\mathrm{agg}_M : \R^M \to \R$ (e.g., mean, min), we extend it to act on the $M$ vector critics by applying it independently to each subgoal head: for every $g \in G$,
\begin{align}
    \big[\mathrm{agg}_M\!\big(Q_{\phi_1}, \dots, Q_{\phi_M}\big)\big](s, a, g)
    \;=\; \mathrm{agg}_M\!\Big(Q_{\phi_1}(s,a, g),\, \dots,\, Q_{\phi_M}(s,a, g)\Big).
    \label{eq:entrywise_agg}
\end{align}
We use \Cref{eq:entrywise_agg} as the meaning of ``entrywise aggregation'' throughout the paper, and refer back to it whenever an aggregator is applied to the critic. The ensemble size $M$ is taken to be 10 ($M = 10$) by default, unless otherwise specified.

\subsection{Target Aggregation}
\label{app:critic_aggregation}

The TD target $Q_{\bar\phi}(s', a')$ is computed by aggregating the predictions of $M$ \emph{target} critics $Q_{\bar\phi_1},\dots,Q_{\bar\phi_M}: S \times A \times G \to \R$, whose parameters are an EMA of the online critics (\Cref{app:critic_ema}). We use the ensemble \emph{mean} as the aggregator $\mathrm{agg}_M$ throughout, applied entrywise per \Cref{eq:entrywise_agg}:
\begin{align}
    Q_{\bar\phi}(s',a', g) \;=\; \frac{1}{M}\sum_{i=1}^{M} Q_{\bar\phi_{i}}(s',a', g) \; \forall g \in\; G.
    \label{eq:vadar_qtarg}
\end{align}
The aggregated scalar $\Qtarg(s',a') = \langle \bw, \langle Q_{\bar\phi}(s',a', g), g \in G \rangle \rangle$ used downstream by the policy update (\Cref{eq:q_agg}) is recovered by an inner product with the policy-extraction weights $\bw$ ($\bw(s) = \getWeight(\bd(s))$ based on the weighting strategies we described in \Cref{tab:value-weighting}). Note that other standard alternatives like the full ensemble minimum, or a randomized two-sample minimum in the style of REDQ~\citep{chen2021randomized} can also be used for the curtailing overestimation bias, but we found the entrywise mean to work consistently well across all \Vadar{} experiments and adopt it as the default.

\subsection{Target Network EMAs}
\label{app:critic_ema}

To stabilize the moving Bellman target, we maintain a separate set of target critic parameters $\bar\phi_1,\dots,\bar\phi_M$ that lag the online parameters via Polyak averaging:
\begin{align}
    \bar\phi_i \;\leftarrow\; (1-\tau)\,\bar\phi_i + \tau\,\phi_i, \qquad i = 1,\dots,M,
    \label{eq:critic_ema}
\end{align}
with decay rate $\tau \in (0,1]$ applied at every gradient step. We use $\tau = 0.005$ throughout. When the base algorithm requires a target policy, the same Polyak update is applied to the policy parameters:
\begin{align}
    \bar\theta \;\leftarrow\; (1-\tau)\,\bar\theta + \tau\,\theta.
\end{align}
Whether a target policy is maintained, and how it is used, is algorithm-specific; see \Cref{app:base_algorithms}.

\subsection{Critic Loss}

Combining \Cref{eq:vadar_qtarg,eq:critic_ema} with \Vadar{}'s formulation (\Cref{eq:vector_td}) yields the following per-ensemble-member critic loss, summed over subgoal heads:
\begin{align}
    \Lcritic(Q_{\phi_i} \mid s,a, \bytarg) \;=\; \sum_{g \in G} \big(Q_{\phi_i}(s,a, g) - \bytarg_g\big)^2, \qquad i = 1,\dots,M,
\end{align}
where $\bytarg$ is the clipped vector TD target from \Cref{eq:vector_td}.


\section{Base-Algorithm Details: \SAC{}, \FQL{}, and \OGPO{}}
\label{app:base_algorithms}
We position \Vadar{} as a wrapper around any value-based RL algorithm that replaces each scalar critic with a goal-conditioned critic $Q_{\phi_i}: S \times A \times G \to \R$, evaluated once per subgoal, with a few variations in the $\Qtarg$ aggregation that can be plugged into the base algorithm's policy extraction objective.
Here we present the policy losses $\Lpol$ and any algorithm-specific implementation details for the three base algorithms used in our experiments: \SAC{}, \FQL{}, and \OGPO{}. \Cref{sec:exp-long-horizon} pairs each base algorithm with a benchmark suite (\SAC{} with \OGBench{} and \Libero{}, \FQL{} with \Robomimic{}, and \OGPO{} with both \Robomimic{} and \Robocasa{}); coupling each policy extraction method to its correspondingly challenging benchmark(s). The descriptions below are restricted to the components that interact with \Vadar{}.

\subsection{Soft Actor-Critic (\SAC{})}
\label{app:sac}

\SAC{} \citep{haarnoja2018soft} is a maximum-entropy off-policy actor-critic algorithm that augments the standard return with a per-step entropy bonus $\alpha \mathcal{H}(\pi(\cdot \mid s))$. The optimal policy is the one that maximizes
\begin{align}
    J(\pi) \;=\; \sum_{t \ge 0} \mathbb{E}_{(s_t, a_t)} \left[ R(s_t, a_t) + \alpha\, \mathcal{H}\!\big(\pi(\cdot \mid s_t)\big) \right],
\end{align}
where $\alpha > 0$ is a temperature controlling the entropy-reward trade-off. We use a Gaussian policy with a $\tanh$-squashing transform that bounds actions in $[-1,1]^{\dim A}$~\citep[Eq.~(11)]{haarnoja2018soft}: $a_t = \tanh(\mu_\theta(s_t) + \sigma_\theta(s_t) \odot \epsilon_t)$, with $\epsilon_t \sim \Normal(0, I)$.

\confpar{Critic loss.} The soft Q-function is trained against an entropy-augmented Bellman target. Plugging into the value-based RL template of \Cref{sec:preliminaries} ($\Lcritic(Q_\phi \mid s,a,\ytarg) = (Q_\phi(s,a) - \ytarg)^2 + \regu$), the entropy bonus enters through the regularizer in the target:
\begin{align}
    \ytarg \;=\; r + \gamma\,\big[\Qtarg(s', a') - \alpha \log \pi_\theta(a' \mid s')\big],
    \qquad a' \sim \pi_\theta(\cdot \mid s'),
    \label{eq:sac_target}
\end{align}
so that $\regu(s', a', \pi_\theta) = -\gamma\,\alpha \log \pi_\theta(a' \mid s')$ in the notation of \Cref{sec:preliminaries}.

\newcommand{\Qphi}{Q_{\phi}}

\confpar{Policy loss.} The actor minimizes the KL-divergence between the policy and the (exponentiated) soft Q-function. Using the reparameterization trick~\citep[Eq.~(12)]{haarnoja2018soft2}, this reduces to:
\begin{align}
    \Lpol(\pi_\theta \mid s, \Qphi) \;=\; \mathbb{E}_{a \sim \pi_\theta(\cdot \mid s)} \Big[\alpha \log \pi_\theta(a \mid s) - \texttt{sg}\left[\Qphi(s, a)\right]\Big],
    \label{eq:sac_policy}
\end{align}
where $\texttt{sg}$ is the stopgrad operator. The expectation is estimated by a single reparameterized sample per state, and the gradient flows through the Squashed Gaussian into both the policy $\mu_\theta$ and $\sigma_\theta$ heads.

\confpar{Temperature auto-tuning.} We tune $\alpha$ to match a target entropy $\bar{\mathcal{H}} = -\dim(A)$ via the auxiliary loss~\citep[Eq.~(18)]{haarnoja2018soft2}:
\begin{align}
    L(\alpha) \;=\; \mathbb{E}_{a \sim \pi_\theta(\cdot \mid s)}\!\left[ -\alpha \log \pi_\theta(a \mid s) - \alpha \bar{\mathcal{H}}\right].
\end{align}

\confpar{Interaction with \Vadar{}.} The scalar critic $Q_\phi: S \times A \to \R$ is replaced by the goal-conditioned critic $Q_\phi: S \times A \times G \to \R$, which is evaluated once for every subgoal $g \in G$. The Bellman target in \Cref{eq:sac_target} is computed for all goals, with the same next action $a'$ and the same entropy regularizer shared across subgoals:
\begin{align}
    \bytarg[g] \;=\; \br[g] + \gamma\,\big[Q_{\bar\phi}(s', a', g) - \alpha \log \pi_\theta(a' \mid s')\big],
    \qquad a' \sim \pi_\theta(\cdot \mid s'),
    \qquad \forall g \in G,
    \label{eq:sac_vadar_target}
\end{align}
where $Q_{\bar\phi}$ is the target critic as defined in \Cref{app:critic_ema}, and the per-ensemble-member critic loss is summed over subgoals (\Cref{app:critic_ensembling}):
\begin{align}
    \Lcritic\!\big(Q_{\phi_i} \mid s,a,\bytarg\big) \;=\; \sum_{g \in G} \big(Q_{\phi_i}(s,a,g) - \bytarg[g]\big)^2.
\end{align}
The actor sees only the aggregated scalar $\Qagg(s,a) = \sum_{g \in G} \bw[g]\, Q_{\phi}(s,a,g)$ of \Cref{eq:q_agg}, so the policy loss in \Cref{eq:sac_policy} becomes
\begin{align}
    \Lpol(\pi_\theta \mid s, \Qagg) \;=\; \mathbb{E}_{a \sim \pi_\theta(\cdot \mid s)} \Big[\alpha \log \pi_\theta(a \mid s) - \texttt{sg}\Big[\sum_{g \in G} \bw[g]\, Q_{\phi}(s,a,g)\Big]\Big].
    \label{eq:sac_vadar_policy}
\end{align}

The temperature loss remains unchanged.

\subsection{Flow Q-Learning (\FQL{})}
\label{app:fql}

\FQL{}~\citep{park2025flowqlearning} is an off-policy method designed for expressive flow-matching policies. Its key idea is to avoid backpropagating through the multi-step ODE integration of a flow policy by training a separate, expressive \emph{one-step} policy that is simultaneously (i) regularized toward a frozen behavioral-cloning (BC) flow policy via a distillation loss and (ii) optimized to maximize the critic.

\textbf{Architecture.} \FQL{} maintains three networks: a critic $Q_\phi$, a BC flow policy $\mu_\theta(s, z)$ that defines a flow-matching velocity field $v_\theta(t, s, x^t)$, and a one-step policy $\mu_\omega(s, z): S \times \R^{\dim A} \to A$ that maps Gaussian noise $z \sim \Normal(0, I_{\dim A})$ directly to an action.

\confpar{BC flow loss.} The BC flow policy is trained with the standard flow-matching objective on the offline dataset $\cD$~\citep[Eq.~(5)]{park2025flowqlearning}:
\begin{align}
    \mathcal{L}_{\mathrm{Flow}}(\theta) \;=\; \mathbb{E}_{\substack{(s, a^1) \sim \cD,\, x^0 \sim \Normal(0, I) \\ t \sim \mathrm{Unif}([0,1])}}\!\Big[\big\|v_\theta(t, s, x^t) - (a^1 - x^0)\big\|_2^2\Big],
    \quad x^t = (1-t)x^0 + t\,a^1.
\end{align}
The corresponding stochastic policy $\pi_\theta(a \mid s)$ is the push-forward of $\Normal(0, I)$ under the ODE $\mu_\theta(s, \cdot)$ defined by $v_\theta$. Crucially, this loss is the only training signal seen by the flow policy: $\mu_\theta$ is never updated against $Q_\phi$.

\confpar{Critic loss.} The critic is trained with the standard TD residual but using the \emph{one-step} policy $\mu_\omega$ in the bootstrap action~\citep[Sec.~3]{park2025flowqlearning}:
\begin{align}
    \Lcritic(Q_\phi) \;=\; \mathbb{E}_{(s,a,r,s') \sim \cD,\, z \sim \Normal(0, I)} \Big[\big(Q_\phi(s,a) - r - \gamma\,\Qtarg(s', \mu_\omega(s', z))\big)^2\Big],
\end{align}
where $\Qtarg = Q_{\bar\phi}$ is the EMA target critic (\Cref{app:critic_ema}).

\confpar{One-step policy loss.} The one-step policy is trained against an actor loss with two terms: a Q-maximization term and a distillation term that pulls $\mu_\omega(s, z)$ toward the BC flow policy's output $\mu_\theta(s, z)$ at the same noise $z$~\citep[Eqs.~(7),(9)]{park2025flowqlearning}:
\begin{align}
    \Lpol(\pi_\omega \mid s, Q_\phi) \;=\; \underbrace{\mathbb{E}_{\substack{s \sim \cD,\, z \sim \Normal(0, I) \\ a^\pi = \mu_\omega(s, z)}}\!\big[-Q_\phi(s, a^\pi)\big]}_{\text{Q loss}} \;+\; \alpha_{\mathrm{FQL}}\,\underbrace{\mathbb{E}_{\substack{s \sim \cD \\ z \sim \Normal(0, I)}}\!\big[\|\mu_\omega(s, z) - \mu_\theta(s, z)\|_2^2\big]}_{\text{distillation loss}}.
    \label{eq:fql_policy}
\end{align}
Gradients of the distillation term are stopped at $\mu_\theta$, so the BC flow policy is not affected.

\confpar{Interaction with \Vadar{}.} When applying \Vadar{} on top of \FQL{}, the scalar critic is replaced by the goal-conditioned critic $Q_\phi: S \times A \times G \to \R$, evaluated once for every subgoal $g \in G$. The TD bootstrap is computed for all goals, using a single bootstrap action $\mu_\omega(s', z)$ shared across subgoals:
\begin{align}
    \bytarg[g] \;=\; \br[g] + \gamma\,Q_{\bar\phi}\!\big(s', \mu_\omega(s', z), g\big),
    \qquad z \sim \Normal(0, I),
    \qquad \forall g \in G,
\end{align}
and the per-ensemble-member critic loss is summed over subgoals:
\begin{align}
    \Lcritic\!\big(Q_{\phi_i} \mid s,a,\bytarg\big) \;=\; \mathbb{E}_{(s,a,\br,s') \sim \cD,\, z}\!\Big[\sum_{g \in G}\big(Q_{\phi_i}(s,a,g) - \bytarg[g]\big)^2\Big].
\end{align}

Similar to before, we define the online aggregate $\Qagg(s,a) := \sum_{g \in G} \bw[g]\, Q_\phi(s,a,g)$ (online critic, matching the \FQL{} convention of \Cref{eq:fql_policy}). The Q-loss term in \Cref{eq:fql_policy} then becomes
\begin{align}
    \Lpol(\pi_\omega \mid s, \Qagg) \;=\; \underbrace{\mathbb{E}_{\substack{s \sim \cD,\, z \sim \Normal(0,I) \\ a^\pi = \mu_\omega(s, z)}}\!\Big[-\sum_{g \in G} \bw[g]\, Q_\phi(s, a^\pi, g)\Big]}_{\text{Q loss}} \;+\; \alpha_{\mathrm{FQL}}\,\mathcal{L}_{\mathrm{distill}}(\omega; \theta).
\end{align}
The distillation term is unchanged. The deployed policy is the one-step $\mu_\omega$, which takes no goal as input and hence requires no goal information at inference (in line with \Vadar{}'s subgoal-agnostic property).

\subsection{Off-Policy Generative Policy Optimization (\OGPO{})}
\label{app:ogpo_policy_loss}

\newcommand{\Nbatch}{N_{\texttt{batch}}}
\newcommand{\Ngroup}{N_{\texttt{group}}}

\OGPO{} \citep{patil2026ogpo} is an off-policy actor-critic for generative control policies (GCPs), like flow matching and diffusion, that side-steps the instabilities of backpropagation through the flow ODE by applying a zero-order PPO-style update to the denoising chain itself. \Vadar{} pairs with \OGPO{} on our most challenging long-horizon benchmarks (\Robomimic{}, and \Robocasa{}). We summarize the policy loss here and note that \OGPO{} follows the same aforementioned critic ensembling and EMA strategies (\Cref{app:critic_ensembling}).

\confpar{Setup.} A flow-based GCP samples an action $\alastt$ via a $\igpfirst$-step denoising chain $a_{t,K} \to a_{t,K-1} \to \dots \to a_{t,0}$, where each step samples $a_{t,k-1} \sim \pi_\theta(\cdot \mid s_t, a_{t,k})$. The likelihood of a full chain factorizes as
\begin{align}
    \pi_\theta(a_{t, K:0} \mid s_t) \;=\; \prod_{k=1}^{K} \pi_\theta(a_{t,k-1} \mid s_t, a_{t,k}).
\end{align}
\OGPO{} treats the denoising trajectory $a_{t,K:0}$ as a sequence in an inner ``denoising MDP'' embedded inside the environment MDP, and applies PPO at the level of this inner MDP only.

\confpar{Annealed importance sampling ratio.} The reference policy is the EMA target $\piema$ (\Cref{app:critic_ema}). The policy ratio is the product of per-step ratios over the entire chain:
\begin{align}
    \ratio \;:=\; \prod_{k=1}^{K} \frac{\pi_\theta(a_{t,k-1} \mid s_t, a_{t,k})}{\piema(a_{t,k-1} \mid s_t, a_{t,k})}.
    \label{eq:ogpo_ratio}
\end{align}

\confpar{OGPO Policy Updates}
\OGPO{} performs policy extraction via a PPO-style zeroth order policy gradient update that considers the entire flow-denoising chain as an MDP. The terminal action $\alastt$ is the only action evaluated against the target critic. \OGPO{} draws a \emph{group} of $\Ngroup$ denoising trajectories $a^{(i,j)}_{K:0}$ from $\piema(\cdot \mid s^{(i)})$ per state and uses the group mean as a value baseline:
\begin{align}
    \hat V^{(i)} \;\gets\; \frac{1}{\Ngroup} \sum_{j=1}^{\Ngroup} \Qtarg\!\big(s^{(i)}, a^{(i,j)}_{0}\big),
    \qquad
    \hat A \;=\; \Qtarg(s_t, a_{t,0}) - \hat V.
\end{align}
This GRPO-style baseline obviates the need for a separately learned value network.

\confpar{Policy loss.} The PPO-clipped surrogate loss is applied over the entire flow denoising chain as:
\begin{align}
    \Lpol(\pi_\theta \mid s, \Qtarg) \;=\; -\,\mathbb{E}_{\substack{s_t \sim \cB \\ a_{t,K:0} \sim \piema(\cdot \mid s_t)}}\!\Big[\min\!\big(\ratio\, \hat A,\; \mathrm{clip}(\ratio, 1-\epsilon, 1+\epsilon)\, \hat A\big)\Big],
    \label{eq:ogpo_policy}
\end{align}
where $\epsilon$ is the PPO clip parameter. Because both $\ratio$ and the denoising trajectories can be computed in parallel without the need for environmental interactions, the loss is fully parallelizable across the $\Nbatch \times \Ngroup$ batch.

\confpar{Noise injection for flow policies.} Evaluating $\ratio$ requires per-step likelihoods $\pi_\theta(a_{t,k-1} \mid s_t, a_{t,k})$ to be non-singular. Following ReinFlow \citep{zhang2025reinflow}, we replace the deterministic Euler step with a Gaussian-noised update $\ainct \sim \Normal(v_\theta(\aitert, k/K, s_t),\, \sigma_k^2 I)$, and retain the marginal path-preserving correction of~\citet{albergo2023stochastic} so that the per-step SDE marginals match standard ODE flow matching rollouts.

\confpar{Interaction with \Vadar{}.} The critic is the goal-conditioned $Q_\phi: S \times A \times G \to \R$, trained on the all-goal targets $\bytarg[g] = \br[g] + \gamma\, Q_{\bar\phi}(s', a', g)$ for every $g \in G$, exactly as above. \Vadar{} composes with \OGPO{} by replacing the scalar $\Qtarg$ in the advantage computation with the aggregate of the target critic evaluated on every subgoal, following \Cref{eq:q_agg}:
\begin{equation}
\begin{aligned}
    \Qagg(s,a) &= \sum_{g \in G} \bw[g]\, Q_{\bar\phi}(s, a, g), \\
    \hat V^{(i)} &\gets \frac{1}{\Ngroup} \sum_{j=1}^{\Ngroup} \Qagg\!\big(s^{(i)}, a^{(i,j)}_{0}\big),
    \qquad
    \hat A = \Qagg(s_t, a_{t,0}) - \hat V.
\end{aligned}
\end{equation}
The same weights $\bw$ are used for all $\Ngroup$ samples of a given state, since they depend only on that state's completion vector. The PPO surrogate, group baseline, and noise-injection scheme are otherwise unchanged. For our experiments, unless otherwise stated, we use the \OGPOplus{} variant (due to its superior performance).

\section{Limitations of Explicit and Implicit Hierarchical Goal Conditioning}
\label{app:hgc-limitations}

We examine the limitations summarized in \Cref{sec:limits_gc} using implicit hierarchical goal conditioning (\ihgc), which preserves goal-conditioned critics but removes goal conditioning from the policy.

\subsection{Implicit Hierarchical Goal Conditioning}

\begin{definition}[Implicit Hierarchical Goal Conditioning (\ihgc)] Implicit goal conditioning follows the Hierarchical Goal Conditioning in \Cref{alg:gc-concise}, with the sole modification that the policy $\pi(a \mid s)$ is \textbf{not} conditioned on subgoal information.
\end{definition}

\par
\Needspace{20\baselineskip}
\begin{wrapfigure}{r}{0.42\textwidth}
      
      \centering
      \includegraphics[width=\linewidth]{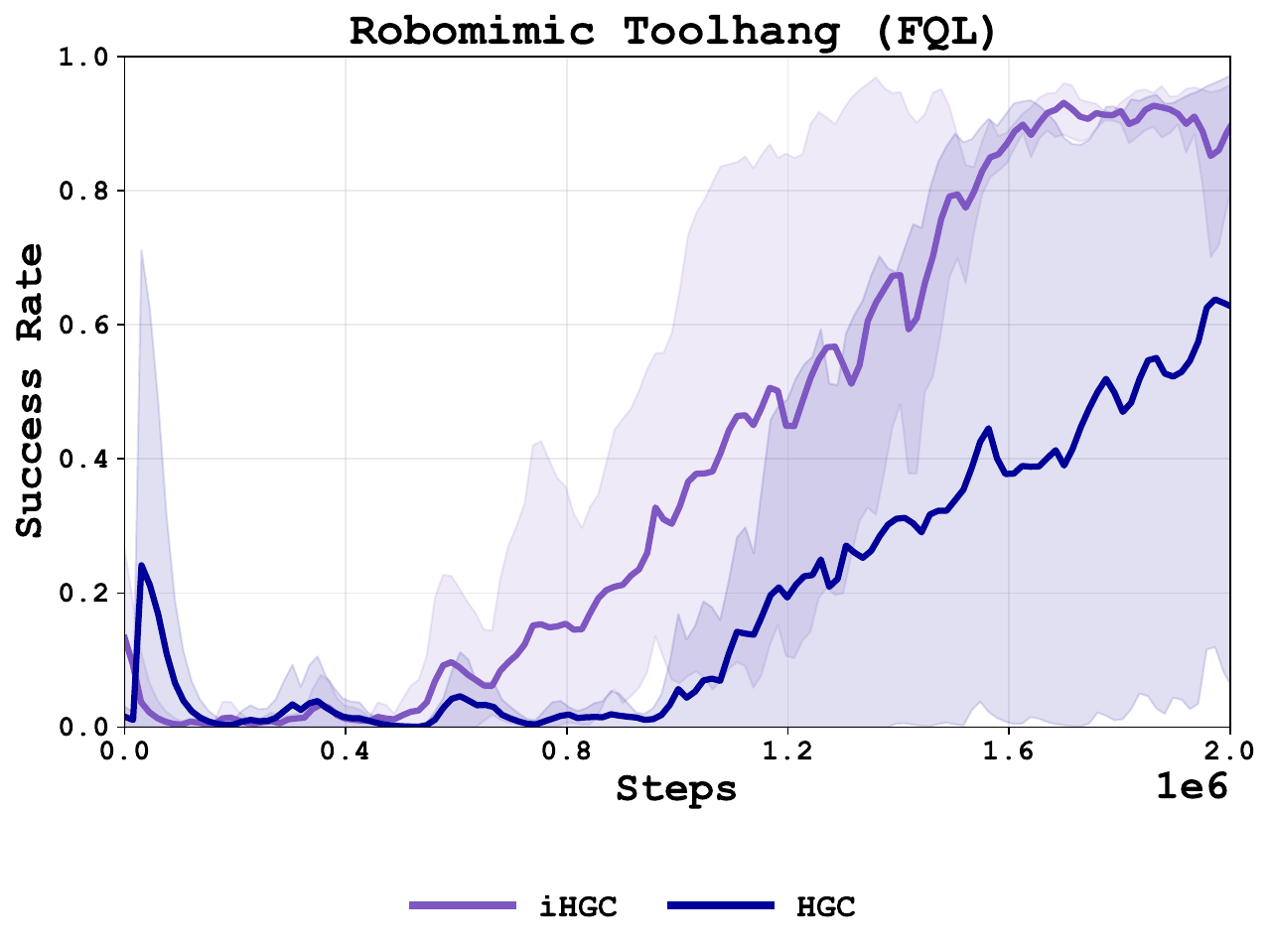}
      \caption{\ToolHang{} with \FQL{}: \ihgc{} learns more smoothly and achieves higher final success than \hgc{}.}
      \label{fig:ihgc-hgc-toolhang-fql}
\end{wrapfigure}

We run \hgc{} using \emph{oracle} dones, providing perfect subgoal selection for this comparison.

\ihgc{} eliminates the first limitation described above, as now the policy need not incorporate goal information at deployment time. As an immediate example,
  \Cref{fig:ihgc-hgc-toolhang-fql} compares \ihgc{} and \hgc{} on \Robomimic{} \ToolHang{} with the \FQL{} (\Cref{app:base_algorithms} for more details) policy extraction backbone. Removing goal information from
  the policy improves final performance and yields a smoother learning curve, suggesting that goal-conditioning the policy can introduce unnecessary instability even
  when goal-conditioned critics remain useful. We hypothesize that removing goal-conditioning inputs avoids aliasing at goal boundaries.

\iclrendwrap

\subsection{Both Explicit and Implicit Hierarchical  Goal Conditioning Suffer from Myopic Behavior}

\begin{figure}[!htbp]
    \centering
    \includegraphics[width=\linewidth]{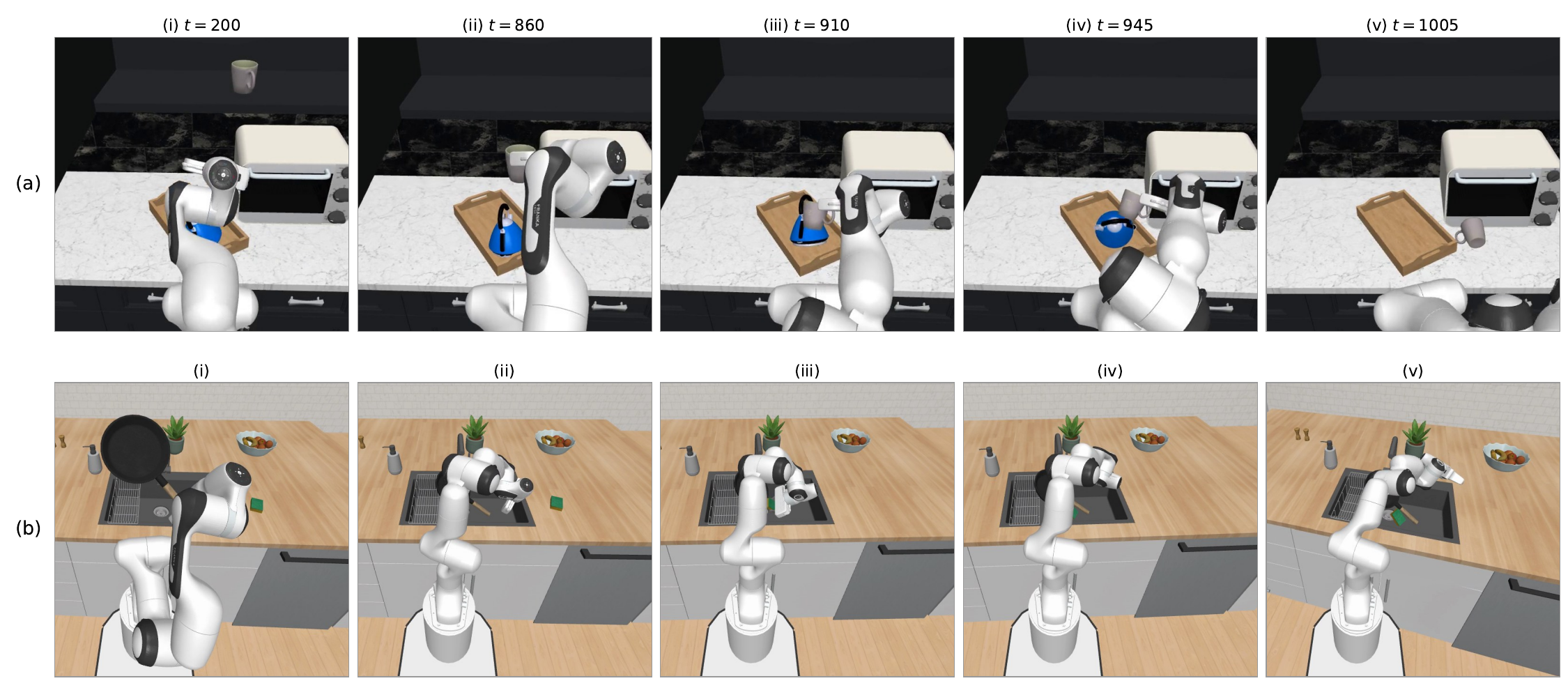}
    \caption{Myopic failure modes of goal-conditioned policies. Rollouts of policies trained to optimize only the current subgoal on \Robocasa{} \arrangetea{} (a, top) and \presoakpan{} (b, bottom); frames are chronological within each row. In (a), the kettle is placed on the tray without leaving room for the mug, so the subsequent mug placement collides with the kettle and knocks it off the counter. In (b), the pan is dropped into the sink in a position that does not permit a stable placement of the sponge on top of it}
    \label{fig:myopia-frame-reels}
\end{figure}

However, we still find that implicit and explicit hierarchical goal conditioning run into a common failure mode: \emph{overly myopic policy exploitation}. Both methods update the policy against the critic of the currently active subgoal alone. As a result, the policy is rewarded for completing the current subgoal without accounting for how its manner of completion affects subsequent subtasks. We illustrate this failure mode on \Robocasa{} \arrangetea{} and \presoakpan{} in \Cref{fig:myopia-frame-reels}.

In \Cref{fig:myopia-frame-reels}(a), the policy places the kettle on the tray in a position that obstructs the subsequent mug placement. The ``placing the kettle" subgoal is complete, but attempting to place the mug causes a collision that knocks the kettle off the counter. In \Cref{fig:myopia-frame-reels}(b), the policy places the pan in the sink in a position that prevents stable placement of the sponge.
Both these failure cases are manifestations of over-optimizing for the current subtask \hyperlink{lim-hgc-myopia}{(Limitation 2)}. The policy does not distinguish between \emph{merely completing the current subgoal} and \emph{completing it in a way that facilitates future subgoals}. Our method, \Vadar{}, addresses this limitation by optimizing the policy for future-subgoal values (\Cref{sec:reward_disagg}), providing a learning signal for actions that enable progress toward full-task completion rather than just myopically completing the current subtask.


\section{Dense Preference Reward: Algorithms and Implementation Details}
\label{app:dense_preference_reward}

\subsection{Background and Limitations}
An alternative to \hgc{} for solving the credit assignment problem
is to craft a denser reward signal $r$. When oracle dense rewards are not available, the most popular data-dependent method  has been preference-based dense rewards defined at the level of full trajectories $\tau = (s_1,a_1,\dots,s_H,a_H)$. Given two trajectories $\tau_1,\tau_2$, one maximizes the Bradley-Terry-Luce (BTL) \citep{19ff28b9-64f9-3656-ba40-08326a05748e} likelihood, assuming the probability that $\tau_2$ is preferred to $\tau_1$  is given by
\begin{align}
P(\tau_1 \prec \tau_2 \mid \psi, \beta) = \frac{\exp(\beta R_\psi(\tau_2))}{\exp(\beta R_\psi(\tau_1))+\exp(\beta R_\psi(\tau_2))},
\end{align}

where $R_\psi(\tau) = \sum_{s \in \tau}r_{\psi}(s)$ is cumulative trajectory return, $r_{\psi}(s)$ is a learned, per-state dense reward, and $\beta$ is a temperature parameter.  Preference based rewards have shown strong promise \citep{wirth2017survey,rafailov2023direct,tajwar2024preference}, but are subject to other limitations: (1) assigning consistent numerical progress values on very long, multi-step tasks is challenging; (2) statistical errors in reward learning can lead to reward hacking; and (3) effective reward learning can require large datasets. Dense rewards are imperfect proxies for desired behavior, and agents may exploit errors in the learned reward rather than complete the task \citep{amodei2016concreteproblemsaisafety,krakovna2020specification}. For an empirical comparison with value disaggregation, see \Cref{sec:pref_vs_td}. We describe our implementation below.

\subsection{Implementation}
We implemented dense prefence rewards for the same policy backbones as \Vadar{}: \OGPO{}, \SAC{}, \FQL{}. Our \pbrl{} implementation was heavily inspired by \citep{shin2023benchmarksalgorithmsofflinepreferencebased,wang2024rlvlmfreinforcementlearningvision,venkataraman2025realworld}  . We used trajectory-wise preferences to train a state-based reward model independently for each task. Trajectory preferences were autogenerated and fit using the Bradley--Terry \citep{19ff28b9-64f9-3656-ba40-08326a05748e} model. When training the reward model, we uniformly sample from the generated preference dataset.  \textbf{We found that the most reliable preference signal was binary trajectory success.} To ensure we used the most accurate preference signal, we experimented with alternatives in \Cref{tab:pbrl-signals}, but found them to be inaccurate.

\begin{table}[!ht]
\centering
\small
\setlength{\tabcolsep}{4pt}
\renewcommand{\arraystretch}{1.12}
\caption{Alternative auto generated preference signals tested for \pbrl{}.}
\label{tab:pbrl-signals}
\begin{tabular}{@{}p{0.28\linewidth}p{0.66\linewidth}@{}}
\toprule
\textbf{Signal} & \textbf{Observed issue} \\
\midrule
$\displaystyle \sum_{s \in \tau} V^*(s)$
&
Uses the value function $V^*$ from an optimal policy $\pi^*$, but produced noisy
preferences because $V^*$ gave poor estimates outside the distribution induced
by $\pi^*$. \\[0.35em]

$\displaystyle \sum_{t, s_t \in \tau} \gamma^t \rsparse(s_t)$
&
Monte Carlo returns depend on the policy used to generate the trajectory, making
the induced reward signal noisy and not policy-invariant at the state level. \\[0.35em]

$\displaystyle \sum_{\rdense \in \tau} \rdense$
&
Uses $\rdense$ as defined from the QIPM main algorithm, but it preferred longer trajectories, which introduced a harder and less reliable
preference signal for the reward model to learn. \\
\bottomrule
\end{tabular}
\end{table}

\confpar{Offline Training}
\pbrl{} was provided privileged information, a diverse set of trajectories to kick off it's reward training from, 100 trajectories collected from $(\pi_{\textsc{early}}, \pi_{\textsc{mid}}, \pi_{\textsc{late}})$, where these were checkpoints collected during the training of an optimal policy, $\pi^*$. This dataset combined with the original optimal success dataset provided on average 500 trajectories to pre-train the reward model from, with different levels of optimality and a diverse range of failure modes. \textbf{We saw that this was required for \pbrl{} to show signs of success}, as demonstrated in \Cref{fig:transport-reward-rollouts}. This is a significant setback compared to \Vadar{}, which is able to operate on far more limited datasets. 
\begin{figure}[!ht]
    \centering
    \includegraphics[width=\textwidth]{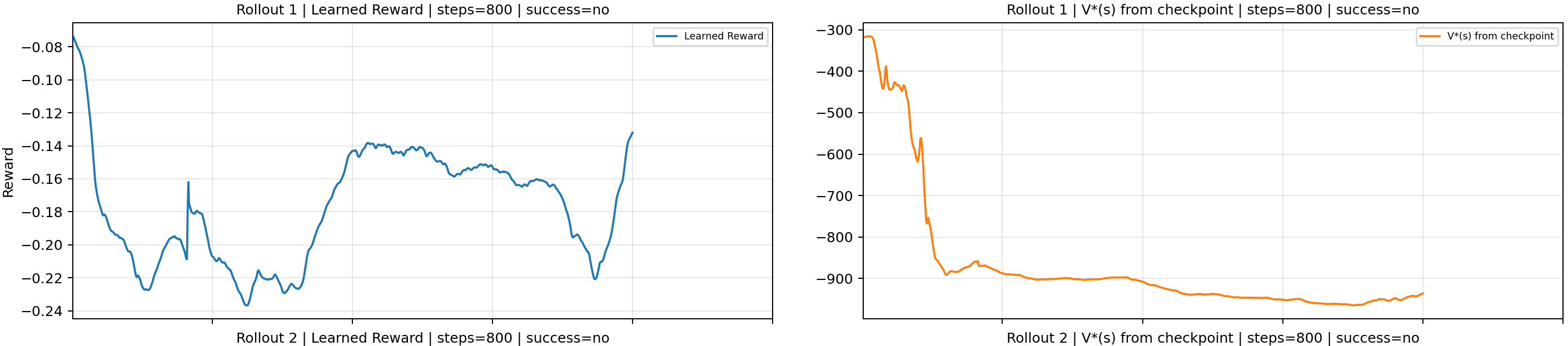}\\
    \includegraphics[width=\textwidth]{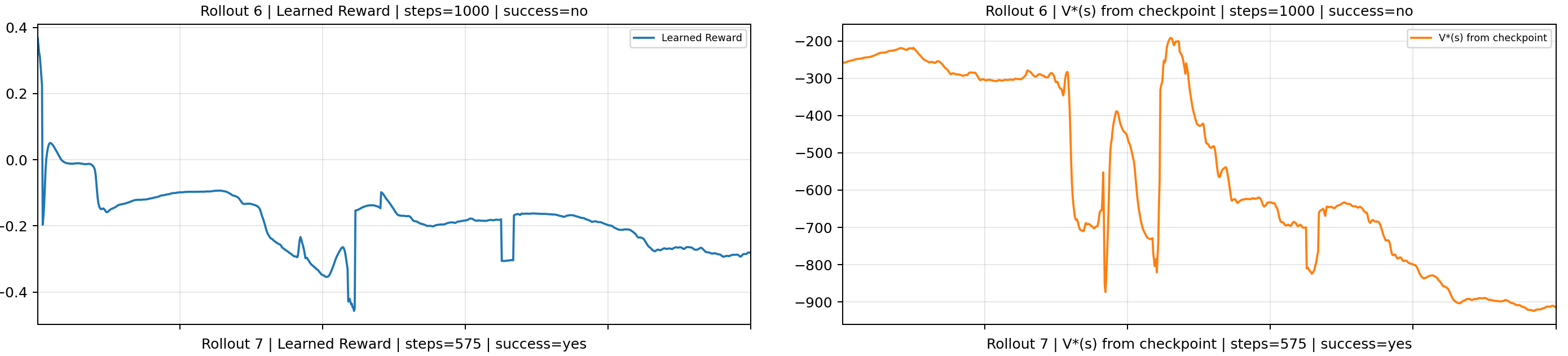}
    \caption{Comparison of success-only (top) and mixed (bottom) dataset based reward rollouts. We see that the success-only dataset is unable to accurately estimate failure cases, not having a clear signal for reward. Both models were using preference based on $V^*$}
    \label{fig:transport-reward-rollouts}
\end{figure}

\confpar{Online Training}
After offline reward-model pretraining, we continue updating the dense reward model during online RL following \citep{wang2024rlvlmfreinforcementlearningvision,lee2021bpref}. 
The reward model is trained using the following iterative loop: the current reward model is used to train the policy, the resulting policy collects new trajectories, these trajectories are converted into additional synthetic preference comparisons, and the reward model is fine-tuned on the offline and online preference dataset. 

During online RL, the environment reward is replaced by the learned state-only reward model $\hat r_\theta(s)$. 
The underlying RL algorithm is otherwise unchanged: \OGPO{}, \SAC{}, and \FQL{} use their standard actor and critic losses, but with rewards given by $\hat r_\theta$ instead of the environment reward.

Every $K_{\mathrm{int}}$ environment steps, the reward model is fine-tuned on the union of offline and online preference pairs using the same Bradley--Terry objective as in offline pretraining. 
The reward model gradually adapts to the on-policy state distribution induced by the improving policy, while still being regularized by the original offline preference data.

After each reward-model update, we relabel the live RL replay buffer with the updated reward model:
$
    r_t \leftarrow \hat r_\theta(s_t).
$
This step makes the online reward update immediately affect subsequent policy learning. 

\confpar{Hyper Parameters}
\Cref{tab:hp-pbrl-reward-model} mentions the primary hyper parameters used to train the model.
\begin{table}[!htb]
\centering
\small
\setlength{\tabcolsep}{6pt}
\caption{Reward-model hyperparameters used for PBRL experiments.}
\label{tab:hp-pbrl-reward-model}
\begin{tabular}{l l}
\toprule
\textbf{Hyperparameter} & \textbf{Value} \\
\midrule
\multicolumn{2}{l}{\textit{Reward model}} \\
ensemble size & $3$ \\
model architecture & MLP of ($512, 512, 512 ,512$) \\
learning rate & $3 \times 10^{-4}$ \\
training epochs & $50$ \\
segment size & \texttt{episode} \\
\midrule
\multicolumn{2}{l}{\textit{Online PBRL schedule}} \\
RM update interval & $5{,}000$ env steps \\
maximum online feedback & $10{,}000$ preference pairs \\
feedback buffer & FIFO \\
\bottomrule
\end{tabular}
\end{table}

\section{Additional Baseline Comparisons}
\label{app:additional-baselines}

\subsection{Hindsight experience replay (HER) struggles with contact-rich control}
\label{sec:exp-her}

An alternative approach to mitigating the credit assignment problem in goal-conditioned RL is hindsight experience replay (HER) \citep{andrychowicz2018hindsightexperiencereplay}, which relabels states reached during a rollout as goals, so that the policy receives learning signal even when the task reward is never observed. HER has largely been evaluated in reaching, pick-and-place and maze settings, where any state is a reasonable goal. To test whether it applies to contact-rich manipulation settings, we evaluate HER with \SAC{} on \Libero{}-Task 6, with goals given by object states and relabelled from future states of the same episode, and with the same BC pretraining, online step budget and networks as our other \SAC{} baselines. We also steelman HER with two variants that provide it with the same subgoal information that \hgc{} receives: an oracle that supplies a representative goal state for each subgoal in turn, and a learned goal embedding trained to recognize subgoal completions. In both the variants, we find HER to be incapable of performing multi stage complex manipulation tasks, with the eventual success rate remaining quite close to the initial BC success rate.

\subsection{TD critic learning outperforms preference-based per-subgoal reward}
\label{sec:pref_vs_td}

\begin{figure}[h]
    \centering
    \includegraphics[width=\linewidth]{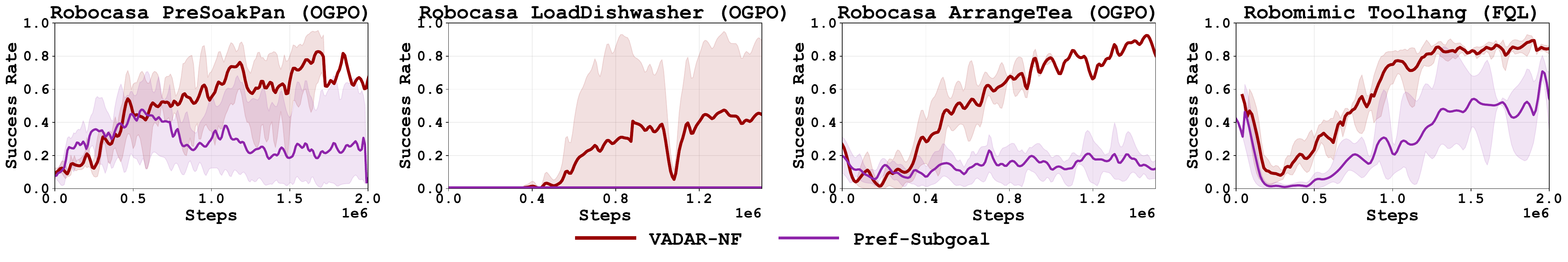}
    \caption{Success rate of \VadarNF{} against \PrefSubgoal{}, which learns a dense reward per subgoal from preferences, on \Robocasa{} with the \OGPO{} backbone and \Robomimic{} \ToolHang{} with the \FQL{} backbone.}
    \label{fig:pref-subgoal}
    \vspace{-1em}
\end{figure}

The \densecolor{\textmd{Dense}} baseline in \Cref{sec:experiments} learns rewards from the full trajectory preference, so it does not receive the same subgoal information as \Vadar{}. To steelman the comparison, we learn separate preference reward models for each of the subtasks, aggregate them using our Next-Future weighting same as \VadarNF{}, and then train a scalar critic on aggregated reward. We call this method \PrefSubgoal{}. \Cref{fig:pref-subgoal} shows that the \PrefSubgoal{} underperforms \VadarNF{} on all four tasks. On \presoakpan{}, \PrefSubgoal{} appears to be more sample efficient in the early stages but later fails to reach the same final performance as \VadarNF{}. On \arrangetea{} and \loaddishwasher{},  it does not improve over its BC initialization. With \FQL{} on \ToolHang{}, \PrefSubgoal{} trains more slowly and plateaus below \VadarNF{}.

These results suggest that directly fitting the critic on multiple temporal resolutions is more reliable than learning dense per-subgoal rewards from preferences. We conjecture that small reward-model errors accumulate over the long horizon, while \Vadar{} directly supervises per-subgoal TD targets at multiple temporal resolutions while sharing representations across subgoals. This structure may benefit from the implicit Bellman-backup mechanism discussed in \cref{sec:exp-bellman}.

\section{DAG Extension for Non-Sequential Subtasks}
\label{app:dag_extension}

Multiple subgoals may be active simultaneously, as in bimanual robotic manipulation. The value-weighting function in \Cref{tab:value-weighting} extends naturally to this setting: a DAG can express both parallel subgoals and prerequisites that must be completed before a subsequent goal becomes available.

In the sequential setting, subtasks must be completed in a fixed linear order. However, many real-world tasks have more complex structures where multiple subtasks can be pursued simultaneously for optimal behavior. Moreover, some subtasks can form a precondition boundary for others, where policies can be rolled out in the future subtasks only if the preconditions are satisfied. \Vadar{} can naturally be applied to a litany of these subtask dependencies often represented as directed acyclic graphs (DAGs) with each node corresponding to an individual Q-function. This provides the core benefits of value disaggregation to rich dependency structures dominant in long-horizon RL settings.

\colorpar{Illustrative Example: Robomimic Transport.}
To illustrate this, consider the robomimic Transport task, a multi-robot task where two robot arms must coordinate to transport a payload from one bin to another while managing trash removal. The task decomposes into 7 interdependent subtasks forming a DAG:

\begin{enumerate}
    \item \texttt{lid\_removed}: Lid removed from start bin
    \item \texttt{payload\_grasped}: Payload grasped from start bin (Robot 0)  
    \item \texttt{payload\_lifted}: Payload lifted above bins (Robot 0)
    \item \texttt{trash\_grasped}: Trash grasped from target bin (Robot 1)
    \item \texttt{trash\_in\_bin}: Trash placed in trash bin (Robot 1)
    \item \texttt{payload\_handed\_off}: Payload transferred from Robot 0 to Robot 1
    \item \texttt{payload\_in\_target}: Payload placed in target bin (SUCCESS)
\end{enumerate}

The dependencies form two parallel execution paths: $(1 \to 2 \to 3)$ for payload handling by Robot 0, and $(4 \to 5)$ for trash management by Robot 1. The handoff step (6) requires both paths to complete, and the final success (7) depends on the handoff, an illustration is provided in \Cref{fig:transport-subgoal-structure}.

\begin{figure}[!ht]
\centering
\small
\resizebox{0.9\linewidth}{!}{%
\begin{tikzpicture}[
    >=Latex,
    node distance=3mm and 6mm,
    goal/.style={
        draw,
        rounded corners=2pt,
        align=center,
        minimum height=5.2mm,
        inner sep=2pt,
        font=\footnotesize
    },
    normal/.style={goal, fill=gray!10},
    pred/.style={goal, fill=purple!15, draw=purple!70!black, line width=0.5pt},
    final/.style={goal, fill=green!18, draw=green!50!black, line width=0.7pt},
    edge/.style={->, thick, draw},
    note/.style={font=\scriptsize, text=purple!70!black}
]

\node[normal] (lid) {Lid removed};
\node[normal, right=of lid] (grasp) {Payload grasped};
\node[normal, right=of grasp] (lift) {Payload lifted};
\node[normal, right=of lift] (handoff) {Payload\\handed off};

\node[normal, below=5mm of grasp] (trashg) {Trash grasped};
\node[normal, right=18mm of trashg] (trashbin) {Trash in\\correct bin};

\node[normal, right=12mm of handoff, yshift=-3.5mm] (target) {$g$: Payload in\\target bin};

\draw[edge] (lid) -- (grasp);
\draw[edge] (grasp) -- (lift);
\draw[edge] (trashg) -- (trashbin);

\draw[edge] (lift) -- (handoff);
\draw[edge] (trashbin) -- (handoff);

\draw[edge] (handoff) -- (target);

\end{tikzpicture}%
}
\caption{Transport Subgoal Structure.}
\label{fig:transport-subgoal-structure}
\end{figure}
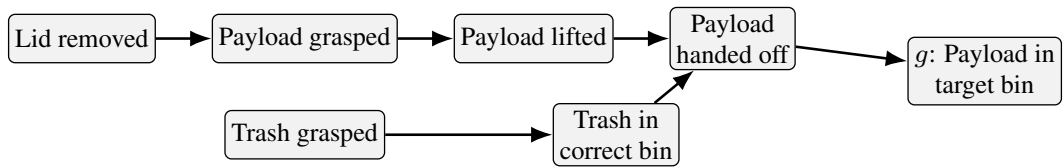
\colorpar{DAG Formalism.}
Recall from \Cref{sec:hgc-prelim} that the subgoals $G$ are partially ordered, where $g^i \prec g^j$ means that $g^i$ must be completed before $g^j$. We represent this structure with a DAG $\mathcal{D}=(G,E)$, with
\[
    (g^i,g^j)\in E
    \iff
    g^i \prec g^j
    \;\;\text{and}\;\;
    \nexists\, g^k \in G \text{ such that } g^i \prec g^k \prec g^j .
\]
This graph is acyclic: a directed cycle would imply $g^i \prec g^i$ for some subgoal $g^i$, contradicting the prerequisite relation. The linear ordering from \Cref{sec:hgc-prelim}  is the special case in which the DAG is a chain: every subgoal except the initial one has exactly one parent. We define the following graph functions:

\[\pred(g) \coloneq \{p \in G : (p,g)\in E\} \qquad \desc(S) \coloneq \{g \in G : \exists s \in S \text{ such that } s \preceq g\}.
\]

\colorpar{Gate Weights for DAG Structures.}
For a DAG, the ``current'' subgoals are the incomplete goals whose prerequisites
have all been completed:
\[
    \Gcurr(\bd)
    \coloneq
    \left\{
        g \in G :
        \bd_g = 0
        \;\;\text{and}\;\;
        \bd_{p}=1 \;\; \forall p \in \pred(g)
    \right\},
\]
We also define the terminal subgoals and the immediate successors of the current frontier as
\[
    \Gterm
    \coloneq
    \{g \in G : \nexists\, h \in G \text{ such that } (g,h)\in E\},
\]\[
    \Gnext(\bd)
    \coloneq
    \{g \in G : \exists\, p \in \Gcurr(\bd) \text{ such that } (p,g)\in E\}
    \cup
    \left(\Gcurr(\bd) \cap \Gterm\right).
\]

For $\gnext$, we handle the edge case where an element from $\gcurr$ is a terminal node, putting those nodes into the set. Using these sets, the DAG value-weighting variants are given in \Cref{tab:value-weighting-dag}, provided that we are given \bd.

\begin{table}[!ht]
\vspace{-1.25em}
\centering
\small
\setlength{\tabcolsep}{3.5pt}
\renewcommand{\arraystretch}{1.08}
\iftoggle{neurips}{%
  \caption{DAG value-weighting variants used for policy extraction in \Vadar{}.}
  \label{tab:value-weighting-dag}
}{}
\begin{tabular}{@{}lll@{}}
\toprule
\textbf{Config.} & \textbf{Weight} $w_g$ & \textbf{Interpretation} \\
\midrule
Current-myopic
& $\mathbb{I}[g \in \Gcurr(\bd)]$
& Only currently available subtasks contribute. \\
Current-future
& $\mathbb{I}[g \in \overline{\desc}(\Gcurr(\bd))]$
& Current subtasks and their descendants contribute. \\
Next-myopic
& $\mathbb{I}[g \in \Gnext(\bd)]$
& Only direct successors of current subtasks contribute. \\
Next-future
& $\mathbb{I}[g \in \overline{\desc}(\Gnext(\bd))]$
& Next subtasks and their descendants contribute. \\
\bottomrule
\end{tabular}
\iftoggle{neurips}{}{%
  \caption{DAG value-weighting variants used for policy extraction in \Vadar{}.}
  \label{tab:value-weighting-dag}
}
\vspace{-1.25em}
\end{table}


\section{Environment Details: Per-Task Descriptions: Subgoals, Dones, and DAG Structure}
\label{app:task_descriptions}

We list the subgoals used by \Vadar{} and \hgc{} on each benchmark task, taken directly from the simulator wrappers in our codebase. For each task we list (i) the high-level objective, (ii) the per-subgoal completion indicators $\done_g(s)$ as evaluated from low-level simulator state (no VLM calls are used for $\done$ evaluation), and (iii) the prerequisite structure used to gate subgoal completion. Subgoal completion is evaluated with \emph{monotonic completion} semantics: once $\done_g$ becomes true and all prerequisites are satisfied, $\done_g$ remains true for the rest of the episode. The default prerequisite structure is a sequential chain ($g^i$ requires $g^1,\dots,g^{i-1}$); tasks that override this with an explicit DAG are noted below. See \Cref{app:dag_extension} for an in-depth explanation of the DAG structure. 
 
\subsection{\Robocasa{}}
\label{app:tasks_robocasa}

Our longest-horizon tasks come from the \Robocasa{}~\citep{nasiriany2026robocasa365} suite. Each task involves multiple sequential phases of long-horizon kitchen manipulation. For our experiments, we use the state-based environments of \Robocasa{} composite tasks (having multiple subtasks to be done in order to achieve success), which do not involve the movement of the robot's base. Concretely, we evaluate on the following \Robocasa{} tasks:

\confpar{\presoakpan{}.} The agent must place a pan in the sink, place a sponge inside the pan in the sink, and then turn on the faucet. This task is decomposed into 5 sequential subgoals:
\begin{enumerate}
    \item $g^1 = \texttt{pan\_grasped\_lifted}$: pan (\texttt{obj1}) is grasped and lifted to height $z > 0.88$ m.
    \item $g^2 = \texttt{pan\_in\_sink}$: pan is contained inside the sink basin (\texttt{obj\_inside\_of(obj1, sink)}).
    \item $g^3 = \texttt{sponge\_grasped\_lifted}$: sponge (\texttt{obj2}) is grasped and lifted to height $z > 0.88$ m.
    \item $g^4 = \texttt{sponge\_in\_pan\_and\_sink}$: sponge is inside the sink, with $xy$-distance to pan $< 0.25$ m and vertical offset $z_{\mathrm{sponge}} - z_{\mathrm{pan}} \in (-0.05, 0.20)$ m.
    \item $g^5 = \texttt{water\_on}$: faucet is on, both pan and sponge are in the sink, and the gripper is far from both objects (matches the environment's original success check).
\end{enumerate}
Final task success is $\done_{g^5}$.

\confpar{\loaddishwasher{}.} The agent must load two dishes (a cup and a bowl) into the dishwasher rack and close the dishwasher. Decomposed into 5 sequential subgoals:
\begin{enumerate}
    \item $g^1 = \texttt{cup\_grasped}$: cup (\texttt{dish0}) is grasped (\texttt{check\_obj\_grasped}).
    \item $g^2 = \texttt{cup\_in\_dw}$: cup is in contact with the dishwasher rack.
    \item $g^3 = \texttt{bowl\_grasped}$: bowl (\texttt{dish1}) is grasped.
    \item $g^4 = \texttt{bowl\_in\_dw}$: bowl is in contact with the dishwasher rack.
    \item $g^5 = \texttt{dw\_closed}$: cup in rack \emph{and} bowl in rack \emph{and} dishwasher door closed (angular tolerance $0.05$).
\end{enumerate}
Final task success is $\done_{g^5}$.

\confpar{\arrangetea{}}: This is the third \Robocasa{} task we evaluate on. In this task, the agent must pick up the kettle and place it on a tray on the kitchen counter, then pick up a mug from inside the kitchen cabinet and then place it next to the kettle on the tray, and then  close the hinged cabinet that initially holds the mug. This task is decomposed into 6 natural subgoals: 
\begin{enumerate}
    \item $g^1 = \texttt{kettle\_grasped}$: kettle is grasped (\texttt{check\_obj\_grasped(kettle)}).
    \item $g^2 = \texttt{kettle\_on\_tray}$: kettle is contained on the tray (\texttt{check\_obj\_in\_receptacle(kettle, container)}).
    \item $g^3 = \texttt{kettle\_on\_tray}\,\&\,\texttt{mug\_grasped}$: kettle still on tray \emph{and} mug is grasped.
    \item $g^4 = \texttt{kettle\_on\_tray}\,\&\,\texttt{mug\_on\_tray}$: kettle still on tray \emph{and} mug is now on the tray.
    \item $g^5 = \texttt{both\_on\_tray}\,\&\,\texttt{right\_door\_closed}$: both objects on tray \emph{and} the right cabinet door is closed.
    \item $g^6 = \texttt{both\_on\_tray}\,\&\,\texttt{cab\_closed}$: both objects on tray \emph{and} the full hinged cabinet is closed.
\end{enumerate}

Final task success is $\done_{g^6}$ which matches the original success of the \arrangetea{} task.

\subsection{\Robomimic{}}
\label{app:tasks_robomimic}

\Robomimic{}~\citep{robomimic2021} contributes two contact-rich, medium-horizon tasks.

\confpar{\Transport{}.} A bimanual task in which two arms coordinate to lift a lid, transport a payload between bins, and dispose of trash. The codebase exposes 7 subgoal indicators; the user-facing description of the DAG is given in \Cref{app:dag_extension}. The 7 subgoal checks are:
\begin{enumerate}
    \item $g^1 = \texttt{lid\_removed}$: lid handle has slid leftward by more than $0.08$ m from its initial $x$-position ($x_{\mathrm{lid}} < x^{\mathrm{init}}_{\mathrm{lid}} - 0.08$).
    \item $g^2 = \texttt{payload\_grasped}$: Robot 0's gripper is in contact with the payload geoms (\texttt{\_check\_grasp}).
    \item $g^3 = \texttt{payload\_lifted}$: payload $z$-position exceeds the table height by $> 0.10$ m.
    \item $g^4 = \texttt{trash\_grasped}$: Robot 1's gripper is in contact with the trash geoms.
    \item $g^5 = \texttt{trash\_in\_bin}$: \texttt{transport.trash\_in\_trash\_bin} flag is true.
    \item $g^6 = \texttt{payload\_handed\_off}$: Robot 1 is grasping the payload \emph{and} Robot 0 is not.
    \item $g^7 = \texttt{payload\_in\_target}$: \texttt{transport.payload\_in\_target\_bin} flag is true.
\end{enumerate}
\textbf{Structure:} two parallel chains $(g^1 \to g^2 \to g^3)$ and $(g^4 \to g^5)$, both prerequisites for $g^6$, which is itself a prerequisite for $g^7$ (\Cref{fig:transport-subgoal-structure}). Final task success is $\done_{g^7}$.

\confpar{\ToolHang{}.} A single-arm precision task: insert a hook frame into a base stand, then hang a wrench on it. Decomposed into 4 sequential subgoals:
\begin{enumerate}
    \item $g^1 = \texttt{frame\_grasped}$: gripper is grasping the frame and the frame is lifted $> 0.02$ m above the table (\texttt{\_check\_grasp\_frame}).
    \item $g^2 = \texttt{frame\_inserted}$: frame is assembled into the stand (\texttt{\_check\_frame\_assembled}).
    \item $g^3 = \texttt{frame\_inserted\_and\_wrench\_grasped}$: frame is still inserted \emph{and} the wrench is grasped and lifted $> 0.02$ m above the table.
    \item $g^4 = \texttt{wrench\_attached}$: wrench is hung on the inserted frame (\texttt{\_check\_tool\_on\_frame}).
\end{enumerate}

The final success is $\done_{g^4}$ which matches the success of the original \ToolHang{} env.

\subsection{\OGBench{}}
\label{app:tasks_ogbench}

Next we use the \texttt{scene-play-singletask-task4} (\texttt{put\_in\_drawer}) from \OGBench{}~\citep{park2025ogbenchbenchmarkingofflinegoalconditioned}. In this task, the agent needs to first unlock the drawer (by pressing a button), then open the drawer, then lift the cube and place it inside the drawer and finally close the drawer. The subgoal decomposition is as follows:
\begin{enumerate}
    \item $g^1 = \texttt{press\_button\_0}$: button 0 state is set to pressed (\texttt{\_cur\_button\_states[0] == 1}); this unlocks the drawer.
    \item $g^2 = \texttt{open\_drawer}$: drawer slide $x$-position $< -0.12$ (slide range $[-0.16, 0]$).
    \item $g^3 = \texttt{block\_in\_drawer}$: block $xyz$-position satisfies the env's \texttt{\_is\_in\_drawer} predicate.
    \item $g^4 = \texttt{close\_drawer}$: $|x_{\mathrm{slide}}| < 0.04$ (drawer back to nearly fully closed).
\end{enumerate}
Final task success is $\done_{g^4}$ which is coonsisitent with the origianl success.

\subsection{\Libero{}}
We evaluate \Vadar{} on 3 multi stage kitchen tasks from the \Libero{}-10 \citep{liu2023liberobenchmarkingknowledgetransfer} suite. These are: \texttt{KITCHEN\_SCENE3}, \texttt{KITCHEN\_SCENE4}, \texttt{KITCHEN\_SCENE6}.  We use the following subgoal decomposition for these tasks: 
\confpar{\texttt{KITCHEN\_SCENE3} (turn on the stove and place the moka pot).} BDDL goal:
$(\texttt{Turnon}\ \texttt{flat\_stove\_1}) \wedge (\texttt{On}\ \texttt{moka\_pot\_1}\ \texttt{flat\_stove\_1\_cook\_region})$. Decomposed into 3 sequential subgoals:
\begin{enumerate}
    \item $g^1 = \texttt{stove\_on}$: \texttt{(turnon flat\_stove\_1)}.
    \item $g^2 = \texttt{moka\_pot\_grasped}$: gripper is grasping any geom containing \texttt{moka\_pot\_1}.
    \item $g^3 = \texttt{moka\_pot\_on\_stove}$: \texttt{(on moka\_pot\_1 flat\_stove\_1\_cook\_region)}.
\end{enumerate}
Final task success is $\done_{g^3}$.

\confpar{\texttt{KITCHEN\_SCENE4} (place black bowl in bottom drawer and close it).} The bottom drawer of \texttt{white\_cabinet\_1} is open at reset. BDDL goal: $(\texttt{Close}\ \texttt{white\_cabinet\_1\_bottom\_region}) \wedge (\texttt{In}\ \texttt{akita\_black\_bowl\_1}\ \texttt{white\_cabinet\_1\_bottom\_region})$. Decomposed into 3 sequential subgoals:
\begin{enumerate}
    \item $g^1 = \texttt{bowl\_grasped}$: gripper is grasping any geom containing \texttt{akita\_black\_bowl}.
    \item $g^2 = \texttt{bowl\_in\_drawer}$: \texttt{(in akita\_black\_bowl\_1 white\_cabinet\_1\_bottom\_region)}.
    \item $g^3 = \texttt{drawer\_closed}$: \texttt{(close white\_cabinet\_1\_bottom\_region)}.
\end{enumerate}
Final task success is $\done_{g^3}$.

\confpar{\texttt{KITCHEN\_SCENE6} (place mug in microwave and close it).} The microwave is open at reset. BDDL goal: $(\texttt{In}\ \texttt{white\_yellow\_mug\_1}\ \texttt{microwave\_1\_heating\_region}) \wedge (\texttt{Close}\ \texttt{microwave\_1})$. Decomposed into 3 sequential subgoals:
\begin{enumerate}
    \item $g^1 = \texttt{mug\_grasped}$: gripper is grasping any geom containing \texttt{white\_yellow\_mug}.
    \item $g^2 = \texttt{mug\_in\_microwave}$: \texttt{(in white\_yellow\_mug\_1 microwave\_1\_heating\_region)}.
    \item $g^3 = \texttt{microwave\_closed}$: \texttt{(close microwave\_1)}.
\end{enumerate}
Final task success is $\done_{g^3}$.

\subsection{\PushTObs{}}
\label{app:tasks_pushtobs}

The multi-objective \PushTObs{} task does not use a sequential subgoal decomposition; instead, the two ``heads'' correspond to two independent per-step signals defined in \Cref{app:multiobj-pushT}:
\begin{enumerate}
    \item $g^1 = \texttt{task\_done}$: block-goal coverage (intersection-over-union proxy) $\geq 0.95$. Ratchets up.
    \item $g^2 = \texttt{collision\_occurred}$: continuous proximity cost $c_t \in [0,1]$ around the obstacle boundary (1 inside, decaying linearly across a $40$-px soft margin). 
\end{enumerate}
Neither head is a prerequisite for the other, so $\getWeight$ collapses to a constant weighting $\bw = (1, \lambda_{\mathrm{safety}})$.

\begin{table}[!htb]
\centering
\small
\caption{Summary of subgoal counts and dependency structure across benchmark tasks.}
\label{tab:task_summary}
\begin{tabular}{llcl}
\toprule
\textbf{Suite} & \textbf{Task} & $|G|$ & \textbf{Structure} \\
\midrule
\Robocasa{}  & \presoakpan{}                          & 5 & sequential chain \\
\Robocasa{}  & \loaddishwasher{}                      & 5 & sequential chain \\
\Robocasa{}  & \arrangetea{}                          & 6 & sequential chain  \\
\Robomimic{} & \Transport{}                           & 7 & DAG, two parallel chains, joint handoff \\
\Robomimic{} & \ToolHang{}                            & 4 & sequential chain \\
\Libero{}    & \texttt{KITCHEN\_SCENE3}               & 3 & sequential chain \\
\Libero{}    & \texttt{KITCHEN\_SCENE4}               & 3 & sequential chain \\
\Libero{}    & \texttt{KITCHEN\_SCENE6}               & 3 & sequential chain \\
\OGBench{}   & \texttt{scene-play-singletask-task4}   & 4 & sequential chain \\
\midrule
--           & \PushTObs{} (multi-objective)          & 2 & independent heads (no prerequisites) \\
\bottomrule
\end{tabular}
\end{table}

\subsection{Hyperparameters}
\Vadar{} is tested with three off policy value based instantiations. Hyperparameters for each instantiation follows the best practices recommended in the corresponding original papers; we tune only the small set of task-level knobs (training-step budget, discount, BC regularization strength) needed to keep each baseline competitive on the task at hand, and we use the \emph{same} setting across all baselines (\sparsecolor{\textmd{Sparse}}, \densecolor{Dense}, \hgc{}, \Vadar{}) on a given task so that comparisons are not confounded by tuning. 

\subsubsection{Common Hyperparameters}
Most of our hyper parameters remain the same for all are experiments, unless explicitly mentioned. The detailed list of these are listed in \Cref{tab:hp-common}.

\begin{table}[!htb]
\centering
\small
\setlength{\tabcolsep}{8pt}
\caption{Maximum episode length per environment (used both for evaluation and as the truncation horizon during online finetuning)}
\label{tab:hp-env-len}
\begin{tabular}{l c}
\toprule
\textbf{Environment} & \textbf{Max episode length} \\
\midrule
\Robocasa{}  (all envs: \arrangetea{}, \presoakpan{}, \loaddishwasher{})                 & $1200$ \\
\Robomimic{} \Transport{}                 & $800$ \\
\Robomimic{} \ToolHang{}                  & $1000$ \\
\Libero{} \texttt{KITCHEN\_SCENE3/4/6}    & $500$ \\
\OGBench{} \texttt{scene-play-singletask-task4} & $100$ \\
\PushTObs{}                                & $300$ \\
\bottomrule
\end{tabular}
\end{table}

\begin{table}[!htb]
\centering
\small
\setlength{\tabcolsep}{6pt}
\caption{Hyperparameters shared across all tasks and backbones (\OGPO{}, \FQL{}, \SAC{})}
\label{tab:hp-common}
\begin{tabular}{l l}
\toprule
\textbf{Hyperparameter} & \textbf{Value} \\
\midrule
\multicolumn{2}{l}{\textit{Optimization}} \\
critic learning rate & $3 \times 10^{-4}$ \\
actor / BC learning rate & $3 \times 10^{-4}$ \\
batch size & 256 \\
replay buffer size & $2 \times 10^{6}$ \\
\midrule
\multicolumn{2}{l}{\textit{Networks}} \\
actor MLP hidden dims & $(512, 512, 512, 512)$ \\
critic MLP hidden dims & $(512, 512, 512, 512)$ \\
critic ensemble size $M$ & $10$ \\
ensemble aggregation & \texttt{mean} \\
target EMA $\tau$ (critic) & $0.05$ \\
target EMA $\bar\tau_\pi$ (actor)  & $0.05$ \\
\midrule
\multicolumn{2}{l}{\textit{\Vadar{} critic}} \\
critic shape & $S \times A \to \R^{|G|}$ \\
advantage gating mode & \texttt{grandparent\_future} (\VadarNF{}) / \texttt{goal\_future} (\VadarCF{}) \\
\midrule
\multicolumn{2}{l}{\textit{Online schedule}} \\
UTD (warmup, online) & $(1, 1)$ \\
\bottomrule
\end{tabular}
\end{table}

\subsubsection{\OGPO{} specific hyperparameters}
\label{app:hp_ogpo}
\OGPO{}\citep{patil2026ogpo} is our flow-based backbone for the \Robocasa{} and \Robomimic{} suites. \Cref{tab:hp-ogpo-common} lists the \OGPO{} hyperparameters that are held fixed across all five \OGPO{}-backed tasks; the per-task differences are then enumerated below.

\begin{table}[!htb]
\centering
\small
\setlength{\tabcolsep}{6pt}
\caption{\OGPO{} hyperparameters held fixed across \emph{all} five \OGPO{}-backed tasks (\presoakpan{}, \loaddishwasher{}, \arrangetea{}, \Transport{}, \ToolHang{})}
\label{tab:hp-ogpo-common}
\begin{tabular}{l l}
\toprule
\textbf{Hyperparameter} & \textbf{Value} \\
\midrule
$h$ (action chunk length)   & $8$ \\
\texttt{q\_warmup\_steps}   & $20{,}000$ \\
\texttt{flow\_steps}        & $10$ \\
\texttt{best\_of\_n}        & $8$ \\
critic target EMA $\tau$    & $0.05$ \\
\bottomrule
\end{tabular}
\end{table}

The remaining hyper parameters vary per task. For the \Robocasa{} suite of tasks (\presoakpan{}, \loaddishwasher{}, \arrangetea{}), we use the bc data obtained from the official VLA checkpoint \citep{nvidia2025gr00tn1openfoundation}. Since we deal with state-based environments and \Robocasa{} environments change the layout/scenes/fixtures for every environment reset, we fix them across resets to avoid partial observability in state-based settings. For \Robocasa{} environments, we set  \texttt{bc\_coeff} to 1 and \texttt{offline\_steps} to $500,000$. We use around 150 demos collected from the VLA checkpoint.  On \Robomimic{}, we do two sets of experiments: (1) with the full  set of the offline demos present (2) reducing the offline demonstration set to a fraction of the original (the \texttt{dataset.proportion} parameter): only this fraction of the offline trajectories is used for BC pretraining. The full list of our hyper parameter settings for \Robomimic{} experiments on \OGPO{} is present in \Cref{tab:hp-low-data}.

\begin{table}[h]
\centering
\small
\setlength{\tabcolsep}{5pt}
\caption{Low-data BC ablation on \Robomimic{}. Both columns use the \OGPO{} backbone with \Vadar{}-style critic disaggregation. \texttt{dataset.proportion} controls the fraction of offline demonstrations used for BC pretraining and offline replay; all other hyperparameters remain the same.}
\label{tab:hp-low-data}
\begin{tabular}{l c c c c}
\toprule
& \Transport{} (low) & \Transport{} (full) & \ToolHang{} (low) & \ToolHang{} (full) \\
\midrule
\texttt{dataset.proportion} & $0.1$ & $1.0$ & $0.30$ & $1.0$ \\
\midrule
offline\_steps      & 300K & 300K & 300K & 300K \\
online\_steps       & 3M   & 3M   & 3M   & 3M \\
q\_warmup\_steps    & 20K  & 20K  & 20K  & 20K \\
start\_training     & 40K  & 40K  & 10K  & 10K \\
\midrule
$\bar\tau_\pi$      & 0.005 & 0.005 & 0.005 & 0.005 \\
\texttt{bc\_coeff}  & 0.1   & 0.1  & 0.1  & 0.1 \\
\texttt{clip\_bc\_threshold} & 0.30 & 0.30 & 0.30 & 0.30 \\
success buffer      & true  & true  & true & true \\
\texttt{subsample\_bon}  & true  & true  & true & true \\
\bottomrule
\end{tabular}
\end{table}

\subsubsection{\FQL{} per-task hyperparameters}

\FQL{}~\citep{park2025flowqlearning} is the secondary flow-based backbone we use on \Robomimic{} \Transport{} and \ToolHang{}. \Cref{tab:hp-fql} lists the values used. We use the same network architecture and learning rates as \OGPO{} (i.e., \Cref{tab:hp-common}); We do behavior cloning using all the available data a bit longer, due to \FQL{} being less performative. The detailed list of hyperparameters used for \FQL{} are listed in \Cref{tab:hp-fql}

\begin{table}[!htb]
\centering
\small
\setlength{\tabcolsep}{6pt}
\caption{Per-task hyperparameters for the \FQL{} backbone on \Robomimic{}. Values not listed match \Cref{tab:hp-common}.}
\label{tab:hp-fql}
\begin{tabular}{l c c}
\toprule
& \Transport{} & \ToolHang{} \\
\midrule
offline\_steps    & 1M  & 500K \\
online\_steps     & 6M  & 3M \\
q\_warmup\_steps  & 50K & 50K \\
start\_training   & 10K & 10K \\
\midrule
$h$ (chunk)       & 8 & 8 \\
$\gamma$          & 0.999 & 0.999 \\
$\bar\tau_\pi$    & 0.005 & 0.005 \\
\texttt{q\_agg}   & mean & mean \\
\midrule
$|G|$, structure  & 7, DAG & 4, chain \\
\midrule
\texttt{group\_num\_samples} & 32 & 32 \\
\texttt{best\_of\_n}, \texttt{flow\_steps} & 8, 10 & 8, 10 \\
$\alpha_{\mathrm{FQL}}$ (distillation coeff.) & 1000 & 1000 \\
\midrule
\texttt{bc\_coeff}           & 1.0  & 1.0 \\
\texttt{clip\_bc\_threshold} & 0.45 & 0.45 \\
\bottomrule
\end{tabular}
\end{table}

\Needspace{12\baselineskip}
\section{Learning with VLM-Provided Subgoal Labels}
\label{sec:vlm}
\subsection{Goal-conditioned policies are computationally prohibitive with frontier model feedback}
\label{sec:vlm-judge}

For our main experiments, the subgoal completion indicators $\done_g$ are evaluated using privileged simulator state. In real-world settings, these intermediate subgoal labels may instead be obtained from a VLM given visual observations and natural-language descriptions of the subgoals \citep{du2023vision}.

\par
\Needspace{21\baselineskip}
\begin{wrapfigure}{r}{0.42\textwidth}
    
    \centering
    \includegraphics[width=\linewidth]{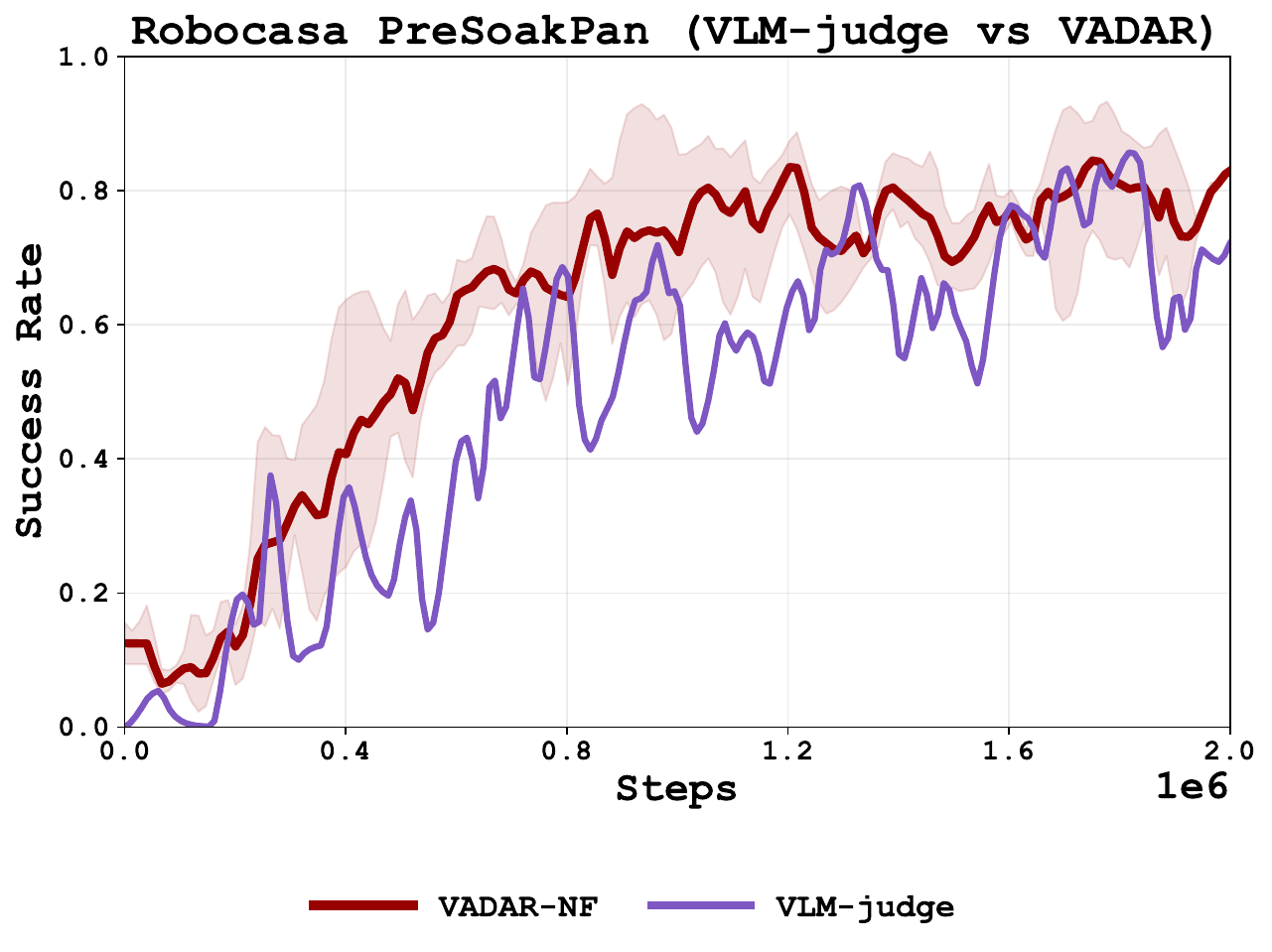}
    \caption{\VadarNF{} on \Robocasa{} \presoakpan{} with the \OGPO{} base algorithm, trained with oracle subgoal completions vs.\ with completions judged by a VLM (\texttt{gemini-3.6-flash}).}
    \label{fig:vlm-judge}
    
\end{wrapfigure}

We test whether \Vadar{} can learn when these intermediate training signals are provided by a VLM rather than by the simulator. On \Robocasa{} \presoakpan{}, we replace the oracle subgoal completion indicators with predictions from a VLM judge. The environment provides only the final task-success signal, while all intermediate subgoal supervision used by \Vadar{} comes from the VLM. The policy and critic remain state-based, and the VLM is used only to generate training-time subgoal labels. Full details and the judge prompt are provided in \Cref{app:vlm-judge-prompt}.

\Cref{fig:vlm-judge} compares \VadarNF{} trained with oracle and VLM-predicted subgoal completions. VLM supervision is somewhat less sample efficient early in training, but reaches a similar final success rate. These results show that VLMs can serve as training-time evaluators for the subgoal rewards used by \Vadar{}, without introducing any VLM dependence at deployment time, while \hgc{} needs such subgoal completion information during inference time, adding to the cost and latency of querying VLMs during deployment.

\iclrendwrap



\section{VLM Judge Prompt}
\label{app:vlm-judge-prompt}

\confpar{Protocol.} \newtext{We use \texttt{gemini-3.6-flash} as the VLM judge for the experiment in \Cref{sec:vlm-judge}. For every 8-step action chunk, the judge receives the frames at step offsets $[0, 2, 5, 7]$ from all three cameras, together with the prompt below, and returns a JSON list of $|G|$ booleans, which we use as the subgoal completion vector $\bd_t$ for the final step of the chunk. The prompt used for \Robocasa{} \presoakpan{} is shown below.}

\begin{AIbox}{VLM judge prompt for \Robocasa{} \presoakpan{}}
\ttfamily\footnotesize\raggedright
You are a strict robot-manipulation subgoal evaluator.\\
\medskip
Task: PreSoakPan\\
The input is one 8-step action chunk. It contains 4 chronological timestamps at step offsets [0, 2, 5, 7]. Every timestamp has exactly three synchronized camera views. Timestamps are ordered oldest to newest.\\
\medskip
Judge whether each subgoal{\textquotesingle}s completion criterion is satisfied at the FINAL timestamp. Use the earlier timestamps as visual history to disambiguate grasping, motion, placement, and occlusion. A transient condition such as {\textquotesingle}gripper is holding{\textquotesingle} must still be true at the final timestamp. Do not infer completion only from the expected task order, and do not require criteria that are not stated below.\\
\medskip
Subgoals:\\
0. pan\_grasped\_lifted\\
\hspace*{1.5em}Completion criterion: The gripper is holding the pan and the pan is lifted off the counter.\\
1. pan\_in\_sink\\
\hspace*{1.5em}Completion criterion: The pan is resting inside the sink basin.\\
2. sponge\_grasped\_lifted\\
\hspace*{1.5em}Completion criterion: The gripper is holding the sponge and the sponge is lifted up.\\
3. sponge\_in\_pan\_and\_sink\\
\hspace*{1.5em}Completion criterion: The sponge is inside the sink and is placed on or in the pan.\\
4. water\_on\\
\hspace*{1.5em}Completion criterion: The faucet water is running, the pan and sponge are in the sink, and the gripper has moved away from both objects.\\
\medskip
Return only strict JSON in this form: \{{\textquotedbl}done{\textquotedbl}:[true,false,...]\}. The list must contain exactly 5 booleans in the listed order.
\end{AIbox}

\section{Multi-Objective Optimization on \PushTObs{}: Full Details}
\label{app:multiobj-pushT}


This appendix gives the full environment, algorithm, and training details for the multi-objective experiment in \Cref{sec:exp-multiobj}, together with the reward-shaping ablation that controls for the safety signal's smoothness.

\confpar{ Generalization to multi-objective losses.} We observe that
\Vadarnc{} does not require that the disaggregated value functions correspond to rewards associated with subgoals of a long-horizon task. Any collection of independently valuable rewards, such as safety constraints or velocity penalties, can be included in the fitting of the multi-objective critics. We validate the utility of \Vadarnc{} for such multi-objective  problems in \Cref{sec:exp-multiobj}.

\subsection{  \Vadar{} enables stable multi-objective optimization}
\label{sec:exp-multiobj}

\par
\Needspace{19\baselineskip}
\begin{wrapfigure}{r}{0.4\textwidth}

    \centering
    \includegraphics[width=\linewidth]{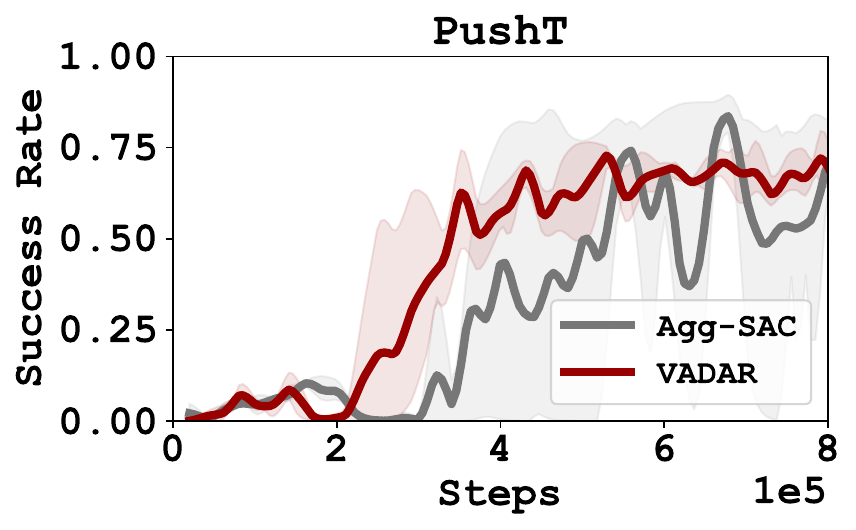}
    \caption{Multi-Objective Optimization using \Vadar{} on a variant of \texttt{PushT} created: \PushTObs{}. Here the goal of the agent is to align the T with the given position whilst avoiding contact with the obstacle }
    \label{fig:pushT-obstacle}
    
\end{wrapfigure}

Next, we test \Vadar{}'s sample efficient multi-objective optimization capability. Specifically, we consider a setting where a reward function is decomposed into $G$ groups $r = \sum_{g \in G} r^g$. While classical RL uses an \emph{single critic} for the aggregate reward $r$, we  propose \Vadar{} for \emph{separate per-reward critics}, which are aggregated after TD updates in the policy extraction step. Full reward and environment details are provided in \Cref{app:multiobj-pushT}. Here we present a high level summary.

 To test \Vadar{} for multi-objectiv optimization, we design \PushTObs{}, a variant of the popular env \texttt{PushT} \citep{chi2023diffusion} where a blue-dot pushes a gray ``T'' to the green goal configuration. We add a fixed circular obstacle and include both a sparse \emph{task reward} for pushing the T-block to the goal pose $r_{t}^{\mathrm{task}}$, and a dense \emph{safety reward} that penalizes proximity to the obstacle $r_{t}^{\mathrm{safety}}$. We instantiate \Vadar{} with \SAC{} and with $G = 2$ heads (one per signal) and use a vector-valued multigoal critic $(\bQ_\phi)$, using a constant weighting for the aggregated critic in policy extraction: $\Qagg = \bQ_\phi[\mathrm{task}] + \lambda_{\mathrm{safety}}\,\bQ_\phi[\mathrm{safety}]$. We compare against standard RL with aggregated reward $r_t = r^{\mathrm{task}}_t + \lambda_{\mathrm{safety}}\,r^{\mathrm{safety}}_t$ as a baseline(\aggSAC{}). See \Cref{app:multiobj-pushT} and \Cref{app:pusht-lambda-sweep} for more details on the env and comparison of \Vadar{} and \aggSAC{} on \PushTObs{}. \Cref{fig:pushT-obstacle} shows  that that $\Vadar{}$ outperforms aggregated rewards, even when 
\emph{ $\lambda_{\mathrm{safety}}$ is chosen to maximize baseline performance}. Moreover, we observe added training stability with \Vadar{} safety, compared to \aggSAC{}.

\relax{}
\relax{}

\relax{}
\iclrendwrap

\subsection{Environment}

\PushTObs{} extends the standard \texttt{PushT}~\citep{chi2024diffusionpolicy} planar block-pushing task with a static circular obstacle of radius $30$~px placed at $(320, 220)$ in the $512{\times}512$ workspace. The agent (a circular pusher) must move a T-shaped block to a target pose while avoiding the obstacle. We use the keypoint observation variant (\texttt{pusht-keypoints-v0}) with a $5$D state $(x_{\mathrm{agent}}, y_{\mathrm{agent}}, x_{\mathrm{block}}, y_{\mathrm{block}}, \theta_{\mathrm{block}})$ and a $2$D continuous action setting the target position of the PD-controlled pusher. Episodes terminate after $300$ steps.

\begin{figure}[!htb]
    \centering
\includegraphics[width=0.75\linewidth]{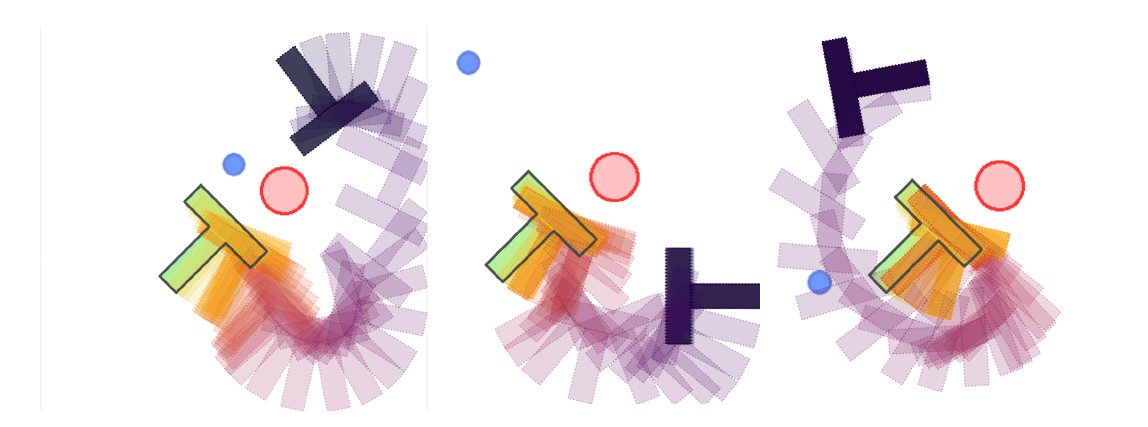}
    \caption{Depiction of our \PushTObs{} env. The agent needs to place the T into the T shaped slot (measured by IoU $\geq 95\%$) whilst avoiding the red shaped obstacle placed in the vicinity of the final position.}
    \label{fig:pusht-multiobjective}
\end{figure}

Each environment step yields two scalar signals:
\begin{itemize}
    \item \emph{Task signal}: a sparse coverage indicator $r^{\mathrm{task}}_t \in \{-1, +5\}$, equal to $+5$ when the block-goal IoU exceeds $95\%$ and $-1$ otherwise.
    \item \emph{Safety signal}: a continuous soft-margin proximity cost $c_t \in [0, 1]$ that takes value $1$ if either the pusher or the block penetrates the obstacle, decays linearly to $0$ across a safety margin of $10$~px around the obstacle boundary, and is $0$ outside the margin. The corresponding per-step safety reward used by the critic is $r^{\mathrm{safety}}_t = -5\,c_t$.
\end{itemize}

\subsection{\Vadar{} Instantiation}

We provide the full pseudocode in \Cref{alg:vadar-multiobj}. We instantiate \Vadar{} on top of a \SAC{} backbone with $G = 2$: head $0$ tracks task completion and head $1$ tracks the safety cost. The vector critic $\bQ_\phi : S \times A \to \R^{2}$ is a 4-layer MLP of width $512$ with layer norm, ensembled over $10$ members and aggregated by mean. The actor is a tanh-squashed Gaussian over a $4$-step action chunk, also a 4-layer MLP of width $512$. The TD targets follow \Cref{eq:vector_td}: the task head uses the standard reward-clipped \Vadar{} target with cap $\cclip \cdot \rdone = 15$, while the safety head uses  the backup $-5\,c_t + \gamma^{H} Q[{\mathrm{safety}}](s', a')$ with $\gamma = 0.99$ and chunk length $H = 4$. Because neither subtask is a prerequisite for the other, the policy weighting $\getWeight$ collapses to a constant $\bw = (1, \lambda_{\mathrm{safety}})$, yielding the aggregated objective
\begin{align*}
    \confinlinestyle \Qagg(s,a) \;=\; \bQ_\phi[{\mathrm{task}}](s,a) \;+\; \lambda_{\mathrm{safety}}\, \bQ_\phi[{\mathrm{safety}}](s,a),
\end{align*}
with $\lambda_{\mathrm{safety}} = 1.0$ in our experiments. (We also do a sweep over the values of  $\lambda_{\mathrm{safety}}$ \Cref{app:pusht-lambda-sweep}). Apart from the vector-valued critic and this aggregation, the agent is parameterized as a standard Squashed Gaussian policy with learnable temperature.

\begin{algorithm}[!htbp]
\caption{\Vadar{} for Multi-Objective RL (instantiated on \PushTObs{} with $G = 2$).}
\label{alg:vadar-multiobj}
\small
\begin{algorithmic}[1]
\State \textbf{Parameters:}  discount $\gamma$; learning rate $\eta$; safety weight $\modcolor{\lambda_{\mathrm{safety}}}$; \modcolor{task-head TD clip $\cclip$}.
\State \textbf{Initialize} \modcolor{multi-objective} critic $\bQ_{\phi}: S\times A \to \modcolor{\R^{2}}$ with heads $\{\mathrm{task},\mathrm{safety}\}$; policy $\pi_\theta:S\to\Delta(A)$; replay buffer $\Buffer\gets\emptyset$.
\For{episode = 1, 2, \dots}
    \While{episode not done}
        \State $a_t \gets \pi_\theta(s_t)$, transition $s_{t+1}\sim P(s_t,a_t)$
        \State Read both signals: \modcolor{task indicator $\done[{\mathrm{task}}]_{t}\in\{0,1\}$ and continuous proximity cost $c_{t}\in[0,1]$.}
        \State \modcolor{Form reward vector $\br_t = (\rsparse(\done[{\mathrm{task}}]_{t}),\, -5\,c_{t})$.}
        \State $\Buffer \gets \Buffer \cup \{(s_t,a_t,s_{t+1},\modcolor{\br_t,\done[{\mathrm{task}}]_{t}})\}$
        \For{number of updates}
        \State Sample $(s,a,s',\modcolor{\br,\done[{\mathrm{task}}}])\sim\Buffer$, $a'\sim\pi_\theta(s')$
        \State $y^{\mathrm{target}}[\mathrm{task}] = \modcolor{\min\{}\br[\mathrm{task}] + \gamma\,\bQ_\phi[{\mathrm{task}}](s',a'),\; \modcolor{\cclip\cdot \rdone\}}$ \label{line:multiobj-task-target}
        \Comment{reward clipping}
        \State $y^{\mathrm{target}}[\mathrm{safety}] = \br[\mathrm{safety}] + \gamma\,\bQ_\phi[{\mathrm{safety}}](s',a')$ \label{line:multiobj-safety-target}
        \State $\phi \gets \phi - \eta\nabla_\phi \modcolor{\sum_{g\in\{\mathrm{task},\mathrm{safety}\}}}\Lcritic(\bQ_\phi[g]\mid s,a,y^{\mathrm{target},g})$
        \State \modcolor{$\Qagg(\cdot,\cdot) \gets \bQ_\phi[{\mathrm{task}}](\cdot,\cdot) + \lambda_{\mathrm{safety}}\,\bQ_\phi[{\mathrm{safety}}](\cdot,\cdot)$} \Comment{constant weighting}
        \State $\theta \gets \theta - \eta\nabla_\theta \Lpol(\pi_\theta\mid s,a,\Qagg)$
        \EndFor
    \EndWhile
\EndFor
\end{algorithmic}
\end{algorithm}

\FloatBarrier

\subsection{Baselines}

To isolate the contribution of value disaggregation from the contribution of the underlying reward terms, we compare against a dense-reward baseline that shares \Vadar{}'s base algorithm (\SAC{}), critic ensemble, network sizes, BC pretraining, and online step budget, and that consumes \emph{the exact same per-step task and safety signals} but combines them into a single scalar reward fitted by a scalar critic. We call this \aggSAC{} (\Cref{alg:dense-parity}), and it combines the task and safety signals term-for-term as $r_t = r^{\mathrm{task}}_t + \lambda_{\mathrm{safety}}\, r^{\mathrm{safety}}_t$, with the same $\lambda_{\mathrm{safety}} = 1.0$. This is the most direct baseline: it differs from \Vadar{} only in collapsing the two signals before the Bellman backup rather than after.

\begin{algorithm}[!htbp]
\caption{Dense-reward shaping baseline for \PushTObs{} (\aggSAC{}).}
\label{alg:dense-parity}
\small
\begin{algorithmic}[1]
\State \textbf{Parameters:}  discount $\gamma$; learning rate $\eta$; safety weight $\lambda_{\mathrm{safety}}$.
\State \textbf{Initialize} \emph{scalar} critic $Q_\phi : S\times A \to \R$; policy $\pi_\theta:S\to\Delta(A)$; replay buffer $\Buffer\gets\emptyset$.
\For{episode = 1, 2, \dots}
    \While{episode not done}
        \State $a_t \gets \pi_\theta(s_t)$, transition $s_{t+1}\sim P(s_t,a_t)$
        \State Read both signals: task indicator $\done^{\mathrm{task}}_{t}\in\{0,1\}$ and continuous proximity cost $c_{t}\in[0,1]$
        \State \emph{Combine before the Bellman backup into a single scalar reward:}
        \State $r_t \gets \rsparse(\done^{\mathrm{task}}_{t}) + \lambda_{\mathrm{safety}}\,(-5\,c_{t})$ \Comment{shared signals with \Cref{alg:vadar-multiobj}; same $\lambda_{\mathrm{safety}}$}
        \State $\Buffer \gets \Buffer \cup \{(s_t,a_t,s_{t+1},r_{t})\}$
        \For{number of updates}
        \State Sample $(s,a,s',r)\sim\Buffer$, $a'\sim\pi_\theta(s')$
        \State $\ytarg = r + \gamma\, Q_\phi(s',a')$ \Comment{single scalar TD target}
        \State $\phi \gets \phi - \eta\nabla_\phi \Lcritic(Q_\phi\mid s,a,\ytarg)$
        \State $\theta \gets \theta - \eta\nabla_\theta \Lpol(\pi_\theta\mid s,a,Q_\phi)$
        \EndFor
    \EndWhile
\EndFor
\end{algorithmic}
\end{algorithm}

\FloatBarrier

\newcommand{\lsafety}{\lambda_{\mathrm{safety}}}

\subsection{Training and Evaluation Protocol}

All methods are pre-trained for $150{,}000$ behavior-cloning steps on a fixed dataset of transitions collected from a previously trained \SAC{} policy on the obstacle-free task, and then fine-tuned online for $800{,}000$ environment steps. We measure the success rate in terms of \emph{success}: block-goal coverage $\geq 95\%$ \emph{and} no collision with the obstacle occurs the episode.

\subsection{Parameter Sweep for $\lambda_{\textrm{safety}}$}
\label{app:pusht-lambda-sweep}

We sweep over values of $\lambda_{\mathrm{safety}} \in \{0.1, 1, 5, 10\}$
(\Cref{fig:sweep-pushT}) and observe two trends. First, \Vadar{} is far less
sensitive to the choice of $\lambda_{\mathrm{safety}}$: it reaches comparable
final success rates at $\lambda = 0.1$ and $\lambda = 1$, degrades only
moderately at $\lambda = 5$, and still makes meaningful progress at $\lambda
= 10$. \aggSAC{} matches \Vadar{} at small $\lambda$ but collapses to near-zero
success once $\lambda \geq 5$, since the safety penalty starts dominating the
aggregate reward. Second, at any fixed $\lambda_{\mathrm{safety}}$, \Vadar{}
trains more stably than \aggSAC{}, with narrower seed bands and fewer
fluctuations during training. We conjecture that this is due to ease of composition at the Q-function
level: each \Vadar{} head fits a $\lambda$-independent target, while
\aggSAC{} absorbs $\lambda$ into the target itself.

\begin{figure}[!htb]
    \centering
\includegraphics[width=\linewidth]{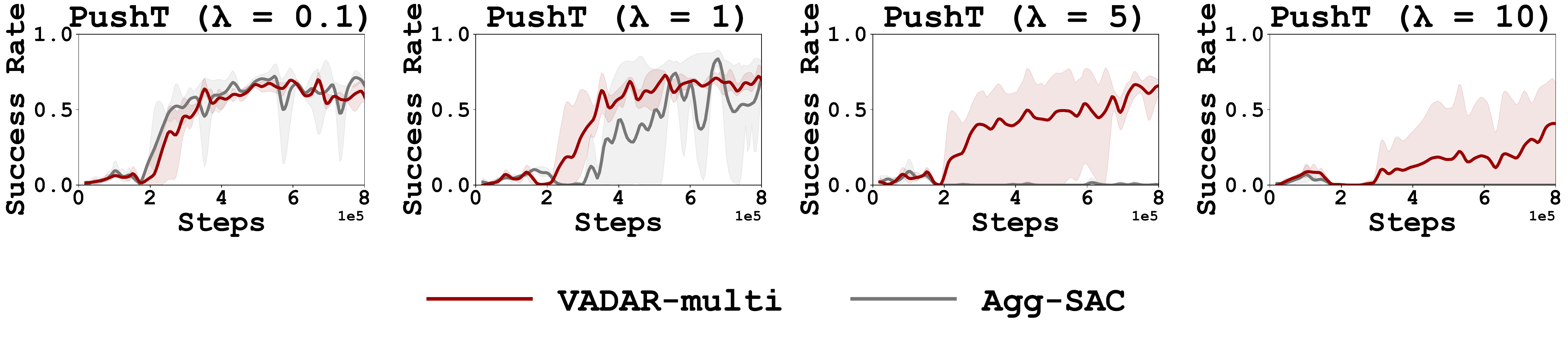}
    \caption{Plots of success rate for different values of $\lambda_{\mathrm{safety}}$ comparing both \Vadar{} and \aggSAC{} on our \PushTObs{} environment. Across all the values of  $\lambda_{\mathrm{safety}}$ that \Vadar{} is more performative and stable than \aggSAC{}}
    \label{fig:sweep-pushT}
\end{figure}


\section{Test-Time Policy Steering}
\label{app:steering}

This appendix gives the full protocol, critic training details, and results for the test-time steering experiment in \Cref{sec:exp-steering}.

\subsection{Setup}
We evaluate a critic by using it to select among candidate actions of a frozen policy. The policy is the BC-pretrained flow policy of \OGPO{}, trained once and never updated. At every step, the policy proposes $N \in \{1, 4, 16, 64\}$ candidate action chunks of length $h = 8$, the critic scores each candidate, and the highest-valued chunk is executed. With $N=1$ the critic is not used, so this setting is the unsteered BC policy and is shared across critics and budgets, and we report a single $N=1$ baseline per task. The \sparsecolor{\textmd{Sparse}} critic is scored by its ensemble mean (\Cref{app:critic_ensembling}). The \Vadar{} critic is scored by the next-future aggregate $\Qagg(s,a) = \sum_{g \in G} \bw[g]\,Q_\phi(s,a,g)$ of \Cref{eq:q_agg}, where the goal-conditioned critic $Q_\phi(s,a,g)$ is evaluated once per subgoal and $\bw = \getWeight(\bd_t)$ is computed from the environment's subgoal completion indicators.

\subsection{Critic Training}
Starting from the frozen BC policy, only the critic is trained; the policy is never updated against the critic. The replay buffer is first seeded with $20{,}000$ environment steps of frozen-policy rollouts. The critic is then updated with 4 gradient steps per environment step on batches drawn half from the offline demonstrations and half from frozen-policy rollouts, for up to $300{,}000$ environment steps, and we evaluate checkpoints at $40{,}000$, $100{,}000$, $200{,}000$ and $280{,}000$ steps. The two critics share all hyperparameters (discount, value bounds, normalization and network sizes, \Cref{tab:hp-common}) and differ only in the critic loss. We report the mean and standard deviation over 3 seeds, each evaluated on 96 episodes.

\begin{figure}[!ht]
    \centering
    \includegraphics[width=\linewidth]{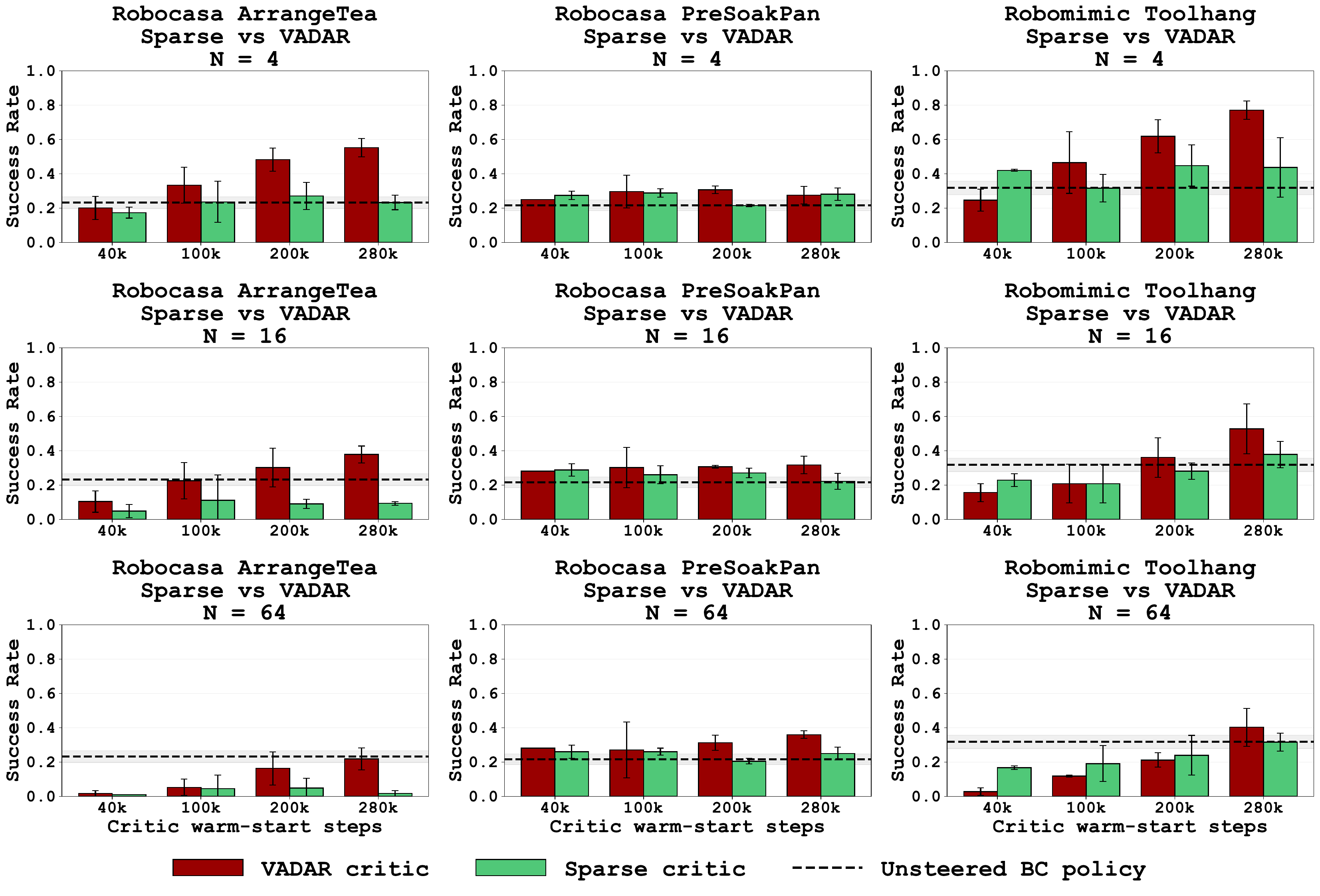}
    \caption{Best-of-$N$ steering of the frozen \OGPO{} BC policy with the \sparsecolor{\textmd{Sparse}} and \Vadar{} critics for $N \in \{4, 16, 64\}$ (rows) on \Robocasa{} \arrangetea{}, \presoakpan{} and \Robomimic{} \ToolHang{} (columns), as a function of the critic training budget. Error bars are one standard deviation over seeds; the dashed line is the unsteered BC policy.}
    \label{fig:steering-grid}
\end{figure}

\begin{table}[!ht]
\centering
\small
\setlength{\tabcolsep}{2.9pt}
\caption{Success rate of best-of-$N$ steering of the frozen \OGPO{} BC policy with the \sparsecolor{\textmd{Sparse}} and \Vadar{} critics, as a function of the critic training budget (mean $\pm$ std over seeds). The $N=1$ row is the unsteered BC policy, and the best cell in each row is highlighted.}
\label{tab:steering}
\begin{tabular}{llcccccc}
\toprule
& & \multicolumn{3}{c}{\sparsecolor{\textmd{Sparse}} critic} & \multicolumn{3}{c}{\Vadar{} critic} \\
\cmidrule(lr){3-5} \cmidrule(lr){6-8}
\textbf{Task} & \textbf{Critic budget} & $N=4$ & $N=16$ & $N=64$ & $N=4$ & $N=16$ & $N=64$ \\
\midrule
\multirow{5}{*}{\arrangetea{}}
 & $N=1$ & \multicolumn{6}{c}{$0.23 \pm 0.04$} \\
 & 40K & $0.17 \pm 0.03$ & $0.05 \pm 0.04$ & $0.01 \pm 0.00$ & \cellcolor{gray!20}$\mathbf{0.20 \pm 0.07}$ & $0.10 \pm 0.06$ & $0.02 \pm 0.02$ \\
 & 100K & $0.24 \pm 0.12$ & $0.11 \pm 0.15$ & $0.05 \pm 0.08$ & \cellcolor{gray!20}$\mathbf{0.33 \pm 0.10}$ & $0.23 \pm 0.11$ & $0.05 \pm 0.05$ \\
 & 200K & $0.27 \pm 0.08$ & $0.09 \pm 0.03$ & $0.05 \pm 0.06$ & \cellcolor{gray!20}$\mathbf{0.48 \pm 0.07}$ & $0.30 \pm 0.11$ & $0.16 \pm 0.10$ \\
 & 280K & $0.23 \pm 0.04$ & $0.09 \pm 0.01$ & $0.02 \pm 0.02$ & \cellcolor{gray!20}$\mathbf{0.55 \pm 0.05}$ & $0.38 \pm 0.05$ & $0.22 \pm 0.06$ \\
\midrule
\multirow{5}{*}{\presoakpan{}}
 & $N=1$ & \multicolumn{6}{c}{$0.22 \pm 0.03$} \\
 & 40K & $0.27 \pm 0.02$ & \cellcolor{gray!20}$\mathbf{0.29 \pm 0.04}$ & $0.26 \pm 0.04$ & $0.25$ & $0.28$ & $0.28$ \\
 & 100K & $0.29 \pm 0.02$ & $0.26 \pm 0.05$ & $0.26 \pm 0.02$ & \cellcolor{gray!20}$\mathbf{0.30 \pm 0.10}$ & \cellcolor{gray!20}$\mathbf{0.30 \pm 0.12}$ & $0.27 \pm 0.16$ \\
 & 200K & $0.22 \pm 0.01$ & $0.27 \pm 0.03$ & $0.20 \pm 0.02$ & \cellcolor{gray!20}$\mathbf{0.31 \pm 0.02}$ & \cellcolor{gray!20}$\mathbf{0.31 \pm 0.01}$ & \cellcolor{gray!20}$\mathbf{0.31 \pm 0.04}$ \\
 & 280K & $0.28 \pm 0.04$ & $0.22 \pm 0.05$ & $0.25 \pm 0.04$ & $0.28 \pm 0.05$ & $0.32 \pm 0.05$ & \cellcolor{gray!20}$\mathbf{0.36 \pm 0.02}$ \\
\midrule
\multirow{5}{*}{\ToolHang{}}
 & $N=1$ & \multicolumn{6}{c}{$0.32 \pm 0.04$} \\
 & 40K & \cellcolor{gray!20}$\mathbf{0.42 \pm 0.01}$ & $0.23 \pm 0.04$ & $0.17 \pm 0.01$ & $0.25 \pm 0.06$ & $0.16 \pm 0.05$ & $0.03 \pm 0.02$ \\
 & 100K & $0.32 \pm 0.08$ & $0.21 \pm 0.11$ & $0.19 \pm 0.10$ & \cellcolor{gray!20}$\mathbf{0.47 \pm 0.18}$ & $0.21 \pm 0.11$ & $0.12 \pm 0.01$ \\
 & 200K & $0.45 \pm 0.12$ & $0.28 \pm 0.05$ & $0.24 \pm 0.12$ & \cellcolor{gray!20}$\mathbf{0.62 \pm 0.10}$ & $0.36 \pm 0.12$ & $0.21 \pm 0.04$ \\
 & 280K & $0.44 \pm 0.17$ & $0.38 \pm 0.08$ & $0.32 \pm 0.05$ & \cellcolor{gray!20}$\mathbf{0.77 \pm 0.05}$ & $0.53 \pm 0.15$ & $0.40 \pm 0.11$ \\
\bottomrule
\end{tabular}
\end{table}

\subsection{Results}
\Cref{fig:steering-grid} and \Cref{tab:steering} give the full sweep, and we observe three trends. First, steering with the \Vadar{} critic improves with the critic training budget at every $N$ on \arrangetea{} and \ToolHang{}, and at the largest budget on \presoakpan{}. Second, the \sparsecolor{\textmd{Sparse}} critic, trained on identical data with an identical budget, does not improve on the unsteered policy on \arrangetea{} at any budget or $N$, is flat in both budget and $N$ on \presoakpan{}, and on \ToolHang{} improves on the unsteered policy at small budgets but plateaus well below the \Vadar{} critic. The only setting in which the \sparsecolor{\textmd{Sparse}} critic is ahead is \ToolHang{} at the smallest budget, and the \Vadar{} critic overtakes it from $100{,}000$ steps onward. Third, steering degrades with large $N$ for both critics, since the maximum over more candidates increasingly selects out-of-distribution actions, and the degradation is milder for the \Vadar{} critic.

\begin{figure}[!tb]
    \centering
    \includegraphics[width=\linewidth]{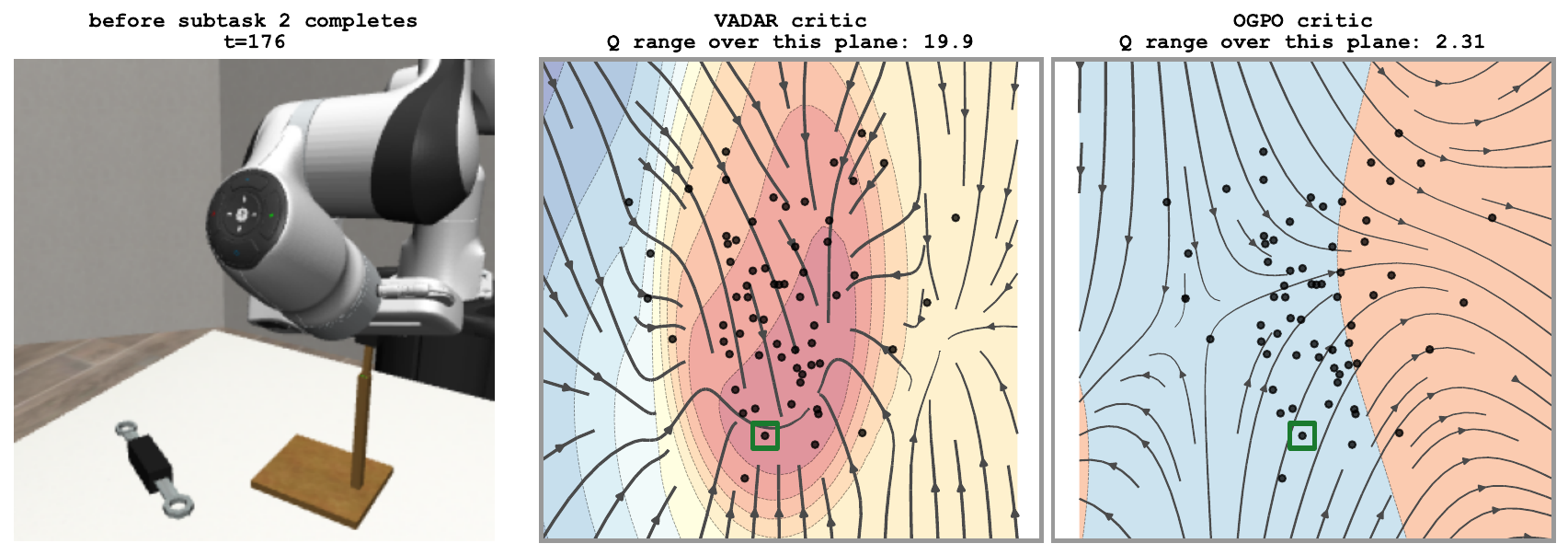}
    \caption{Q-values of the \Vadar{} (middle) and \sparsecolor{\textmd{Sparse}} (right, labelled \OGPO{} in the panel) critics, each trained with the largest critic budget of 280K steps, at a single step of \Robomimic{} \ToolHang{} before the frame is inserted into the stand (left), plotted over the plane of the first two principal components of the BC policy's 64 candidate action chunks. Dots are the candidates and the green square is the executed chunk; streamlines follow the gradient of $Q$ with respect to the action, and the contour spacing is the same in both panels.}
    \label{fig:q-landscape}
\end{figure}

To understand this gap, we roll out the frozen policy and, at each step, score the same 64 candidate action chunks with both critics. \Cref{fig:q-landscape} shows the two critics at one such step on \ToolHang{}, just before subgoal $g^2$ (\texttt{frame\_inserted}, \Cref{app:task_descriptions}) is completed. The insertion is the step of the task that demands the most precision, and hence one of its most important decision steps. We take the plane spanned by the two principal directions of the 64 candidates, evaluate each critic on a grid of actions in this plane, and plot its contours together with the direction of its action-gradient. As both critics predict returns in the same units, we use the same contour spacing for both, so that a critic with little variation over the plane shows few contours. The \Vadar{} critic varies by an order of magnitude more than the \sparsecolor{\textmd{Sparse}} critic over this plane. The executed chunk is the action taken by the successful trajectory, i.e.\ the correct action at this step. The \Vadar{} critic identifies it: it assigns its highest values to the candidates near the executed chunk, and its gradient points towards the executed chunk from across the candidate set. The \sparsecolor{\textmd{Sparse}} critic does not: it assigns nearly the same value to all candidates, and its gradient does not point towards the executed chunk. This suggests that \Vadar{} learns a better critic: one that correctly ranks the actions at the step that decides the task, whilst the \sparsecolor{\textmd{Sparse}} critic trained on the same data cannot tell them apart. We conjecture that this difference in critic quality is what drives \Vadar{}'s advantage in both the online RL experiments, where the critic supervises policy extraction, and the steering experiments, where it ranks candidate actions directly.

\section{Last-$Q$ Extraction: Full Benchmark Results}
\label{app:lastq-full}

\Cref{fig:lastq-full} gives the full comparison accompanying \Cref{sec:exp-lastq}.

\begin{figure}[!htbp]
    \centering
    \includegraphics[width=\linewidth]{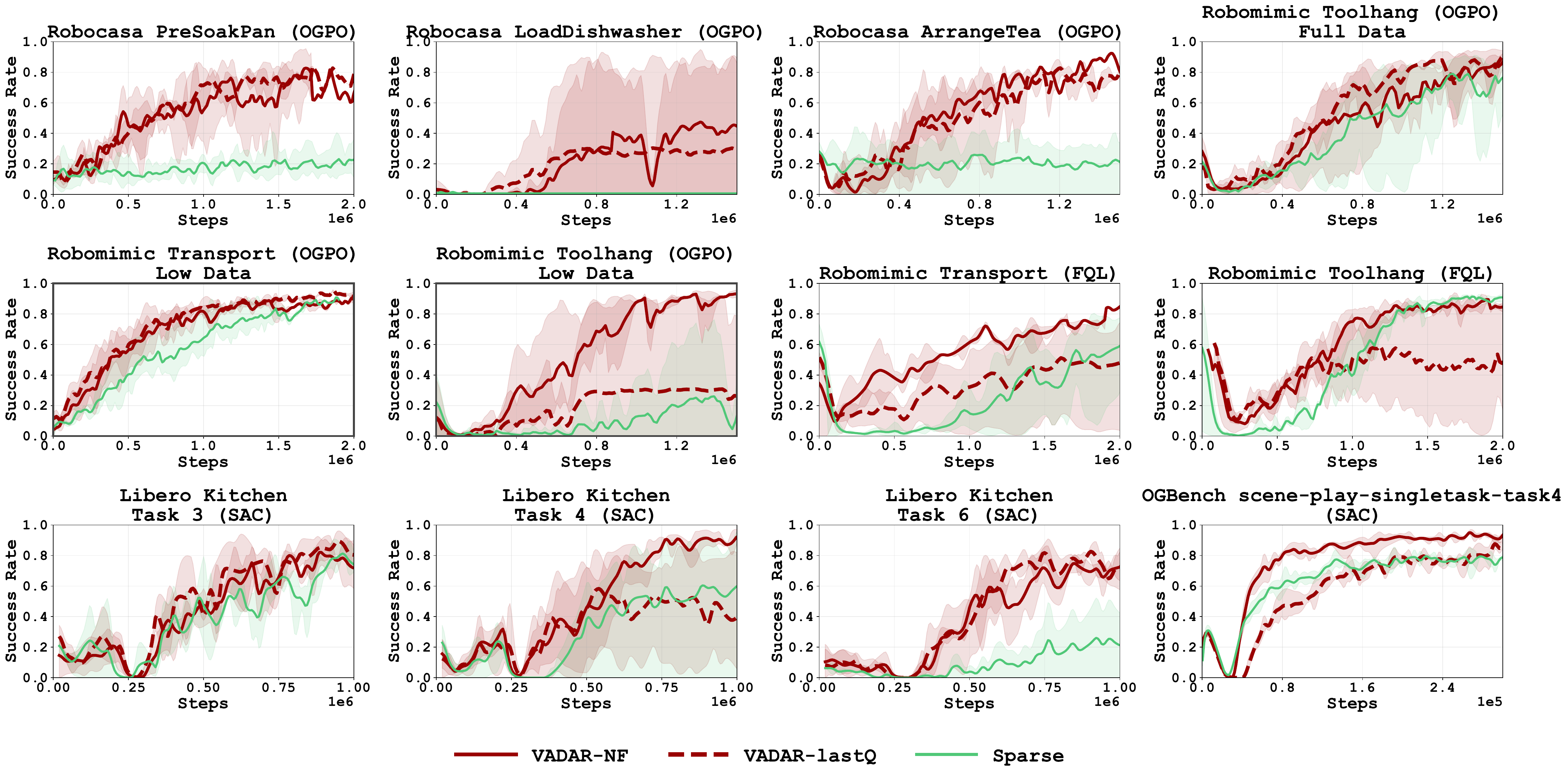}
    \caption{Success rate of \LastQ{} compared with \VadarNF{} and \sparsecolor{\textmd{Sparse}} across all benchmarks and instantiations.}
    \label{fig:lastq-full}
\end{figure}


\section{Connections to Successor Features and Reward Shaping}
\label{app:successor-features}
\Vadar{}'s aggregation step (\Cref{eq:q_agg}) is structurally similar to a success feature decomposition of a Q-function. Classical success features assume reward is linear in a feature map $\phi: \cS \times \cA \to \R^{|G|}$ \citep{dayan1993improving, barreto2017successor}, $r(s,a) = \langle \bw, \bphi(s,a) \rangle$ for some fixed $\bw \in \R^{|G|}$. The Q-function then inherits the same decomposition $Q^\pi(s,a)=\langle \bw, \bpsi^\pi(s,a) \rangle$, where $\bpsi$ is the successor feature, which represents the expected discounted sum of returns under policy $\pi$. Crucially as motivated by \citep{barreto2017successor}, this allows efficient policy extraction from arbitrary reward weightings without retraining policies across tasks. 

\Vadar{} instantiates the reward features $\bphi(s,a)$ as per-subgoal rewards $\phi_g(s,a) = r_g(s,a)$ so the successor features are precisely the per-subgoal Q-values $\bQ_{\phi_g}(s,a) = \bpsi^\pi_g(s,a)$. The weight vector $\bw(s) = \getWeight(\bd(s))$ is a state-dependent current subgoal completion vector. This allows \Vadar{} to recover sufficient statistics for evaluating the entire ancestral tree of weighted objectives in a single multi-step RL training run with $\bw(s)$ as the contraction operator.

\section{Fixed-Weight Value Aggregation and Reward Shaping}
\label{sec:reward-shaping-theory}
For fixed weights and true per-goal Q-functions under a common policy, value aggregation is equivalent to evaluating an aggregated reward.  Consider family of per-goal MDPs $M_g = (\statespace, \actionspace, P_0, P, R_g,\gamma)$ which share common dynamics and discount factor, but differ in their reward $R_g$. For example,  in the goal-conditioned setting described in \Cref{sec:hgc-prelim},   choosing $R_g(s,a) \gets \rsparse(\done_g(s))$ yields that $M_g$ is the MDP where rewards are associated with subgoal $g$. In principle, however, the theory accomodates any choices of rewards. Finally, we let $Q^{\pi,M_g}$ denote the true Q-function   (as described in the preliminaries) associated with $\pi$ under MDP $M_g$.

We now show that the Q-function achieved by aggregating per-goal Q-functions is equivalent to the Q-function of a policy in an MDP whose reward aggregates per-goal rewards:

\begin{proposition} Fix a weighting vector $\bw \in \R^G$. Then, for any goal-unconditioned policy $\pi: S \to \Delta(A)$, and any $s,a$, we have the equality
\begin{align}
    \Qagg^{\pi}(s,a) := \sum_{g} \bw[g] Q^{\pi,M_g}(s,a) = \Exp^{\pi}\left[\sum_{t\ge 0}\gamma^t R_{\bw}(s_t,a_t) \mid (s_0,a_0) = (s,a)\right] = Q^{\pi,M_{\bw}}(s,a),
\end{align}
where above $ R_{\bw} := \sum_{g}\bw[g]R_g$ and $M_{\bw}$ is the MDP defined by $(\statespace, \actionspace, P_0, P, R_\bw,\gamma)$
\end{proposition}
\begin{proof} Observe that for any MDP $M_{g}$ with reward $R_g$, we have
\begin{align*}Q^{\pi,M_g}(s,a) = \Exp^{\pi}\left[\sum_{t\ge 0}\gamma^t R_{g}(s_t,a_t) \mid (s_0,a_0) = (s,a)\right],
\end{align*}
where $\Exp^{\pi}$ is expectation under the dynamics of the MDP and policy $\pi$. As the dynamics are shared between all per-goal MDPs, we may leverage linearity of expectation to commute the expectation with the weighted summation.
\end{proof}
This identity applies to fixed weights and exact Q-functions. It does not establish equivalence between \Vadar{}'s state-dependent weighting and clipped, approximate critic updates and learning a scalar critic on an aggregated reward. The experiments in \Cref{sec:exp-lastq,sec:exp-reward-shaping} distinguish reward aggregation from the benefits of dense critic supervision.

\begin{AIbox}{Reward shaping does \emph{not} explain Vadar's success}  It might be tempting to attribute \Vadar's success to reward shaping. However, in \Cref{sec:exp-lastq}, we find that even using the \emph{final subgoal critic} can lead to strong performance in many benchmarks. We conclude that \Vadar's benefits come as much from representation learning induced by denser critic supervision as they do from explicit reward shaping  during policy extraction.
\end{AIbox}

\subsection{Subgoal Reward Shaping Does \emph{Not} Reliably Enable Long-Horizon Learning}
\label{sec:exp-reward-shaping}

\begin{figure}[!htb]
    \centering
    \includegraphics[width=\linewidth]{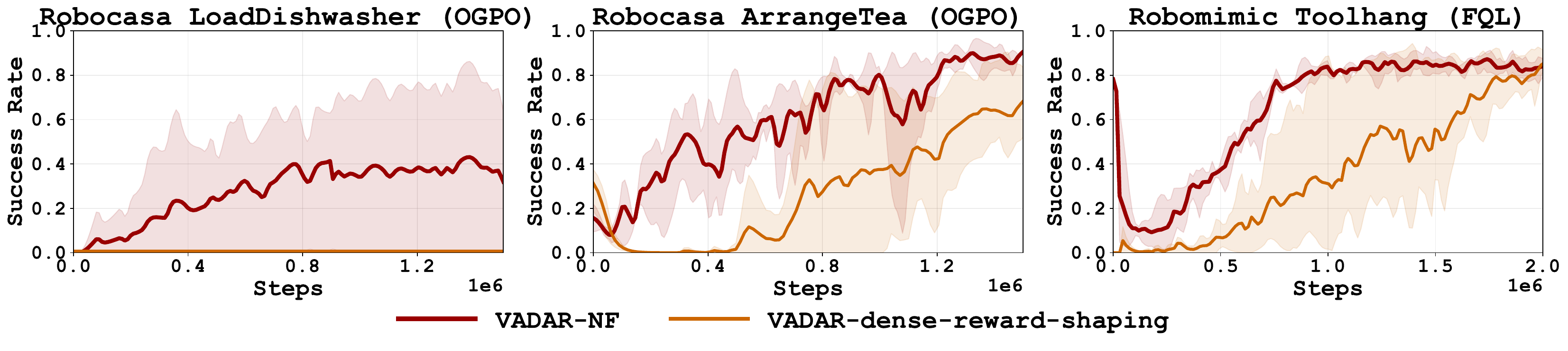}
    \caption{Subgoal reward shaping ablation. \VadarNF{} against \VadarShaped{}, which sums the per-subgoal rewards into a single shaped reward before the Bellman backup and trains a standard scalar critic on it, on \Robocasa{} with \OGPO{} and \Robomimic{} \ToolHang{} with \FQL{}.}
    \label{fig:dense-shaping}
\end{figure}

\Cref{sec:reward-shaping-theory} establishes that, for fixed weights and true Q-functions under a common policy, a weighted sum of subgoal values equals the value of the correspondingly weighted reward. This identity does not guarantee equivalent learning with state-dependent weights or clipped, approximate critics. This suggests a natural alternative to \Vadar{}: aggregate the per-subgoal rewards before the Bellman backup, and train a standard scalar critic and policy on the resulting shaped reward. We call this variant \VadarShaped{}.

\Cref{fig:dense-shaping} shows that this equivalence does not translate to learning in practice. \VadarShaped{} fails to learn on the harder \Robocasa{} tasks and is substantially less sample efficient than \VadarNF{} elsewhere. Since \VadarShaped{} combines the subgoal rewards before the Bellman update, the critic is trained on a single scalar target rather than learning each subgoal value separately. Together with the \LastQ{} results, this suggests that \Vadar{}'s gains come from multi-resolution temporal supervision, rather than from reward shaping alone.

\section{Alternative critic parameterizations yield limited additional benefit}
\label{sec:exp-critic-param}

\begin{figure}[!htbp]
    \centering
    \includegraphics[width=\linewidth]{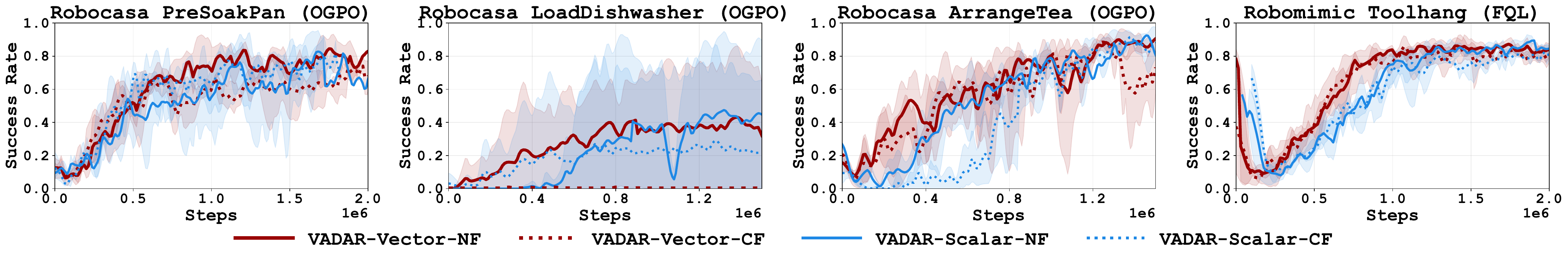}
    \caption{Critic parameterization ablation. \Vadar{} with its goal-conditioned critic $S{\times}A{\times}G \to \R$ (blue) against \VadarVector{}, which uses a vector-valued critic $S{\times}A \to \R^{|G|}$ (red), under next-future (solid) and current-future (dotted) weightings on \Robocasa{} with \OGPO{} and \Robomimic{} \ToolHang{} with \FQL{}.}
    \label{fig:scalar-vs-vector}
\end{figure}

An alternative to \Vadar{}'s goal-conditioned critic $Q_\phi(s,a,g)$, which is evaluated once per subgoal to form the vector of values used for policy extraction, is a vector-valued critic $\bQ_\phi: S{\times}A \to \R^{|G|}$ that predicts all subgoal values in a single forward pass. We call this variant \VadarVector{}.

\Cref{fig:scalar-vs-vector} shows that it performs roughly the same as \Vadar{} across tasks and weightings. This is consistent with the previous subsections: both critics are supervised on every subgoal at every transition, and it is this dense supervision, together with non-myopic policy extraction, that accounts for \Vadar{}'s performance rather than the particular critic parameterization. 

A second alternative is an \emph{iterative} critic, which refines its subgoal values over several passes and can be viewed as simulating a Bellman update inside the critic (\Cref{app:vadar_iterative_critics}). We instantiate this with \MIQ{} (\Cref{app:mip-q}) and again find performance on par with \Vadar{}'s monolithic critic, with the iterative critic training somewhat more stably (\Cref{sec:exp-iterative}).

\FloatBarrier

\Needspace{28\baselineskip}
\section{Combining \Vadar{} with Iterative Critics}
\label{app:vadar_iterative_critics}

\Vadar{} is compatible with critics that estimate values through iterative computation \citep{agrawalla2025floq}. Here we explore how iterative computation and multi-goal prediction could implement an implicit Bellman update during value prediction. Empirically, iterative critics provide limited additional benefit over the monolithic critic (\Cref{sec:exp-critic-param}).

\newcommand{\Niter}{N_{\mathrm{iter}}}
 \vspace{-0.5em}
\subsection{Simulating the Bellman Update via Iterative Critic Parametrizations}
\label{app:iterative-critic-intuition}
 Standard RL critics compute Q-values in a single monolithic forward pass.
 In the context of \Vadar{}, a monolithic critic produces all $|G|$ heads in a single shot, so the subgoal heads never communicate at inference time.
 Here we argue that combining iterative critic parameterizations with \Vadar{} can be viewed as implementing a Bellman update at inference time. 

\par
\Needspace{16\baselineskip}
\begin{wrapfigure}{r}{0.5\textwidth}

  \centering
\includegraphics[width=\linewidth]{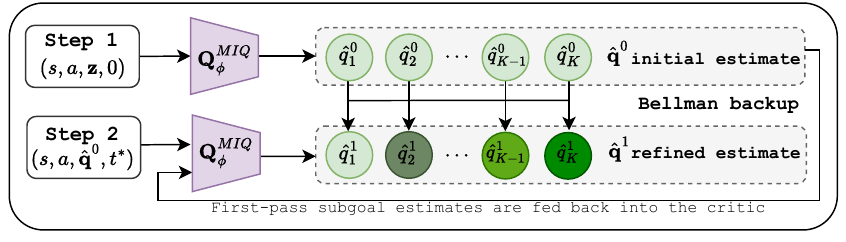}
\caption{Iterative Critic Prediction}
\label{fig:iterative-critic}

\end{wrapfigure}

An iterative  \Vadar{} critic proceeds in steps $j = 0,\dots,\Niter$, producing an estimates $\hat \bq^j \in \R^{G}$ of the critic values per step.  At each step, the critic updates its estimate $\hat \bq^{j+1} \gets \bQ_{\phi}(\bs,\ba,\hat \bq^j,j)$. 
Therefore, each updated $\bq^{k+1}$ can depend on previous estimates of critic values $\bq^k$. Informally, this computation graph can express belief-propogation updates across critic time steps, such as the classical Bellman update \citep{sutton1998reinforcement} which defines the Q-values. We depict this intuition in \Cref{fig:iterative-critic}. For training an iterative critic, we introduce a two-step minimal iterative Q-network (\MIQ) inspired by \citet[MIP]{pan2026adonoisingdispellingmyths} described in \Cref{app:mip-q}, which we anecdotally find more stable to train than flow-based critics ~\citep{agrawalla2025floq}.

\iclrendwrap

\subsection{The \MIQ{} Critic}
\label{app:mip-q}

\Vadar{}'s vector critic synergizes with iterative critic parameterizations: each iteration can propagate information across subgoal heads, so the second iteration's prediction for one subgoal can be conditioned on the first iteration's prediction for the others. We instantiate this with \MIQ{}, a critic-side adaptation of the Minimal Iterative Policy of~\citet{pan2026adonoisingdispellingmyths} that retains the two-step iteration with $t^\star = 0.9$.

\subsubsection{Architecture}
\label{app:mipq_architecture}
A MIQ critic is a single MLP $Q_\phi: S \times A \times \R \times [0,1] \to \R$ that takes a state-action pair $(s,a)$, a scalar $z$ that plays the role of the noise variable in a flow-matching linear interpolant, and a time index $t \in [0,1]$. To recall the flow-matching analogy: a flow model is trained to map the linear interpolant $I_t = t\,x + (1-t)\,z$ between Gaussian noise $z$ and a clean target $x$ at time $t \in [0,1]$ back to $x$, so that at $t = 0$ the input is pure noise and as $t \to 1$ it approaches the target. MIP~\citep{pan2026adonoisingdispellingmyths} retains this interpolant but supervises the network with the same target twice, at $t = 0$ and at $t = t^\star$ (with $t^\star = 0.9$), using a regression loss instead of a velocity-matching loss. MIQ is the same construction with the supervision target $x$ instantiated as the scalar Q-target $\ytarg$: the flow-matching interpolant is replaced by an interpolant in Q-value space, and $z \sim \mathrm{Unif}[-\sigma,\sigma]$ is the noise variable that lives on the noisy end of that interpolant. At inference, MIQ runs two autoregressive steps:
\begin{align}
    \hat q^{0} \;\gets\; Q_\phi(s,a,\, z,\, t=0), \qquad
    \hat q^{1} \;\gets\; Q_\phi(s,a,\, \hat q^{0}, \sigma/10),\, t^\star),
    \label{eq:mipq_inference}
\end{align}
i.e., the second iteration substitutes its own iteration-zero output $\hat q^{0}$ in place of an interpolated input, since $\hat q^{0}$ has been trained to approximate $\ytarg$.

\subsubsection{Training Objective}
\label{app:miq_loss}

Following~\citet[Eq.~(4.4)]{pan2026adonoisingdispellingmyths}, MIQ supervises both iterations against the standard scalar TD target $\ytarg = r + \gamma\, \Qtarg(s', a')$. To train the iteration-one mapping, we cannot feed the network the iteration-zero output (it is a moving target during training); instead, we feed a noise-corrupted version of the supervision signal itself. Concretely, we sample $z \sim \mathrm{Unif}[-\sigma,\sigma]$ and form the linear interpolant
\begin{align}
    z_{t^\star} \;=\; (1-t^\star)\, z \;+\; \ytarg,
\end{align}
which is a noisy copy of $\ytarg$ (smaller noise the larger $t^\star$ is, hence $t^\star$ near $1$). The MIQ critic loss for a single ensemble member is
\begin{align}
    \Lcritic^{\mathrm{MIQ}}(Q_\phi \mid s, a, \ytarg, z) \;=\; \big(Q_\phi(s,a,\, z,\, 0) - \ytarg\big)^2 \;+\; \big(Q_\phi\!\big(s,a,\, z_{t^\star},\, t^\star\big) - \ytarg\big)^2.
    \label{eq:mipq_loss}
\end{align}
The first term trains the iteration-zero pathway to predict $\ytarg$ from a fully random $z$; the second term trains the iteration-one pathway to denoise a near-correct estimate $z_{t^\star}$ back to $\ytarg$. At inference we substitute the iteration-zero output $\hat q^{0}$ for $z_{t^\star}$, which is well-behaved because the iteration-zero pathway has already been trained to produce a good approximation of $\ytarg$.

\subsubsection{Ensembling}
\label{app:mipq_ensembling}

We use an explicit ensemble of $M$ independently parameterized MIQ networks $Q_{\phi_1},\dots,Q_{\phi_M}$ (\texttt{ValueMIPEnsemble} in our codebase). At each update, a single noise draw $z$ is shared across the ensemble; each member produces its own $\hat q^{0,(i)}$ at iteration zero and uses that member's own $\hat q^{0,(i)}$ as the iteration-one input. Aggregation across the ensemble (\Cref{eq:vadar_qtarg}) is unchanged from \Cref{app:critic_ensembling}: we average the iteration-one predictions to form $\Qtarg$.

\subsubsection{\MIQ{} with \Vadar{}}
\label{app:miq_hyperparameters}
We have seen how \MIQ{} can be used as a general critic. Adapting it to our \Vadar{} framework requires minimal modifications. The scalar $Q_\phi : S \times A \times \R \times [0,1] \to \R$ now becomes a vector critic: $\bQ_\phi : S \times A \times \R^{|G|} \times [0,1] \to \R^{|G|}$. The noise variable $\bz \sim \mathrm{Unif}[-\sigma,\sigma]^{|G|}$ also becomes a $|G|$ dimensional vector. And the TD target is the vector $\bytarg \in \R^{|G|}$ from \Cref{eq:vector_td}, and \Cref{eq:mipq_inference} becomes
\begin{align*}
    \hat \bq^{0} \;\gets\; \bQ_\phi(s,a,\, \bz,\, 0), \qquad \hat \bq^{1} \;\gets\; \bQ_\phi(s,a,\, \hat \bq^{0},\, t^\star).
\end{align*}

Because the second pass' input is the full vector $\hat \bq^{0}$, head $g$'s second  estimate depends on every head's previous estimate.

\subsubsection{Policy Extraction via MIQ + \Vadar{}}
\label{sec:exp-iterative}


To validate the hypothesis that iterative critic parameterizations learn implicit Bellman updates over the subgoal DAG, we test \MIQ{}+\Vadar{} on \Libero{} $(|G|=3)$ with \SAC{} and \Robomimic{} \Transport{} $(|G|=7)$ with \OGPO{} as the backbone. We compare the \Vadar{}'s monolithic multi-head critic against \MIQ{} with varying Q-function update-to-data (UTD) ratios. 

\begin{figure}[!htbp]
    \centering
    \includegraphics[width=\linewidth]{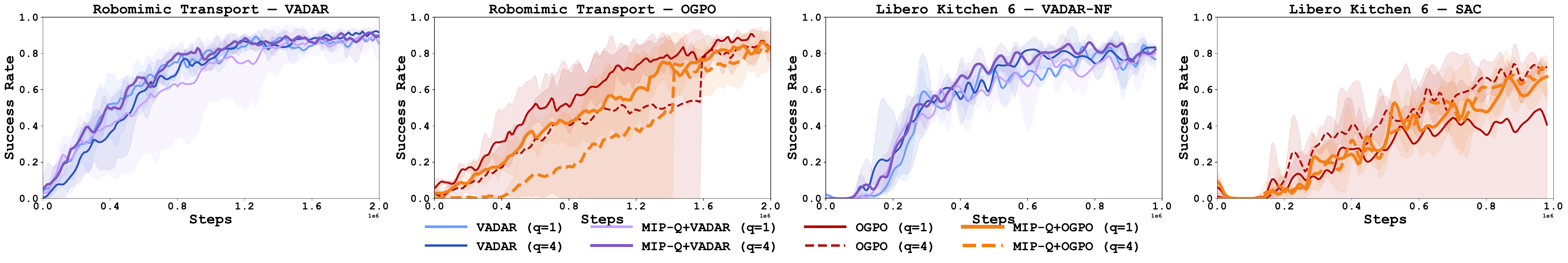}
     \caption{\MIQ{}+\Vadar{} (orange/purple) vs.\ monolithic vector critics (red/blue) at critic UTD $q=1$ (solid) and $q=4$ (dashed); the policy update frequency is held fixed across all runs. The largest \MIQ{} effect is on \OGPO{} \Transport{}, where the monolithic critic degrades sharply at $q=4$ and \MIQ{} restores parity with $q=1$.}
    \label{fig:mipq-grid}
\end{figure}

\colorpar{Results.} Across the four panels in \Cref{fig:mipq-grid}, success rates with \MIQ{} are broadly comparable to the monolithic critic, with a minor improvement in policy extraction for the higher critic UTD ($q=4$) ablation.

\begin{figure}[H]
    \centering
    \includegraphics[width=\linewidth]{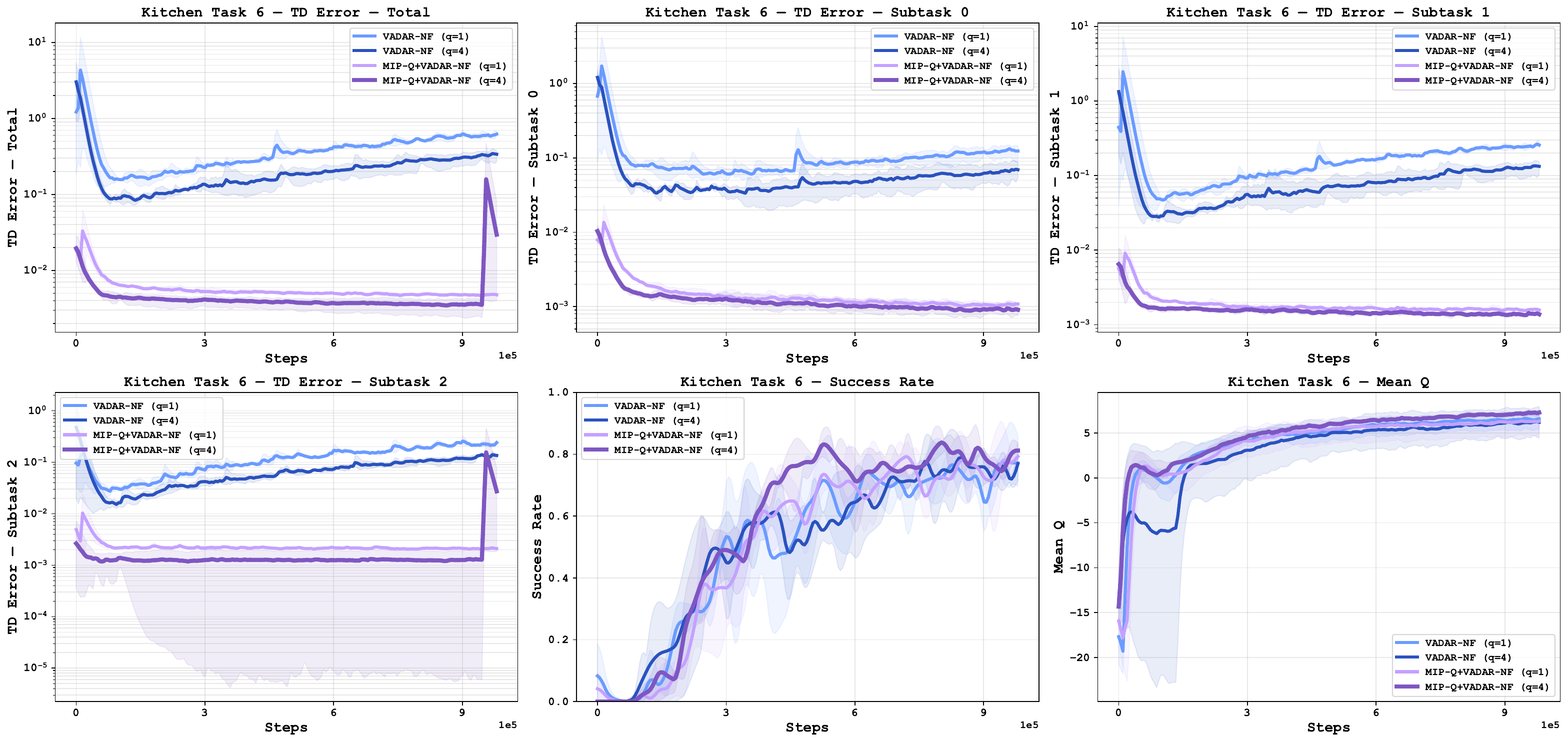}
    \caption{We compare monolithic \Vadar{} critics with \MIQ{}+\Vadar{} at critic UTD $q \in \{1,4\}$, holding the actor update frequency fixed. \MIQ{} substantially reduces total and per-subgoal TD residuals and produces lower-variance mean-$Q$ estimates, while maintaining comparable or slightly improved success rates.}
    \label{fig:td-errors-mipq}
\end{figure}

The more informative differences show up on the critic itself: in \Cref{app:mip-q} (\Cref{fig:td-errors-mipq}), per-subgoal TD residuals on \Libero{} \texttt{KITCHEN\_SCENE6} drop under \MIQ{} across all three subtasks and stay flat throughout training, while the monolithic critic's TD error drifts upward late in training. The mean-Q trajectory under \MIQ{} is more stable as hypothesized due to the implicit message passing down the subgoal DAG. We posit this being more beneficial to use-cases where critic stability even at subgoal level is a large concern.

This view connects to the \LastQ{} ablation of \Cref{sec:exp-lastq}, where policy extraction uses only the final-subgoal value and approaches \VadarNF{} on several tasks. As discussed in \Cref{sec:exp-bellman}, these results are consistent with intermediate targets improving the shared representations used for final-subgoal prediction. \MIQ{} introduces an explicit dependence between successive value estimates: its second pass conditions every subgoal prediction on the full first-pass vector $\hat \bq^{0}$. The comparison therefore tests the additional benefit of iterative refinement beyond shared subgoal supervision; comparable success rates do not establish that the non-iterative critic performs the same internal computation.


\clearpage

\end{document}